\documentclass{article} %

\usepackage{natbib} %

\usepackage{iclr2024_conference,times}

\usepackage[utf8]{inputenc} %
\usepackage[T1]{fontenc}    %
\usepackage{url}            %
\usepackage{booktabs}       %
\usepackage{amsfonts}       %
\usepackage{nicefrac}       %
\usepackage{microtype}      %
\usepackage{multirow}
\usepackage{enumitem}
\usepackage{arydshln}
\usepackage{rotating}
\usepackage{booktabs}
\usepackage{caption}
\usepackage{subcaption}
\usepackage{wrapfig}
\usepackage{algorithm}
\usepackage[noend]{algpseudocode}

\usepackage{listings} %
\definecolor{codegreen}{rgb}{0,0.6,0}
\definecolor{codegray}{rgb}{0.5,0.5,0.5}
\definecolor{codepurple}{rgb}{0.58,0,0.82}
\definecolor{backcolour}{rgb}{0.95,0.95,0.92}
\lstdefinestyle{mystyle}{
    backgroundcolor=\color{backcolour},
    commentstyle=\color{codegreen},
    keywordstyle=\color{blue},
    numberstyle=\tiny\color{codegray},
    stringstyle=\color{codepurple},
    basicstyle=\ttfamily\small,
    breakatwhitespace=false,
    breaklines=true,
    captionpos=b,
    keepspaces=true,
    numbers=left,
    numbersep=5pt,
    showspaces=false,
    showstringspaces=false,
    showtabs=false,
    tabsize=2
}
\usepackage{amsmath}
\allowdisplaybreaks %
\usepackage{amssymb}
\usepackage{mathtools}
\usepackage{amsthm}
\usepackage{nicefrac}
\usepackage{thmtools}
\usepackage{xcolor}
\usepackage{bbm}
\usepackage{cancel}
\usepackage{centernot}
\usepackage{tikz}
\usetikzlibrary{snakes,arrows,shapes}

\usepackage{xspace}
\makeatletter
\DeclareRobustCommand\onedot{\futurelet\@let@token\@onedot}
\def\@onedot{\ifx\@let@token.\else.\null\fi\xspace}

\def\eg{\emph{e.g}\onedot} 
\def\ie{\emph{i.e}\onedot} 
 
\def\st{\emph{s.t}\onedot}
\def\wrt{w.r.t\onedot} 

\makeatother

\declaretheorem[name=Theorem,refname=Thm.]{theorem}

\declaretheorem[name=Lemma]{lemma}

\declaretheorem[name=Proposition,refname=Prop.]{proposition}
\declaretheorem[name=Remark,style=remark]{remark}

\newcommand{\E}[2]{\mathbb{E}_{#1}{\left[#2\right]}}

\newcommand{\R}{\mathbb{R}}

\newcommand{\ind}{\mathrel{\perp\!\!\!\perp}}

\newcommand*{\argmax}{\mathop{\mathrm{argmax}}}
\newcommand{\defeq}{\mathrel{\mathop:}=}
\newcommand{\LHS}{\mathrm{LHS}}
\newcommand{\RHS}{\mathrm{RHS}}

\newcommand{\kl}{D_{\mathrm{KL}}}
\newcommand{\rank}{\mathrm{rank}}

\usepackage{pifont}%
\newcommand{\cmark}{\ding{51}}%
\newcommand{\xmark}{\ding{55}}%

\usepackage[pagebackref,colorlinks=true,citecolor=blue,linkcolor=red]{hyperref}

\usepackage{minitoc}

\title{Bridging State and History Representations:
Understanding Self-Predictive RL}

\newcommand{\customsize}{\fontsize{8}{8}\selectfont}
\author{%
  Tianwei Ni$^\clubsuit$,
  Benjamin Eysenbach$^\spadesuit$,
  Erfan Seyedsalehi$^\diamondsuit$\thanks{Work done while ES was at McGill University.} , 
  Michel Ma$^\clubsuit$,
  Clement Gehring$^\clubsuit$,
  \\
  \textbf{Aditya Mahajan}$^\diamondsuit$,
  \textbf{Pierre-Luc Bacon}$^\clubsuit$
\\
  $^\clubsuit$Mila, Université de Montréal,
  $^\spadesuit$Princeton University,
  $^\diamondsuit$Mila, McGill University
  \\
  {\customsize\texttt{\{tianwei.ni, michel.ma, clement.gehring, pierre-luc.bacon\}@mila.quebec},}\\
  {\customsize\texttt{eysenbach@princeton.edu}, \texttt{erfan.seyedsalehi@mail.mcgill.ca},
  \texttt{aditya.mahajan@mcgill.ca}}
}

\iclrfinalcopy %

\begin{document}

\maketitle

\doparttoc %
\faketableofcontents %

\vspace{-1em}
\begin{abstract}

Representations are at the core of all \emph{deep} reinforcement learning (RL) methods for both Markov decision processes (MDPs) and partially observable Markov decision processes (POMDPs). Many representation learning methods and theoretical frameworks have been developed to understand what constitutes an effective representation. However, the relationships between these methods and the shared properties among them remain unclear. In this paper, we show that many of these seemingly distinct methods and frameworks for state and history abstractions are, in fact, based on a common idea of \emph{self-predictive} abstraction. 
Furthermore, we provide theoretical insights into the widely adopted objectives and optimization, such as the stop-gradient technique, in learning self-predictive representations.
These findings together yield a minimalist algorithm to learn self-predictive representations for states and histories. We validate our theories by applying our algorithm to standard MDPs, MDPs with distractors, and POMDPs with sparse rewards. These findings culminate in a set of preliminary guidelines for RL practitioners.\footnote{Code is available at \url{https://github.com/twni2016/self-predictive-rl}.}\looseness=-1

\end{abstract}

\vspace{-1em}
\section{Introduction}
\vspace{-0.7em}

Reinforcement learning holds great potential to automatically learn optimal policies, mapping observations to return-maximizing actions. 
However, the application of RL in the real world encounters challenges when observations are high-dimensional and/or noisy.
These challenges become even more severe in partially observable environments, where the observation (history) dimension grows over time.
In fact, current RL algorithms are often brittle and sample inefficient in these settings~\citep{wang2019benchmarking,stone2021distracting,tomar2021learning,morad2023popgym}.

To address the curse of dimensionality, a substantial body of work has focused on compressing observations into a latent state space, known as state abstraction in MDPs~\citep{dayan1993improving,dean1997model,li2006towards}, history abstraction in POMDPs~\citep{littman2001predictive,castro2009equivalence,baisero2020learning}, and sufficient statistics or information states in stochastic control~\citep{striebel1965sufficient,Kwakernaak1965, bohlin1970information, Kumar1986}. 
Traditionally, this compression has been achieved through hand-crafted feature extractors~\citep{sutton1995generalization,konidaris2011value} or with the discovery of a set of core tests sufficient for predicting future observations~\citep{littman2001predictive,SinghLittmanJongPardoeStone2003}. Modern approaches learn the latent state space using an encoder to automatically filter out irrelevant parts of observations~\citep{lange2010deep,watter2015embed,munk2016learning}.  
Furthermore, \textit{deep} RL enables end-to-end and online learning of compact state or history representations alongside policy training.
As a result, numerous representation learning techniques for RL have surfaced (refer to \autoref{tab:main_prior_work}), drawing inspiration from diverse fields within ML and RL. However, this abundance of methods may have inadvertently presented practitioners with a ``paradox of choice'', hindering their ability to identify the best approach for their specific RL problem.

This paper aims to offer systematic guidance regarding the essential characteristics that good representations should possess in RL (the ``\textbf{what}'') and effective strategies for learning such representations (the ``\textbf{how}''). 
We begin our analysis from first principles by comparing and connecting various representations proposed in prior works for MDPs and POMDPs, resulting in a unified view.
Remarkably, these representations are all connected by a \textbf{self-predictive} condition -- the encoder can predict its next latent state~\citep{subramanian2022approximate}. 
Next, we examine how to learn such self-predictive condition in RL, a difficult subtask due to the bootstrapping effect~\citep{gelada2019deepmdp,schwarzer2020data,tang2022understanding}. We provide fresh insights on why the popular ``stop-gradient'' technique, in which the parameters of the encoder do not update when used as a target, has the promise of learning the desired condition without representational collapse in POMDPs. Building on our new theoretical findings, we introduce a minimalist RL algorithm that learns self-predictive representations end-to-end with a \textit{single} auxiliary loss, \textit{without} the need for reward model learning (thereby removing planning)~\citep{franccois2019combined,gelada2019deepmdp,tomar2021learning,hansen2022temporal,ghugare2022simplifying,ye2021mastering,subramanian2022approximate}, reward regularization~\citep{eysenbach2021robust}, multi-step predictions and projections~\citep{schwarzer2020data,guo2020bootstrap}, and metric learning~\citep{zhang2020learning,castro2021mico}. Furthermore, the simplicity of our approach allows us to investigate the role of representation learning in RL, in isolation from policy optimization.

The core contributions of this paper are as follows. 
We establish a unified view of state and history representations with novel connections (\autoref{sec:representation}), revealing that many prior methods optimize a collection of closely interconnected properties, each representing a different facet of the same fundamental concept. Moreover, we enhance the understanding of self-predictive learning in RL regarding the choice of the objective and its impact on the optimization dynamics (\autoref{sec:learning}). Our theory results in a simplified and novel RL algorithm designed to learn self-predictive representations fully end-to-end (\autoref{sec:algo}). 
Through extensive experimentation across three benchmarks (\autoref{sec:experiments}), we provide empirical evidence substantiating all our theoretical predictions using our simple algorithm. 
Finally, we offer our recommendations for RL practitioners in \autoref{sec:recommend}.
Taken together, we believe that our work potentially aids in addressing the longstanding challenge of learning representations in MDPs and POMDPs.\looseness=-1

\vspace{-0.7em}
\section{Background}
\label{sec:background}
\vspace{-0.7em}

\textbf{MDPs and POMDPs.} In the context of a POMDP $\mathcal M_O = (\mathcal O, \mathcal A, P, R, \gamma, T)$\footnote{While the classic definition of a POMDP~\citep{cassandra1994acting} features a state space, we assume it to be unknown, thus our view of a POMDP is a black-box input-output system~\citep{subramanian2022approximate}. It is also similar to predictive state representations~\citep{littman2001predictive} that rely solely on observable quantities.
}, an agent receives an observation $o_t \in \mathcal O$ at time step $t$, selects an action $a_t \in \mathcal A$ based on the observed history $h_t \defeq (h_{t-1},a_{t-1},o_t) \in \mathcal H_t$\footnote{In general, $h_t \defeq (h_{t-1}, a_{t-1}, r_{t-1}, o_t)$~\citep{izadi2005using}. In this study, we assume rewards are inaccessible during policy inference.},
and obtains a reward $r_t \sim R(h_t,a_t)$ along with the subsequent observation $o_{t+1} \sim P(\cdot \mid h_t,a_t)$. The initial observation $h_1\defeq o_1$ is sampled from the distribution $P(o_1)$. The total time horizon is denoted as $T \in \mathbb N^+ \cup \{+\infty\}$, and the discount factor is $\gamma \in [0,1]$ (less than $1$ for infinite horizon).
To maintain brevity, we employ the ``prime'' symbol to represent the next time step, for example writing $h'=(h,a,o')$.
Under the above assumptions, our agent acts according to a policy $\pi(a\mid h)$ with action-value $Q^\pi(h,a)$. Furthermore, it can be shown that there exists an optimal value function $Q^*(h,a)$ such that $Q^*(h,a) = \E{}{r\mid h,a} + \gamma \E{o'\sim P(\mid h,a)}{\max_{a'} Q^*(h',a')}$, and a deterministic optimal policy $\pi^*(h) = \argmax_a Q^*(h,a)$. 
In an MDP $\mathcal M_S = (\mathcal S,\mathcal A, P, R, \gamma, T)$, the observation $o_t$ and history $h_t$ are replaced by the state $s_t \in \mathcal S$\footnote{In finite-horizon MDPs, we assume $s$ includes the time step $t$.}.

\textbf{State and history representations.} In a POMDP, an \textbf{encoder} is a function $\phi: \mathcal H_t \to \mathcal Z$ that produces a history \textbf{representation} $z=\phi(h) \in \mathcal Z$. Similarly, in an MDP, we replace $h$ with $s$, resulting in a state encoder $\phi: \mathcal S \to \mathcal Z$ and a state representation $z=\phi(s) \in \mathcal Z$.
This representation is known as an ``abstraction''~\citep{li2006towards} or a ``latent state''~\citep{gelada2019deepmdp}. 
Such encoders are sometimes shared and simultaneously updated by downstream components (\eg policy, value, world model) of an RL system~\citep{hafner2020mastering,hansen2022temporal}.
In this paper, we are interested in such a shared encoder, or \textbf{the encoder learned for the value function} if the encoders are separately learned.\looseness=-1

Below, we present the key abstractions that are central to this paper, along with their established connections. We will highlight the \textbf{conditions} met by each abstraction. We defer additional common abstractions and related concepts to \autoref{sec:more_def}.

\textbf{1. $Q^*$-irrelevance abstraction.} An encoder $\phi_{Q^*}$ provides a $Q^*$-irrelevance abstraction~\citep{li2006towards} if it contains the necessary information for predicting the return. 
Formally, if $\phi_{Q^*}(h_i) = \phi_{Q^*}(h_j)$, then
$
Q^*(h_i, a) = Q^*(h_j,a), \forall a
$.
A $Q^*$-irrelevance abstraction can be achieved as a by-product of learning an encoder $\phi$ through a value function $\mathcal Q(\phi(h),a)$ end-to-end using model-free RL. If the optimal values match, then $\mathcal Q^*(\phi_{Q^*}(h),a) = Q^*(h,a),\forall h,a$.

\textbf{2. Self-predictive (model-irrelevance) abstraction.} We view the model-irrelevance concept~\citep{li2006towards} from a self-predictive standpoint. Specifically, a model-irrelevant encoder $\phi_L$
fulfills two conditions: \textbf{expected reward prediction (RP)} and \textbf{next latent state distribution prediction (ZP)}\footnote{RP and ZP are labeled as (P1) and (P2), respectively, in~\citet{subramanian2022approximate}.}, ensuring that the encoder can be used to predict expected reward and the next latent state distribution.\looseness=-1
\vspace{0.5em} %
\begin{align}
\label{eq:RP}
&\exists R_z: \mathcal Z \times \mathcal A \to \R, \quad \st\quad  \E{}{r\mid h,a} = R_z (\phi_L(h),a), \quad \forall h,a, \tag{RP}  \\
\label{eq:ZP}
&\exists P_{z}: \mathcal Z \times \mathcal A \to \Delta(\mathcal Z), \quad \st\quad P(z'\mid h,a) = P_z(z' \mid \phi_L(h), a), \quad \forall h,a,z', \tag{ZP} \\ 
\label{eq:EZP}
& \E{}{z' \mid h,a} = \E{}{z' \mid \phi_L(h),a},\quad \forall h,a. \tag{EZP}
\end{align}
\newcommand{\RP}{\ref{eq:RP}\xspace}%
\newcommand{\ZP}{\ref{eq:ZP}\xspace}%
\newcommand{\EZP}{\ref{eq:EZP}\xspace}%
A weak version of \ZP is the \textbf{expected next latent state $z$ prediction (EZP)} condition.
\ZP can be interpreted as a sufficient statistics condition on $\phi_L$: the next latent state $z'$ is conditionally independent of the history $h$ when $\phi_L(h)$ and $a$ is known, symbolized as $z' \ind h \mid \phi_L(h), a$. 
Satisfying \ZP only is trivial and can be achieved by employing a constant representation $\phi(h) = c$, where $c$ is a fixed constant. 
Therefore, \ZP must be used in conjunction with other conditions (\eg, \RP) to avoid such degeneration.
The $\phi_L$ is known as a bisimulation generator~\citep{givan2003equivalence} in MDPs and an information state generator~\citep{subramanian2022approximate} in POMDPs.

\newcommand{\OP}{\ref{eq:OP}\xspace}%
\newcommand{\Rec}{\ref{eq:Rec}\xspace}%
\newcommand{\OR}{\ref{eq:OR}\xspace}%
\textbf{3. Observation-predictive (belief) abstraction.} 
This abstraction is implicitly introduced by \citet{subramanian2022approximate}, which we denote by $\phi_O$, and satisfies three conditions: expected reward prediction \RP, \textbf{recurrent encoder (Rec)} and \textbf{next observation distribution prediction (OP)}\footnote{OP and Rec are labeled as (P2a) and (P2b), respectively, in \citet{subramanian2022approximate}.}.
\begin{align}
\label{eq:Rec}
&\exists \psi_z: \mathcal Z \times \mathcal A \times \mathcal O \to \mathcal Z, \quad \st\quad \phi(h') = \psi_z(\phi_O(h),a,o'), \quad \forall h,a,o', \tag{Rec} \\
\label{eq:OP}
&\exists P_{o}: \mathcal Z \times \mathcal A \to \Delta(\mathcal O), \quad \st\quad P(o'\mid h,a) = P_o(o' \mid \phi_O(h), a), \quad \forall h,a,o', \tag{OP} \\
\label{eq:OR}
& \exists \psi_{o}: \mathcal Z \to \mathcal O,\quad \st\quad o = \psi_o(\phi_O(h)),\quad \forall h. \tag{OR}
\end{align}
Similarly, the \OP condition is equivalent to $o' \ind h \mid \phi_O(h), a$, and \OP is closely related to \textbf{observation reconstruction (OR)}, widely used in practice~\citep{yarats2021improving}.
The recurrent encoder (\Rec) condition is satisfied for encoders parameterized with feedforward or recurrent neural networks~\citep{elman1990finding,hochreiter1997long}, but not Transformers~\citep{vaswani2017attention}. In this paper, we assume the \Rec condition is always satisfied. 
In POMDPs, $\phi_O$ is well-known as a belief state generator~\citep{kaelbling1998planning}.

We extend the relations between these abstractions known in MDPs~\citep{li2006towards} to POMDPs.\looseness=-1
\begin{theorem}[\textbf{Relationships between common abstractions (informal)}]
\label{thm:hierarchy_informal}
An encoder satisfying $\phi_O$ also belongs to $\phi_L$; an encoder satisfying $\phi_L$ also belongs to $\phi_{Q^*}$; the reverse is not necessarily true. 
\end{theorem}
\vspace{-0.5em}

\vspace{-0.7em}
\section{A Unified View on State and History Representations}
\label{sec:representation}
\vspace{-0.5em}

\subsection{An Implication Graph of Representations in RL}
\label{sec:graph}
\vspace{-0.5em}

\begin{wrapfigure}{r}{0.43\linewidth}
\vspace{-3em}
    \centering
    \begin{tikzpicture}[>=latex',line join=bevel,very thick,scale=0.5]
\node (ZP) at (119.17bp,94.309bp) [draw,fill=yellow,circle] {\ZP};
  \node (Q) at (27.794bp,41.552bp) [draw,circle] {$\phi_{Q^*}$};
  \node (RP) at (27.794bp,147.07bp) [draw,circle] {\RP};
  \node (Rec) at (212.85bp,23.372bp) [draw,circle] {\Rec};
  \node (OP) at (205.37bp,187.41bp) [draw,circle] {\OP};
  \node (OR) at (230.8bp,117.55bp) [draw,circle] {\OR};
  \draw [brown,->] (ZP) ..controls (89.708bp,77.298bp) and (74.614bp,68.584bp)  .. (Q);
  \draw [orange,->] (ZP) ..controls (87.938bp,112.34bp) and (69.843bp,122.79bp)  .. (RP);
  \draw [gray,->] (ZP) ..controls (153.26bp,77.125bp) and (174.22bp,61.644bp)  .. (Rec);
  \draw [red,->] (ZP) ..controls (141.31bp,128.5bp) and (162.76bp,152.08bp)  .. (OP);
  \draw [orange,->] (Q) ..controls (20.866bp,83.269bp) and (20.796bp,101.24bp)  .. (RP);
  \draw [brown,->] (RP) ..controls (34.494bp,113.26bp) and (34.841bp,95.206bp)  .. (Q);
  \draw [blue,->] (Rec) ..controls (175.55bp,42.861bp) and (155.01bp,58.208bp)  .. (ZP);
  \draw [blue,->] (OP) ..controls (182.27bp,152.13bp) and (160.89bp,128.7bp)  .. (ZP);
  \draw [transparent,->] (OR) ..controls (191.7bp,109.41bp) and (167.92bp,104.46bp)  .. (ZP);
  \draw [red,->] (OR) ..controls (220.89bp,144.79bp) and (218.52bp,151.3bp)  .. (OP);
\end{tikzpicture}
    \vspace{-0.7em}
\caption{\footnotesize \textbf{An implication graph} showing the relations between the conditions on \textit{history} representations. The source nodes of the edges with the same color \textit{together} imply the target node. In MDPs, \OR implies all the other conditions.
All the connections are discovered in this work, except for (1) \OP + \Rec implying \ZP, (2) \ZP + \RP implying $\phi_{Q^*}$.\looseness=-1}
\label{fig:relation_main}
\vspace{-2.5em}
\end{wrapfigure}
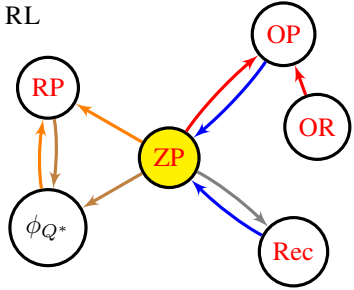

Using the taxonomy of state and history abstractions, it becomes possible to establish theoretical links among different representations and their respective conditions discussed earlier. These connections are succinctly illustrated in a directed graph, as shown in \autoref{fig:relation_main}.
In this section, we highlight the most significantly novel finding, while postponing the other propositions and proofs to \autoref{sec:unified_view}. 

The definition of self-predictive and observation-predictive abstractions suggests the classic \textit{phased} training framework. In phased training, we alternatively train an encoder to predict expected rewards (\RP) and predict next latent states (\ZP) or next observations (\OP), and also train an RL or planning agent on the latent space with the encoder ``detached'' from downstream components.
On the other hand, we show in our \autoref{thm:EG_ZP_implies_Er} that if we learn an encoder \textit{end-to-end} in a model-free fashion but using \ZP (or \OP) as an auxiliary task, then the ground-truth expected reward can be induced by the latent $Q$-value and latent transition. Thus, the encoder also satisfies \RP and generates  $\phi_L$ (or $\phi_O$) representation already. 

\begin{theorem}[\textbf{\ZP + $\phi_{Q^*}$ imply \RP}]
\label{thm:EG_ZP_implies_Er}
If an encoder $\phi$ satisfies \ZP, and  $\mathcal Q(\phi(h),a) = Q^*(h,a),\forall h,a$, then we can construct a latent reward function $\mathcal R_z(z,a) \defeq \mathcal Q(z,a) - \gamma 
\E{z' \sim P_z(\mid z,a)}{\max_{a'} \mathcal Q(z', a')}$, such that $\mathcal R_z(\phi(h),a) = \E{}{r \mid h,a},\forall h,a$.
\end{theorem}
\vspace{-0.7em}

This result holds significance, as it suggests that end-to-end approaches to learning $\phi_{Q^*}$ + \ZP (\OP) have similar theoretical justification as classic phased approaches. 

\subsection{Which Representations Do Prior Methods Learn?}

\bgroup
\def\arraystretch{0.95}
\setlength{\tabcolsep}{0.2em}
\begin{table}[t]
    \vspace{-3.5em}
    \centering
    \footnotesize
\caption{\footnotesize \textbf{Which optimal representation will be learned by the value function in prior works?} The ``PO'' column shows if the approach applies to POMDPs. The ``Conditions'' column shows the conditions that the encoder of the optimal value satisfies (see the appendix for the ``metric'' and ``regularization'' conditions). The \ZP loss shows the loss function they use to learn \ZP condition. The \ZP target shows whether they use online or stop-gradient (including detached and EMA) encoder target.
Due to the space limit, we omit the citations for recurrent model-free RL~\citep{hausknecht2015deep,kapturowski2018recurrent,ni2021recurrent} and belief-based methods~\citep{wayne2018unsupervised,Hafner2019LearningLD,han2019variational,lee2019stochastic}.
}
\vspace{-1em}
    \label{tab:main_prior_work}
    \begin{tabular}{cccccc}
    \toprule
      \textbf{Work}   &  \textbf{PO?}  & \textbf{Abstraction} & \textbf{Conditions} & \textbf{\ZP loss} & \textbf{\ZP target} \\
      \midrule
      Model-Free \& Classic Model-Based RL & \xmark & $\phi_{Q^*}$ & $\phi_{Q^*}$ & N/A & N/A \\ 
      MuZero~\citep{schrittwieser2020mastering} & \xmark & unknown & $\phi_{Q^*}$ + \RP & N/A & N/A \\
      MICo~\citep{castro2021mico} & \xmark & unknown & $\phi_{Q^*}$ + metric & N/A & N/A \\ 
     CRAR~\citep{franccois2019combined} & \xmark & $\phi_L$ & $\phi_{Q^*}$ + \RP + \ZP + reg. & $\ell_2$ & online   \\
     DeepMDP~\citep{gelada2019deepmdp} & \xmark & $\phi_L$ & $\phi_{Q^*}$ + \RP + \ZP & W ($\ell_2$) & online   \\ 
     SPR~\citep{schwarzer2020data} & \xmark & $\phi_L$ &$\phi_{Q^*}$ + \ZP & $\cos$ & EMA \\ 
     DBC~\citep{zhang2020learning} & \xmark &$\phi_L$ & $\phi_{Q^*}$ + \RP + \ZP + metric & FKL & detached \\
     LSFM~\citep{lehnert2020successor} & \xmark & $\phi_L$ &  $\phi_{Q^*}$ + \RP + \EZP & SF & detached \\
     Baseline in~\citep{tomar2021learning} & \xmark &$\phi_L$ & $\phi_{Q^*}$ + \RP + \ZP & $\ell_2$ & detached \\
     EfficientZero~\citep{ye2021mastering} & \xmark & $\phi_L$ & $\phi_{Q^*}$ + \RP + \ZP & $\cos$ & detached \\ 
     TD-MPC~\citep{hansen2022temporal} & \xmark & $\phi_L$ & $\phi_{Q^*}$ + \RP + \ZP & $\ell_2$ & EMA \\

    ALM~\citep{ghugare2022simplifying} & \xmark & $\phi_L$ & $\phi_{Q^*}$ + \ZP & RKL & EMA \\
    TCRL~\citep{zhao2023simplified}& \xmark & $\phi_L$ & \RP + \ZP & $\cos$ & EMA  \\
    OFENet~\citep{ota2020can} & \xmark & $\phi_O$ & $\phi_{Q^*}$ + \OP & N/A & N/A  \\ %
     \noalign{\vskip 2pt}
    \hdashline[5pt/2pt]
     \noalign{\vskip 2pt}
      Recurrent Model-Free RL & \cmark & $\phi_{Q^*}$ & $\phi_{Q^*}$ & N/A & N/A \\ 
     PBL~\citep{guo2020bootstrap} & \cmark & $\phi_L$ & $\phi_{Q^*}$ + \ZP & $\ell_2$ & detached \\
    AIS~\citep{subramanian2022approximate} & \cmark & $\phi_L$, $\phi_O$ & \RP + \ZP or \OP  & $\ell_2$, FKL  & detached \\ 
    Belief-Based Methods  & \cmark & $\phi_O$ & \RP + \ZP + \OR & FKL & online \\
    Causal States~\citep{zhang2019learning} & \cmark & $\phi_O$ & \RP + \OP  & N/A & N/A \\
    \midrule
    Minimalist $\phi_L$ (this work) & \cmark &  $\phi_L$ & $\phi_{Q^*}$  + \ZP   &  $\ell_2$, KL & stop-grad  \\
     \bottomrule
    \end{tabular}
    \vspace{-1em}
\end{table}
\egroup

With the unified view of state and history representations, we can categorize prior works based on the conditions satisfied by the \textit{optimal} encoders of their value functions. \autoref{tab:main_prior_work} shows representative examples. 
The unified view enables us to draw interesting connections between prior works, even though they may differ in RL or planning algorithms and the encoder objectives. Here we highlight some important connections and provide a more detailed discussion of all prior works in \autoref{app:prior_works}.

To begin with, it is important to recognize that classic model-based RL actually learns  $\phi_{Q^*}$ in value function. Model-based RL trains a policy and value by rolling out on the learned model. However, the policy and value do not share representations with the model~\citep{sutton1990integrated,sutton2012dyna,chua2018deep,kaiser2019model,janner2019trust}, or learn their representations from maximizing returns~\citep{tamar2016value,oh2017value,silver2017predictron}.
Secondly, as shown in \autoref{tab:main_prior_work}, there is a wealth of prior work on approximating $\phi_L$, stemming from different perspectives. These include bisimulation~\citep{gelada2019deepmdp}, information states~\citep{subramanian2022approximate}, variational inference~\citep{eysenbach2021robust,ghugare2022simplifying}, successor features~\citep{barreto2017successor,lehnert2020successor}, and self-supervised learning~\citep{schwarzer2020data,guo2020bootstrap}. The primary differences between these approaches lie in their selection of (1) architecture (whether learning \RP, $\phi_{Q^*}$, or both), (2) \ZP objectives (such as $\ell_2$, cosine, forward or reverse KL, as discussed in \autoref{sec:biased}), and (3) \ZP targets for optimization (including online, detached, EMA, as detailed in \autoref{sec:optim}).
Finally, observation-predictive representations are typically studied in POMDPs, where they are known as belief states~\citep{kaelbling1998planning} and related to predictive state representations~\citep{littman2001predictive}.\looseness=-1

\section{On Learning Self-Predictive Representations in RL}
\label{sec:learning}

The implication graph (\autoref{fig:relation_main}) establishes the theoretical connections among various representations in RL, yet it does not address the core learning problems. This section aims to give some theoretical answers to \textbf{how} to learn self-predictive representations. 
While self-predictive representation holds promise, it poses significant learning challenges compared to grounded model-free and observation-predictive representations. The bootstrapping effect, where $\phi$ appears in both sides of \ZP (since $z'$ also relies on $\phi(h')$), contributes to this difficulty. 
We present detailed analyses of the objectives in \autoref{sec:biased} and optimization in \autoref{sec:optim}, with proofs deferred to \autoref{app:optim}.
Building on these analyses, we propose a simple representation learning algorithm for $\phi_L$ in RL in \autoref{sec:algo}.

\subsection{Are Practical \ZP Objectives Biased?}
\label{sec:biased}

\autoref{thm:EG_ZP_implies_Er} suggests that we can learn $\phi_L$ by simply training an auxiliary task of \ZP on a model-free agent. Prior works have proposed several auxiliary losses, summarized in \autoref{tab:main_prior_work}'s \ZP loss column.
Formally, we parametrize an encoder with $f_\phi: \mathcal H_t \to \mathcal Z$ (deterministic case) or $f_\phi: \mathcal H_t \to \Delta(\mathcal Z)$ (probabilistic case)\footnote{Despite being a special case of probabilistic encoders, deterministic encoders deserve distinct discussion because they can be optimal in POMDPs and have been frequently used in prior works. In addition, it should be noted that a probabilistic one may help as it smooths the objective.}.
The latent transition function $P_z(z'\mid z,a)$ is parameterized by $g_\theta: \mathcal Z \times \mathcal A \to \mathcal Z $ (deterministic case) or $g_\theta: \mathcal Z \times \mathcal A \to \Delta(\mathcal Z)$ (probabilistic case). 
We use $\mathbb P_\phi(z\mid h)$ and $ \mathbb P_{\theta}(z' \mid z,a)$ to represent the encoder and latent transition, respectively.
The self-predictive metric for \ZP is \textit{ideally}:
\begin{align}
\label{eq:zp_loss}
\mathcal L_{\ZP,\mathbb D}(\phi,\theta;h,a) \defeq \E{z\sim \mathbb P_\phi(\mid h)}{\mathbb D(\mathbb P_{\theta}(z' \mid z,a)\mid \mid \mathbb P_\phi(z'\mid h,a))},
\end{align}
where $\mathbb P_\phi(z' \mid h, a) = \E{o'\sim P(\cdot \mid h,a)}{\mathbb P_\phi(z' \mid h')}$. $\mathbb D(\cdot\mid \mid \cdot)\in\R_{\ge 0}$ compares two distributions. When \autoref{eq:zp_loss} reaches minimum, then for any $z \sim \mathbb P_\phi(\mid h)$, the \ZP condition is satisfied.

When designing \textit{practical} \ZP loss, prior works are mainly divided into deterministic $\ell_2$ approach~\citep{gelada2019deepmdp,schwarzer2020data,tomar2021learning,hansen2022temporal,ye2021mastering}\footnote{Cosine distance is an $\ell_2$ distance on the normalized vector space $\mathcal Z = \{\ z \in \R^d \mid \|z\|_2=1\}$.
} or probabilistic $\mathtt f$-divergence approach~\citep{zhang2020learning,ghugare2022simplifying,Hafner2019LearningLD} that includes forward and reverse KL divergences (in short, FKL and RKL):\looseness=-1
\begin{align}
\label{eq:l2}
J_\ell(\phi,\theta,\widetilde{\phi}; h,a) &\defeq \E{o'\sim P(\mid h,a)}{\|g_\theta(f_\phi(h),a)-f_{\widetilde{\phi}}(h')\|_2^2},\\
\label{eq:kl}
J_{D_\mathtt f}(\phi,\theta,\widetilde{\phi}; h,a) &\defeq \E{z\sim \mathbb P_{\phi}(\mid h), o'\sim P(\mid h,a)}{D_\mathtt f\left(\mathbb P_{\widetilde{\phi}}(z'\mid h')\mid\mid \mathbb P_{\theta}(z'\mid z,a)\right)},
\end{align}
where $\widetilde{\phi}$, called \ZP \textbf{target}, can be \textbf{online} (exact $\phi$ that allows gradient backpropagation), or the \textbf{stop-gradient} version $\overline \phi$ (detached from the computation graph and using a copy or exponential moving average (EMA) of $\phi$). The update rule is $\overline \phi \gets\tau \overline \phi + (1-\tau)\phi $, with $\tau = 0$ for \textbf{detached} and $\tau \in (0,1)$ for generic \textbf{EMA}. We summarize the choices of \ZP targets in one column of \autoref{tab:main_prior_work}.

We first investigate the relationship between the ideal objective~\autoref{eq:zp_loss}  and practical objectives~\autoref{eq:l2}~and~\autoref{eq:kl} to better understand their implications. 

\begin{proposition}[The practical $\ell_2$ objective~\autoref{eq:l2} is an \textbf{upper bound} of the ideal objective~\autoref{eq:zp_loss} $\mathcal L_{\ZP,\ell}(\phi,\theta;h,a)$ that targets \EZP condition. The equality holds in deterministic environments.]
\label{prop:l2}
\end{proposition}
\begin{proposition}[The practical $\mathtt f$-divergence objective~\autoref{eq:kl} is an \textbf{upper bound} of the ideal objective~\autoref{eq:zp_loss} $\mathcal L_{\ZP,D_\mathtt f}(\phi,\theta;h,a)$ that targets \ZP condition. The equality holds in deterministic environments.]
\label{prop:fdiv}
\end{proposition}\vspace{-0.5em}
These propositions show that environment stochasticity ($P(o'\mid h,a)$) affects both the practical $\ell_2$ and $\mathtt f$-divergence objectives. While unbiased in \textit{deterministic} tasks to learn the \ZP condition\footnote{The \EZP condition is equivalent to the \ZP condition in deterministic tasks.}, they are problematic in \textit{stochastic} tasks (\eg with data augmentation or noisy distractors) due to double sampling issue~\citep{baird1995residual}, as the ideal objective~\autoref{eq:zp_loss} cannot be used (see \autoref{app:optim} for discussion). 

\subsection{Why Do Stop-Gradients Work for \ZP Optimization?}
\label{sec:optim}

In this subsection, we further discuss optimizing the practical \ZP objective~\autoref{eq:l2}. Specifically, we aim to justify that stop-gradient (detached or EMA) \ZP targets, widely used in practice~\citep{schwarzer2020data,zhang2020learning,ghugare2022simplifying}, play an important role in optimization. 
We find that they may lead to \EZP condition in stochastic environments (\autoref{prop:stationary_point}) and can avoid representational collapse under some linear assumptions (\autoref{thm:collapse}). Meanwhile, online \ZP targets lack these properties.

Before introducing the results, we want to clarify the discrepancy between learning \ZP and the well-known TD learning~\citep{sutton1988learning}. 
At first glance, \autoref{eq:l2} (stop-gradient version) is reminiscent of mean-squared TD error, both having the bootstrapping structures. However, \autoref{eq:l2} has an extra challenge due to the missing reward, leading to a trivial solution of constant representation, known as complete representational collapse~\citep{jing2021understanding}.
Moreover, \ZP requires distribution matching, while the Bellman equations that TD learning aims to optimize only require matching expectations. 

\begin{proposition}[The $\ell_2$ objective \autoref{eq:l2} with \textit{stop gradients} ($J_\ell(\phi,\theta,\overline \phi; h,a)$) ensures stationary points that satisfy \EZP, but the $\ell_2$ objective with \textit{online} targets lacks this guarantee.]
\label{prop:stationary_point}
\end{proposition}\vspace{-0.5em}
\autoref{prop:stationary_point} suggests the adoption of stop-gradient targets in \autoref{eq:l2} to preserve the stationary points of \EZP in both deterministic and stochastic tasks.

\begin{theorem}[\textbf{Stop-gradient provably avoids representational collapse in linear models}] 
\label{thm:collapse}
Assume a linear encoder  
$f_\phi(h) \defeq  \phi^\top h_{-k:} \in \R^d$ with parameters $\phi \in \R^{k(|\mathcal O|+|\mathcal A|)\times d}$, which always operates on $h_{-k:}$, a recent-$k$ truncation of history $h$. Assume a linear deterministic latent transition $g_\theta(z,a) \defeq \theta_z^\top z + \theta_a^\top a \in \R^d
$ with parameters $\theta_z \in \R^{d\times d}$ and $\theta_a \in \R^{|\mathcal A|\times d}$. 
If we train $\phi,\theta$ using the stop-gradient $\ell_2$ objective $\E{h,a}{J_\ell(\phi,\theta,\overline \phi;h,a)}$ without RL loss, and $\theta$ relies on $\phi$ by reaching the stationary point with $\nabla_\theta \E{h,a}{J_\ell(\phi,\theta,\overline \phi;h,a)} = 0$, then the matrix multiplication $\phi^\top\phi$ will retain its initial value over continuous-time training dynamics.
\end{theorem}\vspace{-0.5em}

\begin{wrapfigure}{r}{0.5\linewidth}
    \centering
    \vspace{-1em}
    \includegraphics[width=0.49\linewidth]{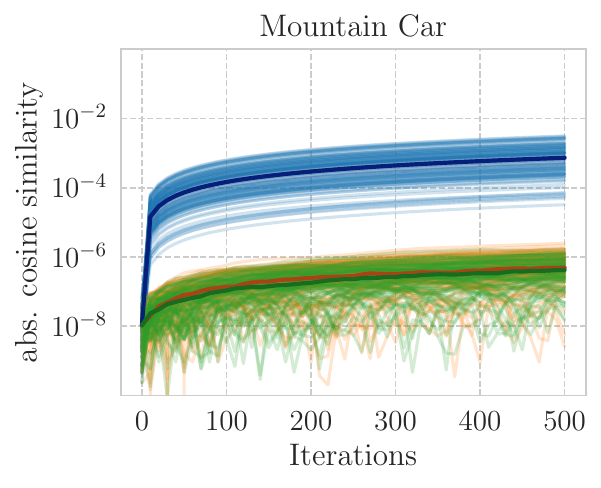}
    \includegraphics[width=0.49\linewidth]{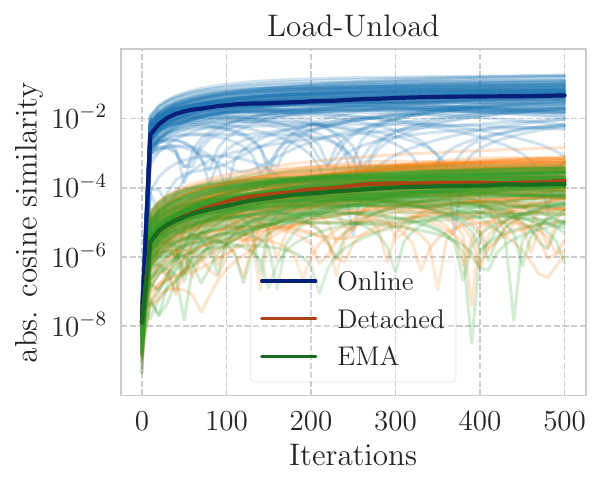}
    \vspace{-2em}
    \caption{\footnotesize The absolute normalized inner product of the two column vectors in the learned encoder when using online, detached, or EMA \ZP target in an MDP \emph{(left)}  and a POMDP \emph{(right)}. We plot the results for 100 different seeds, which controls the rollouts used to sample transition and the initialization of the representation. The bold lines represent the median of the seeds.}
    \label{fig:innerproduct}
    \vspace{-2em}
\end{wrapfigure}
\autoref{thm:collapse} extends the results of~\citep[Theorem~1]{tang2022understanding} to action-dependent latent transition, POMDP, and EMA settings. This theorem also implies that $\phi$ will keep full-rank during training if the initialized $\phi$ is full-rank\footnote{This is due to the fact that $\rank(A^\top A) = \rank(A)$ for any real-valued matrix $A$.}.\looseness=-1

Similar to \citet{tang2022understanding}, we illustrate our theoretical contribution by examining the behavior of the learned encoder over time when starting from a random orthogonal initialization. We extend these results by considering both the MDP and the POMDP setting and consider two classical domains, mountain car~\citep{moore1990efficient} (MDP) and load-unload~\citep{meuleau2013solving} (POMDP), where we fit an encoder $\phi$ with a latent state dimension of $2$. 
\autoref{fig:innerproduct} shows the orthogonality-preserving effect of the stop-gradient by comparing the cosine similarity between columns of the learned $\phi$. As expected by \autoref{thm:collapse}, we see this similarity stay several orders of magnitude smaller when using stop-gradient (detached or EMA) compared to the online case.
Note that although our theory discusses the continuous-time dynamics, we can approximate them with gradient steps with a small learning rate, as was done for these results.

\subsection{A Minimalist RL Algorithm for Learning Self-Predictive Representations}
\label{sec:algo}

\begin{figure}[t]
\vspace{-4.5em}
\begin{algorithm}[H]
    \caption{\textbf{Minimalist $\phi_L$}: learning self-predictive representations in RL}
    \label{pseudo-code}
    \begin{algorithmic}[1]
    \Require Encoder $f_\phi: \mathcal H_t \to \mathcal Z$, Actor $\pi_{\nu}: \mathcal Z \to \mathcal A$, Critic $Q_{\omega}: \mathcal Z \times \mathcal A \to \R$, 
    Latent Transition Model $g_\theta: \mathcal Z \times \mathcal A \to \mathcal Z$. Learning Rate $\alpha > 0$ and Loss Coefficient $\lambda > 0$. 
    \Procedure{Update}{$h,a,o',r$} 
    \State Compute any model-free RL loss $\mathcal L_{\text{RL}}$ (based on DDPG~\citep{lillicrap2015continuous} here) let $Q^{\text{tar}}(h',r) \defeq r + \gamma Q_{\overline{\omega}}(f_{\overline\phi}(h'),\pi_{\overline{\nu}} (f_{\overline\phi}(h')) )$, 
    \vspace{-0.5em}
    \begin{align}
    \label{eq:RL_loss}
    \mathcal L_{\text{RL}}(\phi, \omega, \nu; h',r) = (Q_{\omega}(f_{\phi}(h),a) - Q^{\text{tar}}(h',r))^2 -Q_{\overline{\omega}}(f_{\overline\phi}(h),\pi_{\nu} (f_{\phi}(h))).
    \end{align}
    \vspace{-1em}
    \State Compute the auxiliary \ZP loss 
    $
    \mathcal L_{\text{aux}}(\phi,\theta;h') = \|g_\theta(f_\phi(h),a)-f_{\overline{\phi}} (h')\|_2^2
    $.
    \State Optimize all parameters using the sum of losses:
    \vspace{-0.5em}
    \begin{align}
    \label{eq:update}
    [\phi, \theta, \nu, \omega] \gets [\phi, \theta, \nu, \omega] - \alpha \nabla (\mathcal L_{\text{RL}}(\phi, \omega, \nu; h',r)  + \lambda  \mathcal L_{\text{aux}}(\phi,\theta;h')).
    \end{align}
    \vspace{-1.5em}
    \EndProcedure
    \end{algorithmic}
\end{algorithm}
\vspace{-3em}
\end{figure}

Our theory leads to a straightforward RL algorithm that can target $\phi_L$. Essentially, it integrates a single auxiliary task into any model-free RL algorithm (\eg, DDPG~\citep{lillicrap2015continuous} and R2D2~\citep{kapturowski2018recurrent}), as indicated by \autoref{thm:EG_ZP_implies_Er}. 
Algo.~\ref{pseudo-code} provides the pseudocode for the update rule of all parameters in our algorithm given a tuple of transition data, with PyTorch code included in Appendix. 
The $\ell_2$ \ZP loss can be replaced with KL objective with probabilistic encoder and latent model, especially in stochastic environments, as suggested by \autoref{prop:fdiv}.  
The $f_{\overline \phi}(h')$ in \ZP loss stops the gradient from the encoder, following \autoref{sec:optim}.
The actor loss $-Q_{\overline{\omega}}(f_{\overline\phi}(h),\pi_{\nu} (f_{\phi}(h))$ freezes the parameters of the critic and encoder, suggested by recent work on memory-based RL~\citep{ni2023transformers}.\looseness=-1 

Our algorithm greatly simplifies prior self-predictive methods and enables a fair comparison spanning from model-free to observation-predictive representation learning. 
It is characterized as \textbf{minimalist} by removing reward learning, planning, multi-step predictions, projections, and metric learning. It is also \textbf{novel} by being the first to learn self-predictive representations end-to-end in POMDP literature. 

Our algorithm also bridges model-free and observation-predictive representation learning. 
We derive learning $\phi_{Q^*}$ by setting the coefficient $\lambda$ to $0$, and learning $\phi_O$ by replacing \ZP loss with \OP loss like OFENet~\citep{ota2020can}. 
As a by-product, by comparing $\{\phi_{Q^*}, \phi_L, \phi_O\}$ derived from our algorithm, we can \textbf{disentangle} representation learning from policy optimization because all representations are learned by the same RL algorithm. This is rarely seen in prior works, as learning $\phi_O$ or $\phi_L$ typically involves planning, while model-free RL does not.
Such disentanglement allows us to examine the sample efficiency benefits derived purely from representation learning. \looseness=-1

\section{Experiments}
\label{sec:experiments}

We conduct experiments to compare RL agents learning the three representations $\{\phi_{Q^*}, \phi_L, \phi_O\}$, respectively. 
To decouple representation learning from policy optimization, we follow our minimalist algorithm (Algo.~\ref{pseudo-code}) to learn $\phi_L$, and instantiate $\phi_{Q^*}$ and $\phi_O$ by setting $\lambda = 0$ and replacing \ZP loss with \OP loss, as we discuss in \autoref{sec:algo}.
We evaluate the algorithms in standard MDPs, distracting MDPs\footnote{Distracting MDPs refers to MDPs with distracting observations irrelevant to optimal control in this work.}, and sparse-reward POMDPs. 
The experimental details are shown in \autoref{sec:experiment_details}. 
Through the subsequent experiments, we aim to validate five hypotheses based on our theoretical insights in \autoref{sec:representation}~and~\autoref{sec:learning}, with their motivation shown in \autoref{sec:motivation}.\looseness=-1

\vspace{-0.5em}
\begin{itemize}[leftmargin=*,itemsep=0pt, topsep=0pt]
    \item \textbf{Sample efficiency hypothesis}: do the extra \OP and \ZP signals help $\phi_O$ and $\phi_L$ have better sample-efficiency than $\phi_{Q^*}$ in standard MDPs (\autoref{sec:standard_mdps}) and \textit{especially} in sparse-reward tasks (\autoref{sec:pomdps})?\looseness=-1 
    \item \textbf{Distraction hypothesis} (\autoref{sec:distracted_mdps}): since learning $\phi_O$ may struggle with predicting distracting observations, is it less sample-efficient than learning $\phi_L$?
    \item \textbf{End-to-end hypothesis} (\autoref{sec:pomdps}): as predicted by \autoref{thm:EG_ZP_implies_Er}, is training an encoder end-to-end with an auxiliary task of \ZP (\OP) comparable to the phased training with \RP and \ZP (\OP)? 
    \item \textbf{\ZP objective hypothesis} (\autoref{sec:standard_mdps}, \autoref{sec:distracted_mdps}): as predicted by \autoref{prop:l2} and \autoref{prop:fdiv}, is using $\ell_2$ loss as \ZP objective similar to KL divergence loss in deterministic tasks, but not necessarily stochastic tasks?\looseness=-1
    \item \textbf{\ZP stop-gradient hypothesis} (\autoref{sec:standard_mdps}, \autoref{sec:pomdps}): as predicted by \autoref{thm:collapse}, does stop-gradient on \ZP targets mitigate representational collapse compared to online \ZP targets? 
\end{itemize}

\subsection{State Representation Learning in Standard MDPs}
\label{sec:standard_mdps}

We first evaluate our algorithm in the relatively low-dimensional MuJoCo benchmark~\citep{todorov2012mujoco}. We implement it by simplifying ALM(3)~\citep{ghugare2022simplifying}, a state-of-the-art algorithm on MuJoCo. ALM(3) aims to learn $\phi_L$ end-to-end with EMA \ZP target and reverse KL, shown by \autoref{thm:EG_ZP_implies_Er}.
ALM(3) requires a reward model in the encoder objective and optimizes the actor via SVG~\citep{heess2015learning} with planning for $3$ steps. Following our Algo.~\ref{pseudo-code}, we remove the reward model and multi-step predictions; instead, we train the encoder using our loss (\autoref{eq:l2}~or~\autoref{eq:kl}) and the actor-critic is conditioned on representations and trained using model-free  TD3~\citep{fujimoto2018addressing}, while keeping the other hyperparameters the same. These simplifications result in a $50\%$ faster speed than ALM(3). Please see \autoref{sec:mdp_details} for a detailed comparison between our algorithm and ALM(3). 
\looseness=-1

\textbf{Validation of sample efficiency and \ZP objective hypotheses.}
\autoref{fig:mujoco} shows that our minimalist $\phi_L$ using EMA \ZP targets can attain similar (in Ant) or \textit{even better} (in HalfCheetah and Walker2d) sample efficiency at 500k steps compared to ALM(3), and greatly outperforms $\phi_{Q^*}$, across the entire benchmark except for Humanoid task. 
This suggests that the primary advantage ALM(3) brings to model-free RL in these MuJoCo tasks, lies in state representation rather than policy optimization. This benefit could be further enhanced by streamlining the algorithmic design. In Humanoid task, ALM(3)'s superior performance is likely due to SVG policy optimization's use of first-order gradient information from latent dynamics, which is particularly beneficial in high-dimensional tasks like this.
In line with \autoref{prop:l2} and \autoref{prop:fdiv}, different \ZP objectives ($\ell_2$, FKL, RKL) perform similarly on these tasks, which are nearly deterministic. Thus, these results support our sample efficiency and \ZP objective hypotheses.\looseness=-1

\begin{figure}[t]
    \vspace{-3.5em}
    \centering
    \includegraphics[width=0.24\linewidth]{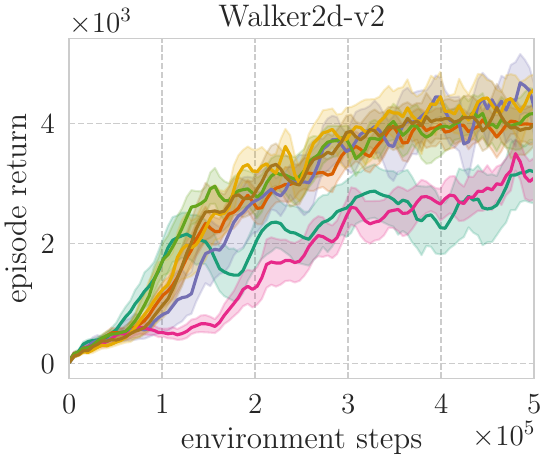} %
    \includegraphics[width=0.25\linewidth]{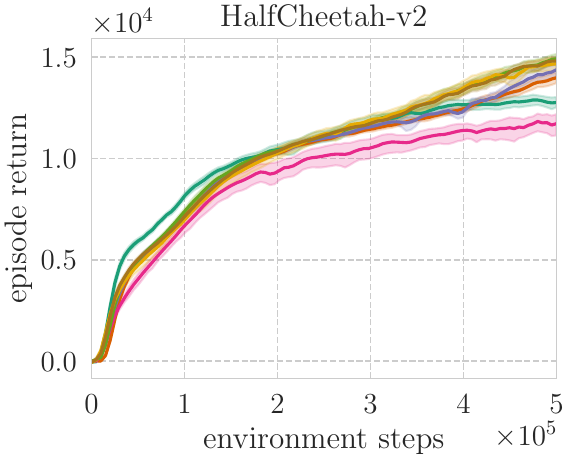} %
    \includegraphics[width=0.24\linewidth]{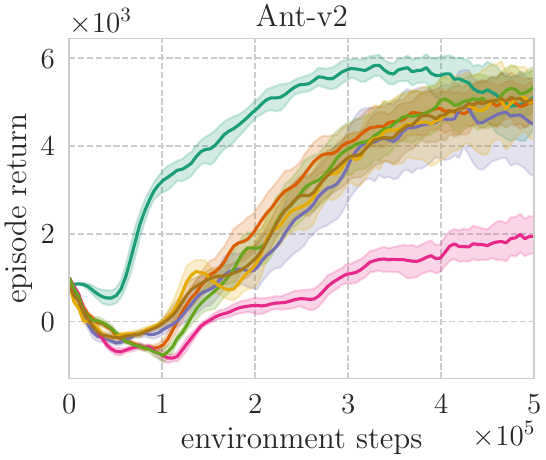} %
    \includegraphics[width=0.24\linewidth]{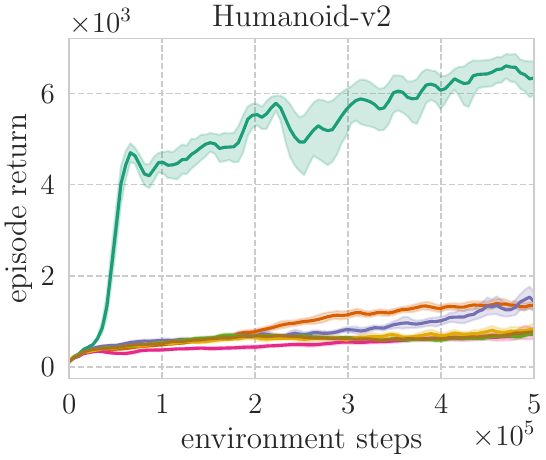} %
    \includegraphics[width=\linewidth]{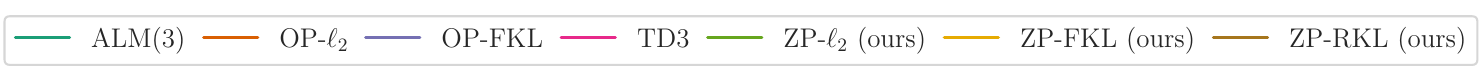} 
    \vspace{-2em}
    \caption{\footnotesize \textbf{Decoupling representation learning from policy optimization using our algorithm based on ALM(3)~\citep{ghugare2022simplifying}.} Comparison between $\phi_{Q^*}$ (TD3), $\phi_L$ (our algorithm (\ZP-$\ell_2$, \ZP-FKL, \ZP-RKL) and ALM(3)), $\phi_O$ (\OP-$\ell_2$, \OP-FKL), in the standard MuJoCo benchmark for 500k steps, averaged over 12 seeds. The observation dimension increases from left figure to right figure ($17,17,111,376$).}
    \label{fig:mujoco}
\end{figure}

\textbf{Validation of \ZP stop-gradient hypothesis.}
We observe a significant performance degradation when switching from the stop-gradient \ZP targets to \textit{online} ones in all MuJoCo tasks for all \ZP objectives ($\ell_2$, FKL, RKL). \autoref{fig:mujoco_optim} (top) shows the results for the $\ell_2$ objective \autoref{eq:l2}. 
To estimate the rank of the associated linearized operator for the MLP encoder, we compute the matrix rank of latent states given a batch of inputs\footnote{The rank of an $m\times n$ matrix $A$ is the dimension of the image of mapping $f:\R^n \to \R^m$, where $f(x) = Ax$.}, where the batch size is $512$ and the latent state dimension is $50$. The MLP encoders are orthogonally initialized~\citep{saxe2013exact} with a full rank of $50$. 
\autoref{fig:mujoco_optim} (bottom) shows that the estimated rank of $\ell_2$ objective with online targets collapses from full rank to low rank, a phenomenon known as \textit{dimensional collapse} in self-supervised learning~\citep{jing2021understanding}. Interestingly, the high dimensionality of a task worsens the rank collapse by comparing from left to right figures. In contrast, stop-gradient targets suffer less from rank collapse, in support to our \ZP stop-gradient hypothesis. For KL objectives, switching to online targets also decreases the rank, though less severely (see \autoref{fig:additional_kl_rank}). \looseness=-1

\begin{figure}[t]
    \centering
    \includegraphics[width=0.24\linewidth]{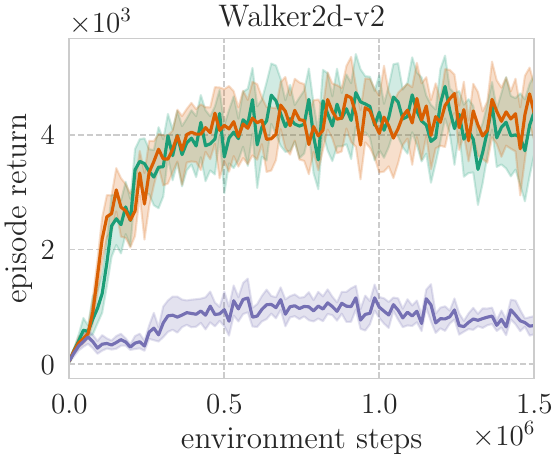}
    \includegraphics[width=0.24\linewidth]{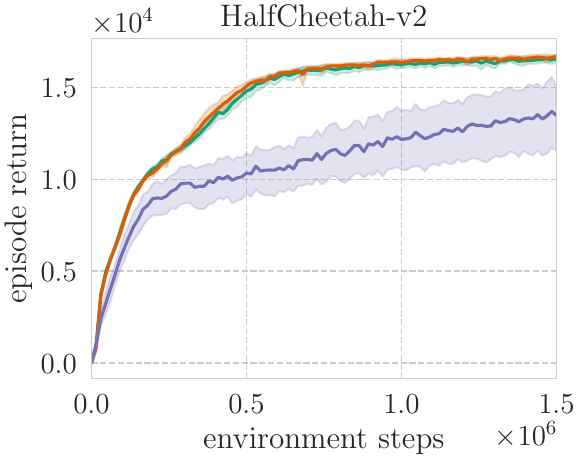}
    \includegraphics[width=0.26\linewidth]{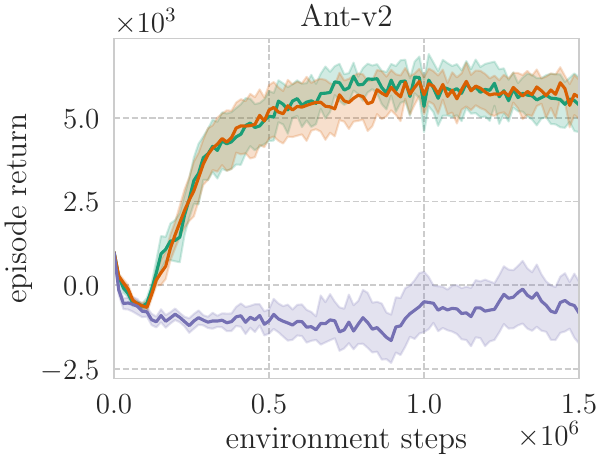}
    \includegraphics[width=0.24\linewidth]{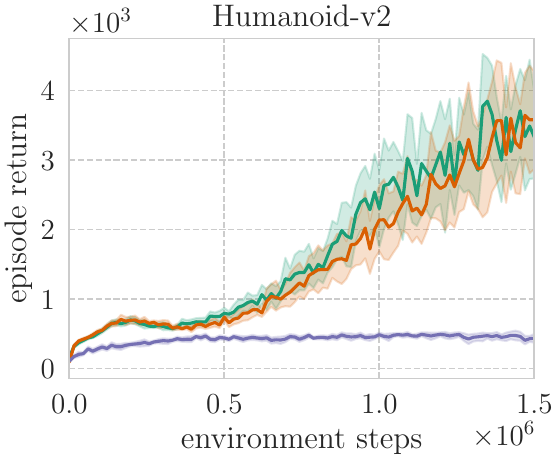}
    
\includegraphics[width=0.24\linewidth]{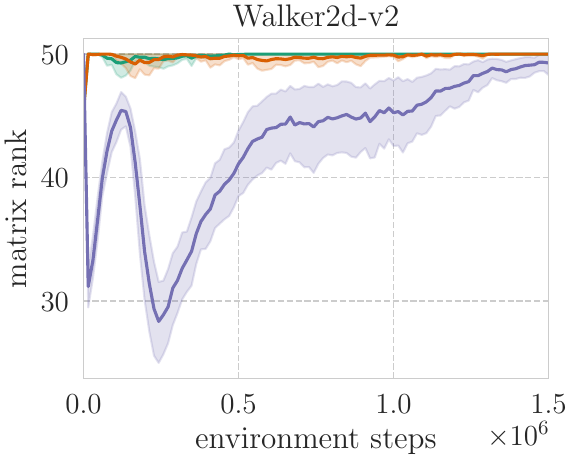}
    \includegraphics[width=0.24\linewidth]{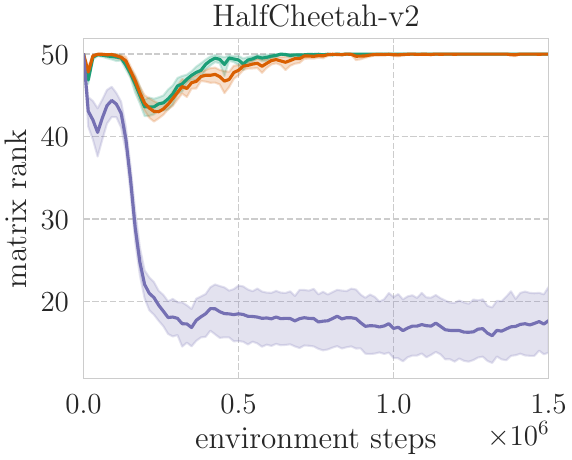}
    \includegraphics[width=0.24\linewidth]{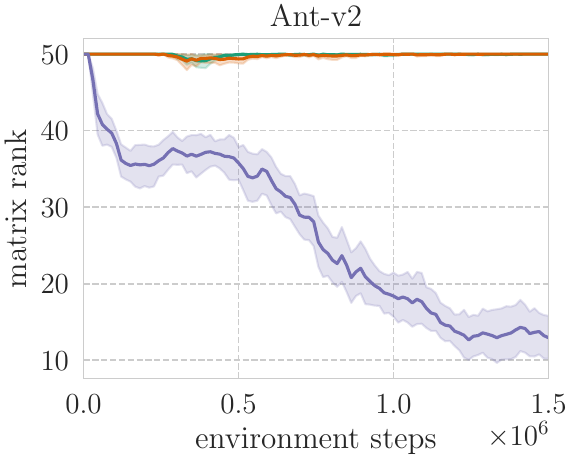}
    \includegraphics[width=0.24\linewidth]{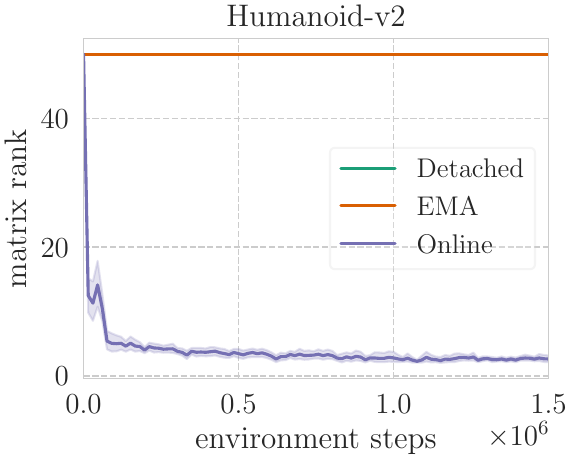}
    \vspace{-0.5em}
    \caption{\footnotesize \textbf{Representation collapse with online targets.} On four benchmark tasks, we observe that using the online \ZP target in $\ell_2$ objectives results in lower returns \emph{(top)}  and low-rank representations \emph{(bottom)}. In line with our theory, using a detached or EMA \ZP target mitigates the representational collapse and yields higher returns.
    }
    \label{fig:mujoco_optim}
    \vspace{-1.5em}
\end{figure}

\begin{figure}[t]
\vspace{-3.5em}
    \centering
    \includegraphics[width=0.25\linewidth]{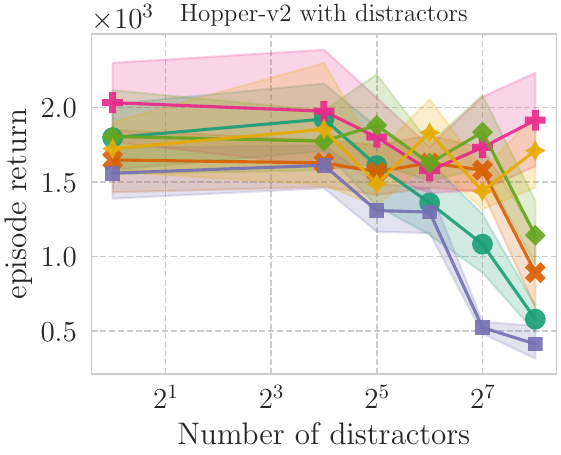}
    \includegraphics[width=0.24\linewidth]{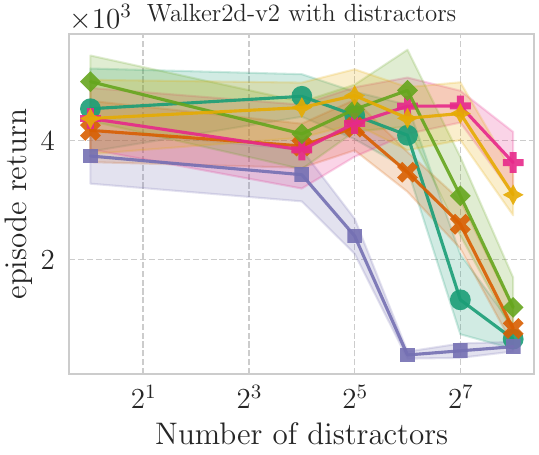}
    \includegraphics[width=0.25\linewidth]{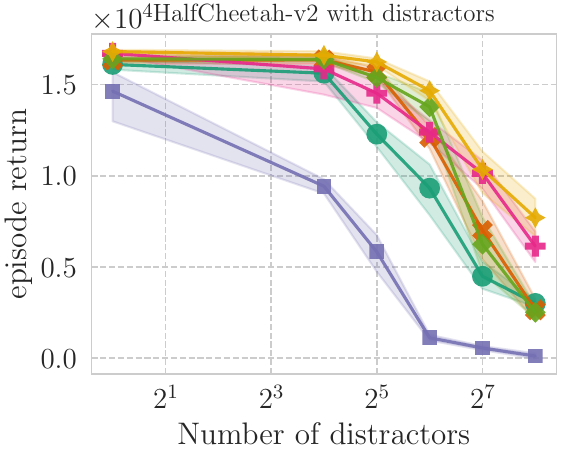}
    \includegraphics[width=0.24\linewidth]{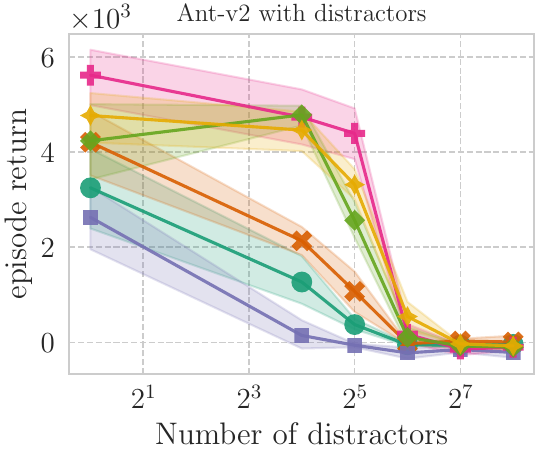}
    \includegraphics[width=\linewidth]{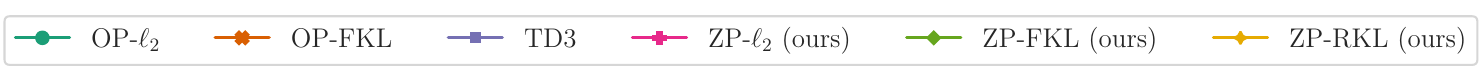}
    \vspace{-2em}
    \caption{\footnotesize \textbf{Self-predictive representations are more robust.} Comparison between $\phi_{Q^*}$ (TD3), $\phi_L$ (\ZP-$\ell_2$, \ZP-FKL, \ZP-RKL) using our algorithm, $\phi_O$ (\OP-$\ell_2$, \OP-FKL) in the \textbf{distracting} MuJoCo benchmark, varying the distractor dimension from $2^4$ to $2^8$, averaged over 12 seeds. The y-axis is final performance at 1.5M steps.\looseness=-1}
    \label{fig:mujoco_distracted}
\end{figure}

\subsection{State Representation Learning in Distracting MDPs} 
\label{sec:distracted_mdps}

We then evaluate the robustness of representation learning by augmenting states with distractor dimensions in the MuJoCo benchmark. The distractors are i.i.d. standard isotropic Gaussians, varying the number of dimensions from $2^4$ to $2^8$, following the practice in \citet{nikishin2022control}. The distracting task has the same optimal return as the original one and can be challenging to RL algorithms, even if model-based RL could perfectly model Gaussians. We use the same code in standard MDPs. \autoref{fig:mujoco_distracted} shows the final averaged returns of each algorithm (variant) in the distracting MuJoCo benchmark.\looseness=-1  

\textbf{Validation of distraction hypothesis.}
By comparing \ZP with \OP objectives, with higher-dimensional distractors, learning $\phi_O$ degrades much faster than learning $\phi_L$, verifying our distraction hypothesis. Surprisingly, model-free RL ($\phi_{Q^*}$) performs worse than $\phi_O$, as $\phi_{Q^*}$ does not need to predict the distractors. However, we still observe a severe degradation in Ant for all methods when there are $128$ distractors, which we will study in the future work.\looseness=-1

\textbf{Extending \ZP objective hypothesis to stochastic tasks.}
In stochastic tasks, \autoref{prop:l2} and \autoref{prop:fdiv} tell us about the (strict) upper bounds for learning \EZP and \ZP conditions with the practical $\ell_2$ and KL objectives, respectively. Yet, they do not indicate which objective is better for stochastic tasks based on these bounds.
In fact, we observe that both $\ell_2$ and reverse KL perform better than forward KL in the distracting tasks, possibly because the entropy term in reverse KL smooths the training objective, and $\ell_2$ objective simplifies the learning.\looseness=-1

\subsection{History Representation Learning in Sparse-Reward POMDPs}
\label{sec:pomdps}

\begin{figure}[t]
    \centering
    \includegraphics[width=0.49\linewidth]{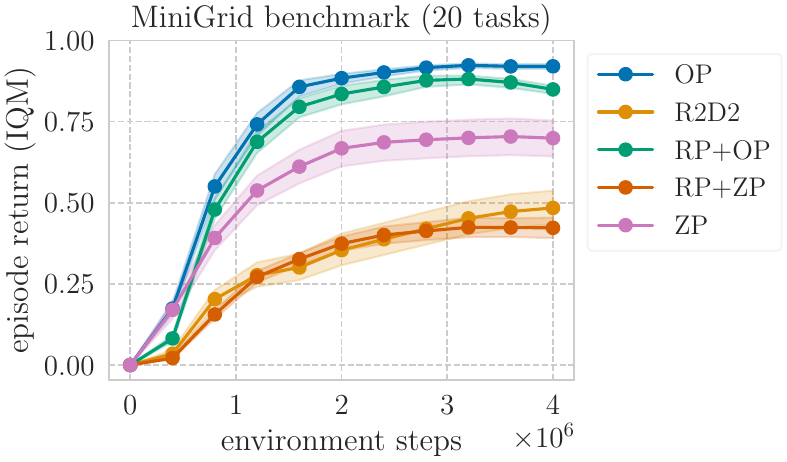}
    \includegraphics[width=0.49\linewidth]{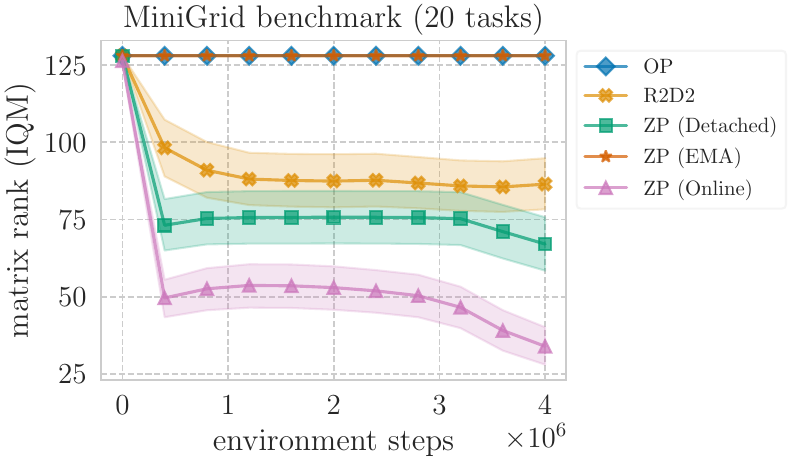}
    \vspace{-1em}
    \caption{\footnotesize\textbf{End-to-end learned self-predictive and observation-predictive representations stand out in sparse-reward tasks.} We show the interquartile mean (IQM)~\citep{agarwal2021deep} across \textbf{20 MiniGrid tasks}, computed over 9 seeds per task with 95\% stratified bootstrap confidence intervals. The individual task plots are shown in Appendix \autoref{fig:minigrid_full} and \autoref{fig:ZP_optim}. \textbf{Left}: comparison of episode returns between $\phi_{Q^*}$ (R2D2), $\phi_L$ (\ZP, \RP + \ZP), $\phi_O$ (\OP, \RP + \OP). \RP + \ZP and \RP + \OP methods are phased, while \ZP and \OP methods are end-to-end. \textbf{Right}: comparison of estimated matrix rank between \ZP targets (online, detached, EMA), R2D2, and \OP. The maximal achievable rank is $128$.  }
    \vspace{-1.5em}
    \label{fig:minigrid}
\end{figure}

Finally, we perform an extensive empirical study on 20 MiniGrid tasks~\citep{gym_minigrid}. These tasks, featuring partial observability and sparse rewards, serve as a rigorous test-bed for \textit{history} representation learning. The rewards are only non-zero upon successful task completion. Episode returns are between 0 and 1, with higher returns indicating faster completion. 
Each task has an observation space of $7\times 7 \times 3 = 147$ dimensions and a discrete action space of $7$ options.  
Our model-free baseline is R2D2~\citep{kapturowski2018recurrent}, composed of double Q-learning~\citep{van2016deep} with LSTMs~\citep{hochreiter1997long} as history encoders. 
We implement it based on recent work~\citep{seyedsalehi2023approximate}. 
R2D2 uses stored LSTM hidden states for initialization, a $50$-step burn-in period, and a $10$-step history rollout during training. This result in a high dimensionality of $(50+10)\times (147+7)=9240$ on histories. 

Based on R2D2, we implement our algorithm by adding an auxiliary task of \ZP (or \OP) under $\ell_2$ objectives, which is sufficient for solving these deterministic tasks (see \autoref{prop:l2}).
To compare the phased training~\citep{seyedsalehi2023approximate}, we further implement methods (\RP + \OP, \RP + \ZP) that also predict rewards and freeze the encoders during Q-learning. Due to the space limit, we show the aggregated plots in \autoref{fig:minigrid} and defer the individual task plots to \autoref{sec:ablation}.

\textbf{Validation of sample efficiency hypothesis.} By examining the aggregated learning curves of end-to-end methods in \autoref{fig:minigrid} left, we find that minimalist $\phi_L$ (\ZP) significantly outperforms $\phi_{Q^*}$ (R2D2) on average. It, however, fall shorts of $\phi_O$ (\OP), which aligns with our expectations given that MiniGrid tasks are deterministic with medium-dimensional clean observations. The enhanced performance of $\phi_L$ and $\phi_O$ over $\phi_{Q^*}$ provides empirical validation of our hypothesis in sparse-reward tasks.\looseness=-1 

\textbf{Validation of end-to-end hypothesis.} By comparing the aggregated learning curves between end-to-end methods and phased methods (\ie, \OP vs \RP + \OP, and \ZP vs \RP + \ZP) in \autoref{fig:minigrid} left, we observe that end-to-end training (\OP or \ZP) yields equal or superior sample efficiency relative to phased training (\RP + \OP or \RP + \ZP) when learning observation-predictive or self-predictive representations. This supports our hypothesis, and is particularly noticeable when end-to-end learning $\phi_L$ (\ZP) markedly excels over its phased learning counterpart (\RP + \ZP).

\textbf{Validation of \ZP stop-gradient hypothesis.} We extend our rank analysis from MuJoCo to MiniGrid using the same estimation metric. As depicted in \autoref{fig:minigrid} right, our finding averaged across the benchmark indicates that EMA \ZP targets are able to preserve their rank, while both detached and online \ZP targets degrade the rank during training. Notably, while our theory does not distinguish detached and EMA targets, the observed lower rank of online targets relative to both detached and EMA targets is in line with our hypothesis. Finally, without any auxiliary task such as \ZP and \OP, R2D2 degrades the rank, which is predictable in the sparse-reward setting~\citep{lyle2021effect}.

\section{Discussion}
\label{sec:recommend}

\textbf{Recommendations.} Based on our theoretical and empirical results, we suggest the following preliminary guidance to RL practitioners:\looseness=-1
\begin{enumerate}[leftmargin=*,itemsep=0pt, topsep=0pt]
\item \textbf{Analyze your task first.} For example, in noisy or distracting tasks, consider using self-predictive representations. In sparse-reward tasks, consider using observation-predictive representations. In deterministic tasks, choose the deterministic $\ell_2$ objectives for representation learning.  
\item \textbf{Use our minimalist algorithm as your baseline.} Our algorithm allows for an independent evaluation of representation learning and policy optimization effects.  Start with end-to-end learning and model-free RL for policy optimization. 
\item \textbf{Implementation tips.} For our minimalist algorithm, we recommend adopting the $\ell_2$ objective with EMA \ZP targets first. When tackling POMDPs, start with recurrent networks as the encoder.
\end{enumerate}

\textbf{Limitations.} The limitations of our work can be divided into theoretical and empirical aspects. 
On the theoretical side, although we show a continuous-time analysis of auxiliary learning dynamics with linear models in \autoref{thm:collapse}, we do not provide a convergence analysis for the joint optimization of RL and auxiliary losses and the results may not hold beyond linear assumption.
On the empirical side, our experiment scope does not cover more complicated domains that require pixel-based observations.  

\textbf{Conclusion.} This work has offered a principled analysis of state and history representation learning in reinforcement learning, bridging the gap between various approaches.
Our unified view and analysis of self-predictive learning also inspire a minimalist RL algorithm for learning self-predictive representations. Extensive empirical studies in benchmarks across standard MDPs, distracting MDPs, and sparse-reward POMDPs, validate most of our hypotheses suggested by our theory.

\section*{Acknowledgements and Disclosure of Funding}

We thank Pierluca D'Oro and Zhixuan Lin for their technical help. We thank Amit Sinha, David Kanaa, David Yu-Tung Hui, Dinghuai Zhang, Doina Precup, Léo Gagnon, Pablo Samuel Castro, Raj Ghugare, Shreyas Chaudhari, Ziyan Luo, and anonymous reviewers for the constructive discussion.
This work was enabled by the computational resources provided by the Calcul Québec (\url{www.calculquebec.ca}) and the Digital Research Alliance of Canada (\url{https://alliancecan.ca/}), with material support from NVIDIA Corporation.
This work was funded by IBM Research and Google DeepMind. 

\newpage
{\small
\bibliography{citation}
\bibliographystyle{iclr2024_conference}
}

\clearpage
\appendix

\part{Appendix}
{
\hypersetup{linkcolor=black}
\parttoc
}

\section{A Unified View on State and History Representations}
\label{sec:unified_view}
\subsection{Notation}
\autoref{tab:notation} shows the glossary used in this paper.

\begin{table}[h]
    \centering
    \caption{\textbf{Glossary of notations} used in this paper.}
    \footnotesize
    \begin{tabular}{ccc}
    \toprule
    \textbf{Notation} & \textbf{Text description} & \textbf{Math description} \\ 
    \midrule
        $\gamma$ & Discount factor & $\gamma \in [0,1]$ \\
        $T$ & Horizon & $T\in \mathbb N \cup \{+\infty\}$ \\ 
        $s_t$ & State at step $t$ & $s\in \mathcal S$ \\
        $o_t$ & Observation at step $t$ & $o\in \mathcal O$ \\ 
        $a_t$ & Action at step $t$ & $a\in \mathcal A$ \\ 
        $r_t$ & Reward at step $t$ & $r\in \R$ \\ 
        $h_t$ & History at step $t$ & $h_t = (h_{t-1},a_{t-1},o_t) \in \mathcal H_t$, $h_1=o_1$ \\
        $P(o_{t+1} \mid h_t,a_t)$ & Environment transition &   \\
        $R(h_t,a_t)$ & Environment reward function & $R: \mathcal H_t \times \mathcal A \to \Delta(\mathbb R)$ \\
        $\pi(a_t \mid h_t)$ & Policy (actor) &  \\ 
        $\pi^*(h_t)$ & Optimal policy (actor) &  \\ 
        $Q^\pi(h_t,a_t)$ & Value (critic) & \\ 
        $Q^*(h_t,a_t)$ & Optimal value (critic) & \\ 
        \midrule
        $\phi$ & Encoder of history & $\phi: \mathcal H_t \to \mathcal Z$ \\
        $z_t$ & Latent state at step $t$ & $z_t = \phi(h_t) \in \mathcal Z$ \\ 
        $P_z(z_{t+1} \mid z_t,a_t)$ & Latent transition & \\
        $R_z(z_t,a_t)$ & Latent reward function & \\ 
        $\pi_z(a_t \mid z_t)$ & Latent policy (actor) &  \\ 
        $\pi_z^*(z_t)$ & Optimal latent policy (actor) &  \\ 
        $Q_z^{\pi_z}(z_t,a_t)$ & Latent value (critic) & \\ 
        $Q_z^*(z_t,a_t)$ & Optimal latent value (critic) & \\ 
        \midrule\midrule
        \RP & Expected \textbf{R}eward \textbf{P}rediction & $\E{}{r_t \mid h_t, a_t} = R_z(\phi(h_t),a_t)$ \\
       \midrule
       \OR & \textbf{O}bservation \textbf{R}econstruction & $o_t = \psi_o(\phi(h_t))$ \\ 
       \OP & Next \textbf{O}bservation \textbf{P}rediction & $P(o_{t+1} \mid h_t,a_t) = P_o(o_{t+1} \mid \phi(h_t),a_t)$ \\
       \midrule
       \ZP & Next Latent State $z$ \textbf{P}rediction & $P(z_{t+1} \mid h_t,a_t) = P_z(z_{t+1} \mid \phi(h_t),a_t)$ \\ 
       \EZP & \textbf{E}xpected Next Latent State $z$ \textbf{P}rediction & $\E{}{z_{t+1} \mid h_t,a_t} = \E{}{z_{t+1} \mid \phi(h_t),a_t}$  \\ 
       \Rec & Recurrent Encoder & $\phi(h_{t+1})= \psi_z(\phi(h_t), a_t,o_{t+1})$ \\
       ZM & Markovian Latent Transition & $z_{t+1} \ind z_{1:t-1},a_{1:t-1} \mid \phi(h_t),a_t$ \\
       \midrule\midrule
       $\phi_{\pi^*}$ & $\pi^*$-irrelevance abstraction & $\phi(h_1) = \phi(h_2) \implies \pi^*(h_1) = \pi^*(h_2)$\\
      $\phi_{Q^*}$ & $Q^*$-irrelevance abstraction & $\phi(h_1) = \phi(h_2) \implies Q^*(h_1,a) = Q^*(h_2,a)$ \\ 
      $\phi_M$ & Markovian abstraction & \RP + ZM  \\ 
       $\phi_L$ & Self-predictive abstraction & \RP + \ZP $\iff$ $\phi_{Q^*}$ + \ZP \\ 
       $\phi_O$ & Observation-predictive abstraction & \RP + \OP + \Rec $\iff$  $\phi_{Q^*}$ + \OP + \Rec \\ 
       \bottomrule
    \end{tabular}
    \label{tab:notation}
\end{table}

\subsection{Additional Background}
\label{sec:more_def}

\paragraph{Remark on the latent state distribution.} In this paper, we assume the latent space $\mathcal Z$ as a pre-specified Banach space, which is a complete normed vector space. We further assume any latent state distribution defined on $\mathcal Z$ has a finite expectation. 
To avoid a measure-theoretic treatment, we assume that Z is discrete-valued in our proof analysis. The proof arguments are easily generalized to the case when $\mathcal Z$ lies in a Banach space using standard arguments.

\paragraph{Remark on the existence of optimal value and policy in POMDPs.} In an MDP, it is well-known that there exists a unique optimal value function following the Bellman equation, which induces an optimal deterministic policy~\citep{puterman2014markov}. In POMDPs, the result is complicated. For a \textit{finite-horizon} POMDP, one can construct a finite-dimensional state space by stacking all previous observations and actions to convert a POMDP into an MDP, thus the MDP result can be directly applied.  For  an \textit{infinite-horizon} POMDP, \citet[Theorem 25]{subramanian2022approximate} shows that the unique optimal value function exists when the POMDP has a time-invariant finite-dimensional information state, which is the case when the unobserved state space is finite.  The POMDP experiments shown in \autoref{sec:pomdps} satisfy this assumption because they have a finite state space. For POMDPs with infinite-dimensional information states, the result remains unclear.

\subsubsection{Additional Abstractions and Formalizing the Relationship}

First, we present two additional abstractions not shown in the main paper, which are also used in prior work.
Then we formalize \autoref{thm:hierarchy_informal} with \autoref{thm:hierarchy} using the concept of granularity in relation.

\textbf{$\pi^*$-irrelevance abstraction.}
An encoder $\phi_{\pi^*}$ yields a $\pi^*$-irrelevance abstraction~\citep{li2006towards} if it contains the necessary information (a ``sufficient statistics'') for selecting return-maximizing actions. 
Formally, if $\phi_{\pi^*}(h_i) = \phi_{\pi^*}(h_j)$ for some $h_i,h_j\in\mathcal H_t$, then
$
\pi^*(h_i) = \pi^*(h_j)
$.
One way of obtaining a $\pi^*$-irrelevance abstraction is to learn an encoder $\phi$ end-to-end with a policy $\pi_z(a \mid \phi(h))$ by model-free RL~\citep{sutton1999policy} such that $\pi_z^*(\phi_{\pi^*}(h)) = \pi^*(h),\forall h$.  

\textbf{Markovian abstraction}. An encoder  $\phi_M$ provides Markovian abstraction if it satisfies the expected reward condition \RP and \textbf{Markovian latent transition (ZM)} condition: for any $z_k = \phi_M(h_k)$, 
\begin{align}
\label{eq:ZM}
P(z_{t+1} \mid z_{1:t},a_{1:t}) = P(z_{t+1} \mid z_t,a_t), \quad \forall z_{1:t+1}, a_{1:t}.  \tag{ZM}
\end{align}
\newcommand{\ZM}{\ref{eq:ZM}\xspace}%

This extends Markovian abstraction~\citep{allen2021learning} in MDPs to POMDPs.

\textbf{Granularity in relation}.
In MDPs, it is well-known that state representations form a hierarchical structure~\citep[Theorem~2]{li2006towards}, but this idea had not been extended to the POMDP case. We do so here by defining an equivalent concept of ``granularity''. We say that an encoder $\phi_A$ is finer than or equal to another encoder $\phi_B$, denoted as $\phi_A \succeq \phi_B$, if and only if for any histories $h_i,h_j\in \mathcal H_t$, 
$\phi_A(h_i) = \phi_A(h_j)$
implies 
$\phi_B(h_i) = \phi_B(h_j)$.  
The relation $\succeq$ is a partial ordering.
Using this notion, we can show \autoref{thm:hierarchy}:

\begin{theorem}[\textbf{Granularity of state and history abstractions (the formal version of \autoref{thm:hierarchy_informal})}]
$\phi_O \succeq \phi_L \succeq \phi_{Q^*} \succeq \phi_{\pi^*}$.
\label{thm:hierarchy}
\end{theorem}

\textbf{Abstract MDP.} Given an encoder $\phi$, we can construct an abstract MDP~\citep{li2006towards} $\mathcal M_\phi = (\mathcal Z, \mathcal A, P_z,R_z,\gamma, T)$ for a POMDP $\mathcal M_O$. The latent reward $R_z$ and latent transition $P_z$ are then given by: 
$
R_z(z,a) = \int P(h \mid z) \E{}{r \mid h,a} dh
$,
$
P_z(z' \mid z,a) =\int  P(h\mid z) P(o' \mid h,a) \delta(z'=\phi(h')) dhdo'
$, where $P(h\mid z) = 0 $ for any $\phi(h) \neq z$ and is normalized to a distribution. 
The optimal latent (Markovian) value function $Q^*_z(z,a)$ statisfies
$
Q^*_z(z,a) = R_z(z,a) + \gamma \E{z'\sim P_z( \mid z,a)}{\max_{a'} Q^*_z(z',a')}
$, and the optimal latent policy $\pi_z^*(z) = \argmax_a Q^*_z(z,a)$.
It is important to note that this definition focuses solely on the process by which the encoder induces a corresponding abstract MDP, without addressing the quality of the encoder itself.

\subsubsection{Alternative Definitions}

In the main paper (\autoref{sec:background}), we present the concepts of self-predictive abstraction $\phi_L$ and observation-predictive abstraction $\phi_O$. 
In most prior works, these concepts were defined in an alternative way -- using a pair of states (histories). In comparison, our definition is based on a pair of a state (history) and a latent state, which we believe is more comprehensible and help derive the auxiliary objectives. 

For completeness, here we restate their definition, extended to POMDPs, and then show the equivalence between their and our definitions.

\textbf{Model-irrelevance abstraction~\citep{li2006towards} (\textbf{bisimulation relation}~\citep{givan2003equivalence}) $\Phi_L$.}
If for any two histories $h_i,h_j\in\mathcal H$ such that $
\Phi_{L}(h_i) = \Phi_{L}(h_j)
$, then 
\begin{align}
\label{eq:RP_bisimulation}
&\E{}{r \mid h_i, a} = \E{}{r \mid h_j, a}, \quad \forall a\in \mathcal A, \\ 
\label{eq:ZP_bisimulation}
&P(z' \mid h_i, a) = P(z' \mid h_j ,a), \quad \forall a\in \mathcal A, z'\in \mathcal Z,
\end{align}
where $P(z' \mid h,a) = \int P(o' \mid h,a) \delta (z' = \Phi_L(h'))do'$. 
Here we extend the concept from MDPs~\citep{li2006towards,givan2003equivalence} into POMDPs. 
It is worth noting that while original concepts assume deterministic rewards or require reward distribution matching for stochastic rewards~\citep{castro2009equivalence} in \autoref{eq:RP_bisimulation}, the requirement can indeed be relaxed. As shown by \citet{subramanian2022approximate}, it is sufficient to ensure expected reward matching to maintain optimal value functions. As such, we adopt this relaxed requirement of expectation matching in our concept. 

\begin{proposition}[\textbf{$\Phi_L$ is equivalent to $\phi_L$}]
\label{prop:info_state_MDPs}
\end{proposition}
\begin{proof}
It is easy to see that $\phi_L$ implies $\Phi_L$. If $\phi_{L}(h_i) = \phi_{L}(h_j)$, then by \RP, 
\begin{align}
\E{}{r \mid h_i,a} = R_z (\phi_L(h_i),a) = R_z (\phi_L(h_j),a) = \E{}{r\mid h_j, a},
\end{align}
and by \ZP, 
\begin{align}
P(z'\mid h_i,a) = P_z(z' \mid \phi_L(h_i), a) = P_z(z' \mid \phi_L(h_j), a) = P(z'\mid h_j,a).
\end{align}
Therefore, $\phi_L$ implies $\Phi_L$. 

Now we want to show $\Phi_L$ implies $\phi_L$. We will use the following fact: for any two random variables $X, Y$ and a function $f$ that maps $Y$ into a random variable $Z$, we have $X \ind f(Y) \mid Y$. This is equivalent to say: 
\begin{align}
\label{eq:add_func_prob}
P(X =x \mid Y =y) = P(X=x \mid Y=y, Z=f(y)), \quad \forall x,y. 
\end{align}
A corollary is on the conditional expectation:
\begin{align}
\label{eq:add_func_exp}
\E{}{X \mid Y =y} = \E{}{X \mid Y=y, Z=f(y)}.
\end{align}

First, to see \RP condition: using the fact (\autoref{eq:add_func_exp}), 
\begin{align}
\E{}{r \mid \mathcal H = h_i, \mathcal A = a} = \E{}{r \mid \mathcal H = h_i, \mathcal A = a, \mathcal Z = \phi(h_i)}. 
\end{align}
By \autoref{eq:RP_bisimulation}, we have for any $h_i, h_j$ such that $\phi(h_i)=\phi(h_j) \defeq z$,
\begin{align}
\E{}{r \mid \mathcal H = h_i, \mathcal A = a, \mathcal Z = z} =  \E{}{r \mid \mathcal H = h_j, \mathcal A = a, \mathcal Z = z}.
\end{align}
This exactly indicates \RP condition: $\E{}{r \mid \mathcal H = h_i, \mathcal A = a}$ is a function of $\phi(h_i), a$.

Similar to the proof of showing \RP, we can show \ZP: using the fact (\autoref{eq:add_func_prob}), 
\begin{align}
P( z' \mid \mathcal H = h_i, \mathcal A = a) = P( z' \mid \mathcal H = h_i, \mathcal A = a, \mathcal Z = \phi(h_i)).
\end{align}
By \autoref{eq:ZP_bisimulation}, we have for any $h_i, h_j$ such that $\phi(h_i)=\phi(h_j) \defeq z$,
\begin{align}
P( z' \mid \mathcal H = h_i, \mathcal A = a, \mathcal Z = z)=  P( z' \mid \mathcal H = h_j, \mathcal A = a, \mathcal Z = z).
\end{align}
This exactly indicates \ZP condition: $P( z' \mid \mathcal H = h_i, \mathcal A = a)$ is a distribution conditioned on $\phi(h_i), a$.

\end{proof}

\textbf{Belief abstraction $\Phi_{O}$ (weak belief bisimulation relation~\citep{castro2009equivalence}).} It satisfies \Rec, and 
if for any two histories $h_i,h_j\in\mathcal H$ such that
$
\Phi_{O}(h_i) = \Phi_{O}(h_j)
$, then 
\begin{align}
&\E{}{r\mid h_i,a} = \E{}{r\mid h_j,a}, \quad \forall a\in \mathcal A, \\
&P(o' \mid h_i, a) = P(o' \mid h_j ,a), \quad \forall a\in \mathcal A, o'\in \mathcal O.
\end{align}
This concept is known as a naive abstraction in MDPs~\citep{jiang2018notes} and weak belief bisimulation relation in POMDPs~\citep{castro2009equivalence}. Similarly, prior concepts assume deterministic reward or distribution matching for stochastic rewards, while we relax it to expected reward matching.

\begin{proposition}[\textbf{$\Phi_O$ is equivalent to $\phi_O$}]
\label{prop:belief_state_MDPs}
\end{proposition}
\begin{proof}
The proof is almost the same as the proof of \autoref{prop:info_state_MDPs} by replacing $z'$ with $o'$.
\end{proof}

\subsection{Propositions and Proofs}
\label{sec:proof_relation}

With the additional background in \autoref{sec:more_def}, we show the complete implication graph in \autoref{fig:relation} built on \autoref{fig:relation_main}. 

\begin{figure}[h]
    \centering
    \begin{minipage}[h]{0.35\linewidth}
    \begin{tikzpicture}[>=latex',line join=bevel,very thick,scale=0.5]
\node (ZP) at (164.1bp,126.44bp) [draw,fill=yellow,circle] {\ZP};
  \node (Q) at (255.48bp,73.681bp) [draw,circle] {$\phi_{Q^*}$};
  \node (RP) at (164.1bp,20.924bp) [draw,circle] {\RP};
  \node (ZM) at (237.57bp,238.24bp) [draw,ellipse] {\ZM};
  \node (Rec) at (117.32bp,225.22bp) [draw,circle] {\Rec};
  \node (OP) at (21.412bp,91.928bp) [draw,circle] {\OP};
  \node (Pi) at (356.26bp,73.681bp) [draw,circle] {$\phi_{\pi^*}$};
  \node (OR) at (44.81bp,162.5bp) [draw,circle] {\OR};
  \draw [brown,->] (ZP) ..controls (193.56bp,109.43bp) and (208.66bp,100.71bp)  .. (Q);
  \draw [orange,->] (ZP) ..controls (164.1bp,90.645bp) and (164.1bp,69.69bp)  .. (RP);
  \draw [magenta,->] (ZP) ..controls (187.8bp,162.5bp) and (207.24bp,192.09bp)  .. (ZM);
  \draw [gray,->] (ZP) ..controls (142.64bp,155.79bp) and (132.91bp,175.63bp)  .. (Rec);
  \draw [red,->] (ZP) ..controls (122.3bp,109.04bp) and (81.565bp,98.908bp)  .. (OP);
  \draw [cyan,->] (Q) ..controls (293.72bp,73.681bp) and (305.79bp,73.681bp)  .. (Pi);
  \draw [orange,->] (Q) ..controls (222.81bp,46.823bp) and (207.29bp,37.779bp)  .. (RP);
  \draw [brown,->] (RP) ..controls (190.03bp,43.629bp) and (205.49bp,52.956bp)  .. (Q);
  \draw [blue,->] (Rec) ..controls (140.49bp,192.52bp) and (150.03bp,172.74bp)  .. (ZP);
  \draw [blue,->] (OP) ..controls (64.509bp,109.67bp) and (105.25bp,119.76bp)  .. (ZP);
  \draw [transparent,->] (OR) ..controls (85.453bp,150.21bp) and (113.37bp,141.77bp)  .. (ZP);
  \draw [red,->] (OR) ..controls (35.687bp,134.98bp) and (33.506bp,128.41bp)  .. (OP);

\end{tikzpicture}
    \end{minipage}
    \vspace{-1em}
\caption{\textbf{The complete implication graph} showing the relations between the conditions on history representations. The source nodes of the edges with the same color together imply the target node. In MDPs, \OR implies all the other conditions. As a quick reminder, \RP: expected reward prediction, \OP: next observation prediction, \OR: observation reconstruction, \ZP: next latent state prediction, \Rec: recurrent encoder, \ZM: Markovian latent transition. 
All the connections are discovered in this work, except for (1) \OP + \Rec implying \ZP, (2) \ZP + \RP implying $\phi_{Q^*}$, (3) $\phi_{Q^*}$ implying $\phi_{\pi^*}$.}
\label{fig:relation}
\end{figure}
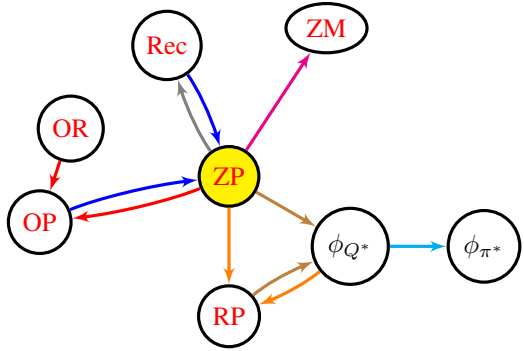

\subsubsection{Results Related to \ZP}

\begin{lemma}[Functions of independent random variables are also independent]
\label{lm:ind_transform}
If $X\ind Y$, then for any functions $f, g$, we have $f(X) \ind g(Y)$. 
\end{lemma}
\begin{proof}
This is a well-known result. Here is an elementary proof. Let $A,B$ be any two sets, 
\begin{align}
&P(f(X) \in A, g(Y) \in B) = P(X\in f^{-1}(A), Y \in g^{-1}(B)) \\
&\stackrel{X\ind Y}= P(X\in f^{-1}(A)) P(Y \in g^{-1}(B)) = P(f(X) \in A) P(g(Y) \in B).
\end{align}
\end{proof}

\begin{lemma}
\label{lm:cond_ind}
If $X \ind Y \mid Z$, then for any function $f$, we have $X \ind Y, f(Z) \mid Z$.
\end{lemma}
\begin{proof}
\begin{align}
&P(Y, f(Z) \mid X,Z) = P(f(Z) \mid X,Z) P(Y \mid X,Z, f(Z)) \\
&= P(f(Z) \mid Z) P(Y \mid Z) =  P(f(Z) \mid Z) P(Y \mid Z,f(Z)) = P(Y,f(Z)\mid Z).
\end{align}
\end{proof}

\begin{proposition}[\textbf{\ZP implies both \ZM and \Rec.}]
\label{prop:zm}
\end{proposition}
\begin{remark}
These are \textbf{new} results. \ZP implying \ZM means $\phi_L\succeq \phi_M$.
\end{remark}
\begin{proof}[Proof of \textbf{\autoref{prop:zm} (\ZP implies \ZM)}]
Since \ZP that $z_{t+1} \ind h_{t} \mid z_t,a_t$, this implies $
z_{t+1} \ind f(h_{t}) \mid z_t,a_t$ for any transformation $f$ by \autoref{lm:ind_transform}. One special case of $f$ is that $f(h_t) =  (z_{1:t},a_{1:t-1})$ , where $z_k = \phi(h_k)$, which is \ZM. 
\end{proof}

\begin{figure}[h]
    \centering
     \begin{tikzpicture}[>=latex',line join=bevel,very thick,scale=0.7]
\node (z) at (90.0bp,137.0bp) [draw,fill=lightgray,circle] {$z_t$};
  \node (next_z) at (162.0bp,45.0bp) [draw,fill=lightgray,ellipse] {$z_{t+1}$};
  \node (h) at (18.0bp,91.0bp) [draw,circle] {$h_t$};
  \node (next_o) at (90.0bp,45.0bp) [draw,ellipse] {$o_{t+1}$};
  \node (a) at (18.0bp,18.0bp) [draw,circle] {$a_t$};
  \draw [->] (h) ..controls (42.815bp,106.85bp) and (55.311bp,114.84bp)  .. (z);
  \draw [->] (h) ..controls (54.387bp,85.804bp) and (83.726bp,80.474bp)  .. (108.0bp,72.0bp) .. controls (117.99bp,68.513bp) and (128.5bp,63.578bp)  .. (next_z);
  \draw [->] (h) ..controls (42.815bp,75.146bp) and (55.311bp,67.162bp)  .. (next_o);
  \draw [->] (a) ..controls (54.232bp,13.71bp) and (83.517bp,12.22bp)  .. (108.0bp,18.0bp) .. controls (118.3bp,20.431bp) and (128.87bp,25.142bp)  .. (next_z);
  \draw [->] (a) ..controls (43.581bp,27.593bp) and (54.077bp,31.529bp)  .. (next_o);
  \draw [->] (next_o) ..controls (115.87bp,45.0bp) and (125.03bp,45.0bp)  .. (next_z);
\end{tikzpicture}
    \caption{\textbf{The graphical model of the interaction between history encoder and the environment.}}
    \label{fig:pgm}
\end{figure}
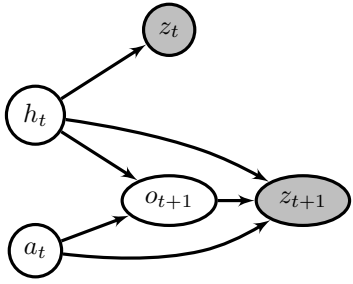

\begin{proof}[Proof of \textbf{\autoref{prop:zm} (\ZP implies \Rec)}]
Let $\phi$ satisfy \ZP, \ie $z' \ind h \mid \phi(h),a$. Then we can show that $z' \ind h \mid \phi(h),a,o'$. This is because the graphical model ($(h,a) \to o'$ and $(h, a,o')\to z'$; see \autoref{fig:pgm}) does not have v-structure such that $(h, z') \to o'$, thus adding the variable $o'$ to conditionals preserves conditional independence, by the principle of d-separation~\citep{pearl1988probabilistic}.

As $(a,o')$ also appears in the condition, we have $z' \ind h' \mid \phi(h),a,o'$ by \autoref{lm:cond_ind}, which is the probabilistic form of \Rec.
\end{proof}

\begin{proposition}[\textbf{\OP and \Rec imply \ZP}~\citep{subramanian2022approximate}.]
\label{prop:op}
\end{proposition}
\begin{proof}[Proof of \textbf{\autoref{prop:op} and \autoref{thm:hierarchy} ($\phi_O \succeq \phi_L$)}]
We directly follow the proof in \citep[Proposition 4]{subramanian2022approximate}. Let $\phi$ satisfy \OP and \Rec, then we will have \ZP: 
\begin{align}
&P(z' \mid h, a) = \int   P(z',o' \mid h,a ) do' = \int   \delta(z' = \phi(h')) P(o' \mid h,a)do' \\
&\stackrel{(\Rec,\OP)}= \int \delta(z' = \psi_z(\phi(h),a,o')) P_o(o' \mid \phi(h),a)do'\\
&= \int P(z',o' \mid \phi(h),a) do' = P_z(z' \mid \phi(h),a).
\end{align}
\end{proof}

\begin{proof}[Proof of \textbf{\autoref{thm:hierarchy} ($\phi_{Q^*}\succeq \phi_{\pi^*}$)}]
If $\phi_{Q^*}(h_i) = \phi_{Q^*}(h_j)$, then
$
Q^*(h_i, a) = Q^*(h_j,a), \forall a
$, and then taking argmax we get the optimal policy,  $\pi^*(h_1) = \argmax_a Q^*(h_1,a) = \argmax_a Q^*(h_2,a) = \pi^*(h_2)$. 
\end{proof}

\begin{proposition}[\textbf{\OR and \ZP imply \OP}]
\label{prop:or_zp}
\end{proposition}
\begin{proof}
Consider given $h,a$, for any $o'$, 
\begin{align}
P(o',\phi(h') \mid h, \phi(h), a) &= P(\phi(h') \mid h,a) P(o'\mid \phi(h'), h, \phi(h), a)  \\ 
&\stackrel{(\ZP, \OR)}= P_z(\phi(h') \mid \phi(h),a) \delta(o'=\psi_o(\phi(h')))  \\
\label{ln:orzp1}
&= P_z(\phi(h') \mid \phi(h),a) P(o'\mid \phi(h'),\phi(h),a)  \\ 
&= P(o',\phi(h') \mid \phi(h),a),
\end{align}
where Ln.~\ref{ln:orzp1} follows that the \OR condition $o' \ind h' \mid \phi(h')$ implies $o' \ind \phi(h), a \mid \phi(h')$ by \autoref{lm:ind_transform}. Therefore, $o', \phi(h') \ind h \mid \phi(h), a$. By \autoref{lm:ind_transform}, we have \OP $o' \ind h \mid \phi(h), a$. 

\end{proof}

\begin{proposition}[\textbf{In MDPs, \OR implies \OP and \RP (thus \ZP)}]
\label{prop:mdp_or}
\end{proposition}
\begin{proof}
Assuming $\phi$ satisfies \OR in MDPs, \ie there exists $\psi_s: \mathcal Z \to \mathcal S$ such that $\psi_s(\phi(s)) = s$, thus preserving the state in the representation. By construction, there exists a distribution $P_s: \mathcal Z \times \mathcal A \to \Delta(\mathcal S)$ such that $P_s(s' \mid \phi(s), a) \defeq P(s' \mid \psi_s(\phi(s)), a) = P(s'\mid s,a)$ satisfying the \OP condition. Similarly, through construction, we can show \OR can imply the other conditions in MDPs.
\end{proof}

\subsubsection{Results Related to Multi-Step Conditions}
Below are results on multi-step \RP, \ZP, and \OP, and due to space limit, we do not show these connections in \autoref{fig:relation}. 

\begin{proposition}[\textbf{\ZP is equivalent to multi-step \ZP}]
\label{prop:multi-zp}
For $k\in \mathbb N^+$, define $k$-step \ZP as 
\begin{align}
P(z_{t+k} \mid h_t,a_{t:t+k-1}) = P(z_{t+k} \mid \phi(h_t),a_{t:t+k-1}), \quad \forall h,a,z.
\end{align}
\end{proposition}
\begin{proof}
As \ZP is $1$-step \ZP, thus multi-step \ZP implies \ZP. 
Now we show that \ZP implies multi-step \ZP. 
{
\small
\begin{align}
&P(z_{t+k} \mid h_t,a_{t:t+k-1}) = \int P(z_{t+1:t+k},o_{t+1:t+k}\mid h_t,a_{t:t+k-1})do_{t+1:t+k}dz_{t+1:t+k-1}\\ 
&=\int \prod_{i=1}^k \delta(z_{t+i}=\phi(h_{t+i}))P(o_{t+i}\mid h_{t+i-1},a_{t+i-1})do_{t+1:t+k}dz_{t+1:t+k-1}\\
&=\int \left(\int \delta(z_{t+k}=\phi(h_{t+k})) P(o_{t+k}\mid h_{t+k-1},a_{t+k-1}) do_{t+k}\right) \\
&\quad \prod_{i=1}^{k-1} \delta(z_{t+i}=\phi(h_{t+i}))P(o_{t+i}\mid h_{t+i-1},a_{t+i-1})do_{t+1:t+k-1}dz_{t+1:t+k-1}\\ 
&=\int P(z_{t+k} \mid h_{t+k-1},a_{t+k-1})\prod_{i=1}^{k-1} \delta(z_{t+i}=\phi(h_{t+i}))P(o_{t+i}\mid h_{t+i-1},a_{t+i-1})do_{t+1:t+k-1}dz_{t+1:t+k-1}\\ 
&\stackrel\ZP= \int P(z_{t+k} \mid \phi(h_{t+k-1}),a_{t+k-1})\prod_{i=1}^{k-1} \delta(z_{t+i}=\phi(h_{t+i}))P(o_{t+i}\mid h_{t+i-1},a_{t+i-1})do_{t+1:t+k-1}dz_{t+1:t+k-1}\\ 
&= \int P(z_{t+k} \mid z_{t+k-1},a_{t+k-1})\prod_{i=1}^{k-1} \delta(z_{t+i}=\phi(h_{t+i}))P(o_{t+i}\mid h_{t+i-1},a_{t+i-1})do_{t+1:t+k-1}dz_{t+1:t+k-1}\\ 
&=\dots \\
&= \int \prod_{i=2}^{k} P(z_{t+i} \mid z_{t+i-1},a_{t+i-1}) P(z_{t+1}\mid h_t,a_t) dz_{t+1:t+k-1}\\
&\stackrel\ZP= \int \prod_{i=2}^{k} P(z_{t+i} \mid z_{t+i-1},a_{t+i-1}) P(z_{t+1}\mid \phi(h_t),a_t) dz_{t+1:t+k-1}\\
&\stackrel\ZM = \int P(z_{t+1:t+k} \mid \phi(h_t),a_{t:t+i-1}) dz_{t+1:t+k-1}\\ 
&= P(z_{t+k} \mid \phi(h_t),a_{t:t+i-1}) .
\end{align}
}

\end{proof}

\begin{proposition}[\textbf{\ZP and \RP imply multi-step \RP}]
\label{prop:multi-rp}
For $k\in \mathbb N^+$, define $k$-step \RP as 
\begin{align}
\E{}{r_{t+k} \mid h_t,a_{t:t+k}} = \E{}{r_{t+k} \mid \phi(h_t),a_{t:t+k}}, \quad \forall h,a
\end{align}
\end{proposition}
\begin{proof}
\begin{align}
&\E{}{r_{t+k} \mid h_t,a_{t:t+k}} 
= \int P(o_{t+1:t+k} \mid h_t,a_{t:t+k-1}) \E{}{r_{t+k} \mid h_{t+k},a_{t+k}}do_{t+1:t+k}\\
&\stackrel\RP=  \int P(o_{t+1:t+k} \mid h_t,a_{t:t+k-1}) R_z(\phi(h_{t+k}),a_{t+k})do_{t+1:t+k}\\
&=  \int P(o_{t+1:t+k} \mid h_t,a_{t:t+k-1}) \delta(z_{t+k}= \phi(h_{t+k}))R_z(z_{t+k},a_{t+k})do_{t+1:t+k}dz_{t+k}\\
&=  \int\left(\int P(o_{t+1:t+k} \mid h_t,a_{t:t+k-1}) \delta(z_{t+k}= \phi(h_{t+k}))do_{t+1:t+k}\right) R_z(z_{t+k},a_{t+k})dz_{t+k}\\
&= \int P(z_{t+k} \mid h_t,a_{t:t+k-1})R_z(z_{t+k},a_{t+k})dz_{t+k}\\
&\stackrel{k\text{-step }\ZP}= \int P(z_{t+k} \mid \phi(h_t),a_{t:t+k-1})R_z(z_{t+k},a_{t+k})dz_{t+k}\\
&=\E{}{r_{t+k} \mid \phi(h_t),a_{t:t+k}},
\end{align}
where $k$-step \ZP is implied by \ZP by \autoref{prop:multi-zp}.
\end{proof}

\begin{proposition}[\textbf{\OP implies multi-step \OP in MDPs, but not POMDPs}]
\label{prop:multi-op}
For $k\in \mathbb N^+$, define $k$-step \OP as 
\begin{align}
P(o_{t+k} \mid h_t,a_{t:t+k-1}) = P(o_{t+k} \mid \phi(h_t),a_{t:t+k-1}), \quad \forall h,a,o.
\end{align}
\end{proposition}
\begin{proof}
We first show the result in MDPs. Assume a state encoder $\phi$ satisfies \OP, 
\begin{align}
P(s_{t+k} \mid s_t,a_{t:t+k-1}) 
&= \int P(s_{t+1:t+k} \mid s_t,a_{t:t+k-1}) ds_{t+1:t+k-1} \\
&\stackrel{\text{MDPs}}= \int P(s_{t+1}\mid s_t,a_t) \prod_{i=2}^k P(s_{t+i}\mid s_{t+i-1},a_{t+i-1}) ds_{t+1:t+k-1} \\
&\stackrel\OP= \int P(s_{t+1}\mid \phi(s_t),a_t) \prod_{i=2}^k P(s_{t+i}\mid s_{t+i-1},a_{t+i-1}) ds_{t+1:t+k-1} \\
&\stackrel{\text{MDPs}}=\int P(s_{t+1:t+k} \mid \phi(s_t),a_{t:t+k-1}) ds_{t+1:t+k-1} \\
&= P(s_{t+k} \mid \phi(s_t),a_{t:t+k-1}) .
\end{align}

However, in POMDPs, \OP does not imply multi-step \OP. This can be shown by a counterexample in \citet[Theorem~4.10]{castro2009equivalence}, where the weak belief bisimulation relation corresponds to single-step \OP and \RP, while trajectory equivalence corresponds to multi-step \OP and \RP. The idea is to show that for two histories $h_t^1$ and $h_t^2$, if $P(o_{t+1}\mid h_t^1, a_t) = P(o_{t+1}\mid h_t^2, a_t), \forall o_{t+1}, a_t$, it does not imply that $P(o_{t+2}\mid h_t^1, a_t, o_{t+1}, a_{t+1}) = P(o_{t+2}\mid h_t^2, a_t, o_{t+1}, a_{t+1}), \forall o_{t+1:t+2}, a_{t:t+1}$. 
\end{proof}

\subsubsection{Results Related to $\phi_{Q^*}$}

\begin{proof}[Proof sketch of \textbf{\ZP + \RP ($\phi_L$) imply $\phi_{Q^*}$}] 
To show $Q^*(h,a) = Q^*_z(\phi_L(h),a),\forall h,a$, please see~\citep[Theorem~5 and Theorem~25]{subramanian2022approximate} for finite-horizon and infinite-horizon POMDPs, respectively. For the approximate version, please see~\citep[Theorem~9 and Theorem~27]{subramanian2022approximate}. 
By definition, $Q^*(h,a) = Q^*_z(\phi_L(h),a),\forall h,a$ implies that $\phi_L$ is a kind of $\phi_{Q^*}$. 
\end{proof}

\begin{proof}[Proof of \textbf{\autoref{thm:EG_ZP_implies_Er} (\ZP + $\phi_{Q^*}$ imply \RP)}] 

Suppose $\phi$ satisfies \ZP and we train model-free RL with value parameterized by $\mathcal Q(\phi(h),a)$ to satisfy the Bellman optimality equation:
\begin{align}
\mathcal Q(\phi(h_t),a_t) = 
\begin{cases}
\E{}{r_t\mid h_t,a_t} &  t=T,  \\
\E{}{r_t\mid h_t, a_t} + \gamma \E{o_{t+1}\sim P(\mid h_t,a_t)}{\max\limits_{a_{t+1}} \mathcal Q(\phi(h_{t+1}),a_{t+1})} & \mathrm{else},
\end{cases}
\end{align}
where the case $t=T$ only applies to finite-horizon problems (the same below). 
This is equivalent to say that $\mathcal Q(\phi(h_t),a_t) = Q^*(h_t,a_t), \forall h_t,a_t$, where $Q^*$ satisfies the Bellman optimality equation, too:
\begin{align}
Q^*(h_t,a_t) = 
\begin{cases}
\E{}{r_t\mid h_t,a_t} &  t=T,   \\
\E{}{r_t\mid h_t,a_t} + \gamma \E{o_{t+1}\sim P(\mid h_t,a_t)}{\max\limits_{a_{t+1}} Q^*(h_{t+1},a_{t+1})} & \mathrm{else}.
\end{cases}
\end{align}

Now we can construct an abstract MDP with $\phi$. The latent transition matches due to \ZP. The latent reward function is purely defined by latent value and latent transition\footnote{In the main paper, we omit the finite-horizon case due to space limit.}:
\begin{align}
\mathcal R_z(z_t,a_t) &\defeq \begin{cases}
\mathcal Q(z_t,a_t) & t=T,  \\
\mathcal Q(z_t,a_t) - \gamma 
\E{z_{t+1} \sim P(\mid z_t,a_t)}{\max\limits_{a_{t+1}} \mathcal Q(z_{t+1}, a_{t+1})} & \mathrm{else}.
\end{cases}
\end{align}

We want to show \RP condition: 
$
\mathcal R_z(\phi(h_t),a_t) = \E{}{r_t \mid h_t,a_t}, \forall h_t,a_t
$.

Here is our proof. 
Recall that the grounded reward function can also be derived reversely by $Q^*$:
\begin{align}
\E{}{r_t\mid h_t,a_t} &\defeq \begin{cases}
Q^*(h_t,a_t) & t=T,  \\
Q^*(h_t,a_t) - \gamma 
\E{o_{t+1} \sim P(\mid h_t,a_t)}{\max\limits_{a_{t+1}} Q^*(h_{t+1}, a_{t+1})} & \mathrm{else}.
\end{cases}
\end{align}
If the problem is finite-horizon with horizon $T$ and when $t=T$, \RP holds due to $\mathcal Q(\phi(h_T),a_T) = Q^*(h_T,a_T)$. 

Now consider general case when $t < T$ in finite-horizon ($\gamma=1$) and any $t$ in infinite-horizon ($\gamma < 1$). Due to $Q$-value match ($\mathcal Q(\phi(h_t),a_t) = Q^*(h_t,a_t)$), it is equivalent to show that 
\begin{align}
\E{o_{t+1}\sim P(\mid h_t,a_t)}{\max_{a_{t+1}}Q^*(h_{t+1},a_{t+1})} = \E{z_{t+1} \sim P(\mid \phi(h_t),a_t)}{\max_{a_{t+1}} \mathcal Q(z_{t+1}, a_{t+1})}, \quad \forall h_{t}, a_{t}, t.
\end{align}

Proof for this:
\begin{align}
\LHS &\stackrel{\phi_{Q^*}}= \E{o_{t+1}\sim P(\mid h_t,a_t)}{\max_{a_{t+1}}\mathcal Q(\phi(h_{t+1}),a_{t+1})} \\
&= \int \left(\int P(o_{t+1} \mid h_{t},a_{t}) \delta (z_{t+1} = \phi(h_{t+1})) do_{t+1}\right)\max_{a_{t+1}} \mathcal Q(z_{t+1},a_{t+1}) dz_{t+1} \\
&= \int P(z_{t+1} \mid h_{t},a_{t}) \max_{a_{t+1}} \mathcal Q(z_{t+1},a_{t+1}) dz_{t+1}\\
&\stackrel{\ZP}= \int P(z_{t+1} \mid \phi(h_{t}),a_{t})  \max_{a_{t+1}} \mathcal Q(z_{t+1},a_{t+1}) dz_{t+1} = \RHS.
\end{align}

\end{proof}

\begin{lemma}[\textbf{Integral probability metric}~\citep{subramanian2022approximate}]
\label{lm:ipm}
Given by a function class $\mathcal F$, integral probability metric (IPM) between two distributions $\mathbb P, \mathbb Q \in \Delta(\mathcal Z)$ is 
\begin{align}
\mathcal D_{\mathcal F} (\mathbb P,\mathbb Q) = \sup_{f\in \mathcal F} |\E{x\sim \mathbb P}{f(x)} - \E{y\sim \mathbb Q}{f(y)} |.
\end{align}
For any real-valued function $g$, the following inequality is derived by definition:
\begin{align}
\label{eq:minkowski}
|\E{x\sim \mathbb P}{g(x)} - \E{y\sim \mathbb Q}{g(y)} | \le \rho_{\mathcal F}(g) \mathcal D_{\mathcal F} (\mathbb P,\mathbb Q),
\end{align}
where $\rho_{\mathcal F}(g) \defeq \inf \{ \rho\in \R_+ \mid \rho^{-1} g \in \mathcal F\}$ is a Minkowski functional.
\end{lemma}
\begin{remark} Some examples include:  
\begin{itemize}
    \item Total Variance (TV) distance is an IPM defined by $\mathcal F_{\text{TV}} = \{f:\|f\|_\infty \le 1\}$.
    \item Wasserstein (W) distance is an IPM defined by $\mathcal F_{\text{W}} = \{f:\|f\|_L \le 1\}$. 
    \item KL divergence is not an IPM, but is an upper bound of TV distance by Pinsker's inequality:
    \begin{align}
        \mathcal D_{\mathcal F_{\text{TV}}} (\mathbb P,\mathbb Q) \le \sqrt{2\kl(\mathbb P \mid\mid \mathbb Q)}.
    \end{align}
\end{itemize}
\end{remark}

\begin{theorem}[\textbf{Approximate version of \autoref{thm:EG_ZP_implies_Er} (approximate \ZP and approximate $\phi_{Q^*}$ imply approximate $\phi_L$)}]
\label{thm:approx}
Suppose the encoder $\phi$ satisfies \textbf{approximate} \ZP (AZP) and we train model-free RL with value parameterized by $\mathcal Q(\phi(h),a)$ to \textbf{approximate} $Q^*(h,a)$, namely: $\forall t,h_t,a_t$, 
\begin{align}
&\exists P_{z}: \mathcal Z \times \mathcal A \to \Delta(\mathcal Z), \quad \st\quad \mathcal D_{\mathcal F} (P(z_{t+1} \mid h_t,a_t), P_z(z_{t+1}\mid \phi(h_t),a_t)) \le \delta_t ,\tag{AZP} \\
&| Q^*(h_t,a_t) - \mathcal Q(\phi(h_t),a_t) | \le \alpha_t. \tag{Approx. $\phi_{Q^*}$}
\end{align}
where $\mathcal D_{\mathcal F}$ is an IPM. Under these conditions, we can construct a latent reward function:
\begin{align}
\mathcal R_z(z_t,a_t) \defeq \begin{cases}
\mathcal Q(z_t,a_t) & t=T , \\
\mathcal Q(z_t,a_t) - \gamma 
\E{z_{t+1} \sim P_z(\mid z_t,a_t)}{\max\limits_{a_{t+1}} \mathcal Q(z_{t+1}, a_{t+1})}  & \mathrm{else} ,
\end{cases}
\end{align}
such that 
\begin{align}
\label{eq:rew_bound}
&|\E{}{r_t \mid h_t,a_t} -   \mathcal R_z(\phi(h_t),a_t) | \le \epsilon_t, \quad \forall t, h_t, a_t, \tag{ARP}\\
&\text{where}\quad \epsilon_t = \begin{cases}
\alpha_T & t = T, \\ 
\alpha_t + \gamma (\alpha_{t+1} +  \rho_{\mathcal F}(\mathcal V_{t+1}) \delta_t) & \mathrm{else},
\end{cases}\\
&\quad\quad \quad \mathcal V(z_t) = \max_{a_{t}} \mathcal Q(z_{t}, a_{t}),
\end{align}
where $\mathcal V_{t+1}$ is the latent state-value function $\mathcal V$ at step $t+1$. 
\end{theorem}
\begin{proof}
For the case of $t=T$ in finite-horizon, ARP holds by the assumption of approx. $\phi_{Q^*}$. Now we discuss generic case of $t$. 
Recall the reward and latent reward can be rewritten as: 
\begin{align}
\E{}{r_t\mid h_t,a_t} &= Q^*(h_t,a_t) - \gamma 
\E{o_{t+1} \sim P(\mid h_t,a_t)}{\max_{a_{t+1}} Q^*(h_{t+1}, a_{t+1})},  \\
\mathcal R_z(z_t,a_t) &= \mathcal Q(z_t,a_t) - \gamma 
\E{z_{t+1} \sim P(\mid z_t,a_t)}{\max_{a_{t+1}} \mathcal Q(z_{t+1}, a_{t+1})} .
\end{align}
Therefore, the reward gap is upper bounded: 
\begin{align}
&|\E{}{r_t\mid h_t,a_t} -   \mathcal R_z(\phi(h_t),a_t) | \\ 
\label{ln:triangle}
&\le \left| Q^*(h_t,a_t) - \mathcal Q(\phi(h_t),a_t)\right| \\ 
& + \gamma \left|\E{o_{t+1} \sim P(\mid h_t,a_t)}{\max_{a_{t+1}} Q^*(h_{t+1}, a_{t+1})} -  \E{z_{t+1} \sim P(\mid \phi(h_t),a_t)}{\max_{a_{t+1}} \mathcal Q(z_{t+1}, a_{t+1})} \right| \\ 
\label{ln:triangle2}
&\le \alpha_t + \gamma \left|\E{o_{t+1} \sim P(\mid h_t,a_t)}{\max_{a_{t+1}} Q^*(h_{t+1}, a_{t+1}) - \max_{a_{t+1}} \mathcal Q(\phi(h_{t+1}), a_{t+1})} \right| \\ 
&+ \gamma \left|\E{o_{t+1} \sim P(\mid h_t,a_t)}{\max_{a_{t+1}} \mathcal Q(\phi(h_{t+1}), a_{t+1})} -  \E{z_{t+1} \sim P(\mid \phi(h_t),a_t)}{\max_{a_{t+1}} \mathcal Q(z_{t+1}, a_{t+1})} \right| \\
\label{ln:max}
&\le \alpha_t + \gamma \E{o_{t+1} \sim P(\mid h_t,a_t)}{\max_{a_{t+1}} \left|Q^*(h_{t+1}, a_{t+1}) - \mathcal Q(\phi(h_{t+1}), a_{t+1})\right|} \\ 
&+ \gamma \left|\E{z_{t+1} \sim P(\mid h_t,a_t)}{\max_{a_{t+1}} \mathcal Q(z_{t+1}, a_{t+1})} -  \E{z_{t+1} \sim P(\mid \phi(h_t),a_t)}{\max_{a_{t+1}} \mathcal Q(z_{t+1}, a_{t+1})} \right| \\ 
\label{ln:induction}
&\le \alpha_t + \gamma \alpha_{t+1} + \gamma \left|\E{z_{t+1} \sim P(\mid h_t,a_t)}{\mathcal V(z_{t+1})} -  \E{z_{t+1} \sim P(\mid \phi(h_t),a_t)}{\mathcal V(z_{t+1})} \right| \\ 
\label{ln:ipm}
&\le \alpha_t + \gamma (\alpha_{t+1} +  \rho_{\mathcal F}(\mathcal V_{t+1}) \delta_t),
\end{align}
where \autoref{ln:triangle} is by triangle inequality, \autoref{ln:triangle2} is by triangle inequality and approx. $\phi_{Q^*}$, \autoref{ln:max} is by the maximum-absolute-difference inequality $|\max f(x) - \max g(x) | \le \max |f(x) - g(x)|$, \autoref{ln:induction} is by approx. $\phi_{Q^*}$, and Ln.~\ref{ln:ipm} is by the property of IPM (\autoref{eq:minkowski} in \autoref{lm:ipm}) and AZP. 
\end{proof}
\begin{remark}
In the infinite-horizon problem, assume $\delta_t = \delta$ and $\alpha_t=\alpha$ for any $t$, and $\mathcal D_{\mathcal F}$ is Wasserstein distance. Furthermore, assume the latent reward $\mathcal R_z(z,a)$ is $L_r$-Lipschitz and the latent transition $P_z(z'\mid z,a)$ is $L_p$-Lipschitz, then by~\citep[Lemma 44]{subramanian2022approximate}, if $\gamma L_p < 1$, 
\begin{align}
\rho_{\mathcal F}(\mathcal V_{t+1}) = \|\mathcal V\|_L \le \frac{L_r}{1-\gamma L_p}, \quad \forall t.
\end{align}
Thus, the reward difference bound can be rewritten as
\begin{align}
\epsilon \le (1 + \gamma)\alpha +  \frac{\gamma L_r \delta}{1-\gamma L_p} .
\end{align}
\end{remark}

\section{Objectives and Optimization in Self-Predictive RL}
\label{app:optim}

\begin{proof}[Proof of \textbf{\autoref{prop:l2}}]
    First, we show $\mathcal L_{\ZP,\ell}(\phi,\theta; h,a)  \le J_{\ell}(\phi,\theta,\phi;h,a), \forall h,a$. 
    
    \begin{align}
    &J_{\ell}(\phi,\theta,\phi;h,a) - \mathcal L_{\ZP,\ell}(\phi,\theta; h,a)  \\
    &= \E{o'\sim P(\mid h,a)}{\|g_\theta(f_\phi(h),a)-f_{\phi}(h')\|_2^2} - \| g_\theta(f_\phi(h),a) - \E{o'\sim P(\mid h,a)}{f_{\phi}(h')}\|_2^2 \\ 
    &= \cancel{\|g_\theta(f_\phi(h),a)\|_2^2} - \cancel{2\E{o'}{\langle g_\theta(f_\phi(h),a), f_{\phi}(h')\rangle}} + \E{o'}{\|f_{\phi}(h') \|_2^2} \\ 
    &\quad -  \cancel{\|g_\theta(f_\phi(h),a)\|_2^2} + \cancel{2 \langle g_\theta(f_\phi(h),a), \E{o'}{f_{\phi}(h')}\rangle} - \|\E{o'}{f_{\phi}(h')} \|_2^2\\ 
    &= \E{o'}{\|f_{\phi}(h') - \E{o'}{f_{\phi}(h')}\|_2^2} \ge 0. 
    \end{align}
    The inner product terms are cancelled due to the fact that inner product is bilinear. The equality holds when $P(o'\mid h,a)$ is deterministic. 
    
    Second, the ideal objective using $\ell_2$ distance $\mathcal L_{\ZP,\ell}(\phi,\theta; h,a)=\| g_\theta(f_\phi(h),a) - \E{o'\sim P(\mid h,a)}{f_{\phi}(h')}\|_2^2$ can only lead to \EZP condition when reaching optimum of zero.  This is because when
    $g_\theta(f_\phi(h),a) = \E{o'\sim P(\mid h,a)}{f_\phi(h')},\forall h,a$, it precisely satisfies the \EZP condition. 
    \end{proof}

    \begin{proof}[Proof of \textbf{\autoref{prop:fdiv}}]
    The goal is to show $\mathcal L_{\ZP,D_\mathtt f}(\phi,\theta; h,a)  \le J_{D_\mathtt f}(\phi,\theta,\phi;h,a), \forall h,a$. 
    
    Recall the definition of $\mathtt f$-divergence that subsumes forward and reverse KL divergences:
    $
    D_\mathtt f (Q\mid\mid P) = \int P(x) f\left(\frac{Q(x)}{P(x)}\right) dx
    $, where $f:[0,\infty)\to \R$ is a convex function.
    \begin{align}
    \mathcal L_{\ZP,D_\mathtt f}(\phi,\theta; h,a) 
    &= \int \mathbb P_\phi(z\mid h)D_\mathtt f( \mathbb P_\phi(z'\mid h,a)\mid\mid \mathbb P_{\theta}(z' \mid z,a))dz\\
    &= \iint \mathbb P_\phi(z\mid h) \mathbb P_{\theta}(z'\mid z,a) f\left(\frac{\E{o'}{\mathbb P_\phi(z'\mid h')}}{\mathbb P_{\theta}(z'\mid z,a)}\right) dzdz' \\
    &= \iint \mathbb P_\phi(z\mid h) \mathbb P_{\theta}(z'\mid z,a) f\left(\E{o'}{\frac{\mathbb P_\phi(z'\mid h')}{\mathbb P_{\theta}(z'\mid z,a)}}\right) dzdz' \\
    &\le  \iint \mathbb P_\phi(z\mid h) \mathbb P_{\theta}(z'\mid z,a)  \E{o'}{f\left(\frac{\mathbb P_\phi(z'\mid h')}{\mathbb P_{\theta}(z'\mid z,a)}\right)} dzdz' \\
    &= \E{z\sim P_\phi(\mid h),o'\sim P(\mid h,a)}{ \int \mathbb P_{\theta}(z'\mid z,a) f\left(\frac{\mathbb P_\phi(z'\mid h')}{\mathbb P_{\theta}(z'\mid z,a)}\right)dz'}\\
    &= J_{D_\mathtt f}(\phi,\theta,\phi;h,a),
    \end{align}
    where we use Jensen's inequality by the convexity of $f$. The equality holds when $P(o'\mid h,a)$ is deterministic according to Jensen's inequality. 
    \end{proof}

    \textbf{Discussion on the double sampling issue.} The ideal objective \autoref{eq:zp_loss} is hard to have an \textit{unbiased} estimate in stochastic environments. This is due to double sampling issue~\citep{baird1995residual} that we do not allow agent to i.i.d. sample twice from transition, \ie $o'_1,o'_2 \sim P(o'\mid h,a)$. To see it more clearly, for example, when $\mathbb D$ is forward KL divergence, The ideal objective becomes: 
    \begin{align}
    \mathcal L_{\ZP,\text{FKL}}(\phi,\theta;h,a) &= \E{z\sim \mathbb P_\phi(\mid h)}{\kl(\mathbb P_\phi(z'\mid h,a) \mid \mid \mathbb P_{\theta}(z'\mid z,a))} \\
    &= \E{z\sim \mathbb P_\phi(\mid h)}{\kl(\E{o'}{\mathbb P_\phi(z'\mid h')} \mid \mid \mathbb P_{\theta}(z'\mid z,a))} \\
    &= \E{z\sim \mathbb P_\phi(\mid h)}{\int \E{o'}{\mathbb P_\phi(z'\mid h')} \log \frac{\E{o'}{\mathbb P_\phi(z'\mid h')}}{\mathbb P_{\theta}(z'\mid z,a)} dz'} \\
    &= \E{z\sim \mathbb P_\phi(\mid h),o'\sim P(\mid h,a), z'\sim \mathbb P_\phi(\mid h')}{\log \frac{\E{o'_+}{\mathbb P_\phi(z'\mid h'_+)}}{\mathbb P_{\theta}(z'\mid z,a)}},
    \end{align}
    where $o',o'_+ \sim P(\mid h,a)$ are two i.i.d. samples, and $h'=(h,a,o'), h'_+ = (h,a,o'_+)$.

\begin{proof}[Proof of \textbf{\autoref{prop:stationary_point}}]
    Recall the stop-gradient objective \autoref{eq:l2}:
    \begin{align}
    J \defeq J_\ell(\phi,\theta,\overline\phi; h,a) = \E{o'\sim P(\mid h,a)}{\|g_\theta(f_\phi(h),a)-f_{\overline\phi} (h')\|_2^2}.
    \end{align}
    The gradients are: 
    \begin{align}
    \nabla_{\phi,\theta} \E{h,a}{J} 
    &= \E{h,a}{(g_\theta(f_\phi(h),a)-\E{o'}{f_{\overline\phi}(h')})^\top \nabla_{\phi,\theta} g_\theta(f_\phi(h),a)}.
    \end{align}
    When $\theta,\phi$ reaches a stationary point, we have $\nabla_{\phi,\theta} \E{h,a}{J}= 0$ and thus $\overline \phi = \phi$.
    Therefore, we have \textit{a} stationary point $(\theta^*,\phi^*)$ such that:
    $g_\theta(f_\phi(h),a) = \E{o'\sim P(\mid h,a)}{f_{\phi}(h')}$, for any $h,a$, which is the \textbf{expected} \ZP (\EZP) condition.
    In a deterministic environment, \EZP is equivalent to \ZP. In a stochastic environment, \EZP is to match the expectation  (instead of distribution).
    
    However, in online objective, the gradient \wrt $\phi$ contains an extra term:
    \begin{align}
    \E{h,a,o'}{(g_\theta(f_\phi(h),a)-f_{\phi} (h'))^\top\nabla_\phi f_\phi(h')}.
    \end{align}
    Thus, when \EZP holds, the gradient is zero in deterministic environments, but can be non-zero in stochastic environments.

    \end{proof}

    \begin{proof}[Proof of \textbf{\autoref{thm:collapse}}]
    \textbf{The setup.} Let $h_{t:-k}$ a vectorization of the recent truncation of history $h_t$ with window size of $k\in \mathbb N$, \ie $h_{t:-k} = \Vec(a_{t-k},o_{t-k+1}, \dots, a_{t-1}, o_{t})\in \R^{x}$,\footnote{We zero pad $a_i$ and $o_i$ if $i \le 0$.} where $x = k(|\mathcal O| + |\mathcal A|)$. 
    We assume a linear encoder that maps history $h_{t} \in \mathcal H_t$ into $z_t$: 
    \begin{align}
    z_t = f_\phi(h_t) \defeq \phi^\top h_{t:-k} \in \R^d,
    \end{align}
    where $k\in\mathbb N$ is a constant, and the parameters $\phi \in \R^{x\times d}$. In other words, the linear encoder only operates on recent histories of a fixed window size. 
    We assume a  linear deterministic latent transition 
    \begin{align}
    z_{t+1} = g_\theta(z_t,a_t) \defeq \theta_z^\top z_t + \theta_a^\top a_t \in \R^d,
    \end{align}
    where the parameters $\theta_z \in \R^{d\times d}$ and $\theta_a \in \R^{a\times d}$. In fact, the result can be generalized to a non-linear dependence of actions. 
    
    \textbf{The proof.} The continuous-time training dynamics of $\phi$:
    \begin{align}
    \dot \phi &= -\E{h_t,a_t}{\nabla_\phi J_\ell(\phi,\theta,\overline \phi;h_t,a_t)} \\
    &= -\E{h_t,a_t,o_{t+1}}{\nabla_\phi \|\theta_z^\top \phi^\top h_{t:-k} + \theta_a^\top a_t - \overline \phi^\top h_{t+1:-k}  \|_2^2} \\
    &= -\E{h_t,a_t}{(\theta_z^\top \phi^\top h_{t:-k} + \theta_a^\top a_t -\E{o_{t+1}}{\overline\phi^\top h_{t+1:-k}})^\top\nabla_\phi \theta_z^\top \phi^\top h_{t:-k}}\\ 
    &= -\E{h_t,a_t}{h_{t:-k} (\theta_z^\top \phi^\top h_{t:-k} + \theta_a^\top a_t -\E{o_{t+1}}{\overline\phi^\top h_{t+1:-k}})^\top} \theta_z^\top.
    \end{align}
    
    The gradient of the loss \wrt $\theta_z$:
    \begin{align}
    &\nabla_{\theta_z} \E{h_t,a_t}{J_\ell(\phi,\theta,\overline \phi;h_t,a_t)} \\ 
    &=\E{h_t,a_t}{(\theta_z^\top \phi^\top h_{t:-k} + \theta_a^\top a_t -\E{o_{t+1}}{\overline\phi^\top h_{t+1:-k}})^\top \nabla_{\theta_z} (\theta_z^\top \phi^\top h_{t:-k} + \theta_a^\top a_t)} \\
    &=\phi^\top \E{h_t,a_t}{h_{t:-k} (\theta_z^\top \phi^\top h_{t:-k} + \theta_a^\top a_t -\E{o_{t+1}}{\overline\phi^\top h_{t+1:-k}})^\top}   \in \R^{d\times d}.
    \end{align}
    
    Therefore, we have
    \begin{align}
    \phi^\top \dot \phi &= -\nabla_{\theta_z} \E{h_t,a_t}{J_\ell(\phi,\theta,\overline \phi;h_t,a_t)}\theta_z^\top.
    \end{align}
    Following the practice in~\citep{tang2022understanding}, we assume $\nabla_{\theta_z} \E{h_t,a_t}{J_\ell(\phi,\theta,\overline \phi;h_t,a_t)} = 0$, \ie  $\theta_z$ reaches the stationary point of the inner optimization that depends on $\phi$, then $\phi^\top \dot \phi = 0$. Thus, the training dynamics of $\phi^\top \phi$ is
    \begin{align}
    \frac{d (\phi^\top \phi)}{dt} &= \dot \phi^\top \phi + \phi^\top \dot \phi =  \dot \phi^\top \phi + (\dot \phi^\top \phi)^\top  =  0.
    \end{align}
    This means that $\phi^\top \phi$ keeps same value during training. 
    \end{proof}

\section{Analyzing Prior Works on State and History Representation Learning}
\label{app:prior_works}

In this section, we provide a concise but analytical overview of previous works that learn or approximate self-predictive or observation-predictive representations on states or histories.
Please see \citet{lesort2018state} for an early and detailed survey on \textit{state} representation learning. 

We focus on the objectives of state or history encoders in their value functions. For each work discussed, we present a summary of the conditions that their encoders aim to satisfy or approximate at the beginning of each paragraph.  In cases where multiple encoder objectives are proposed, we select the one employed in their primary experiments for our discussion. In particular, we list the exact objectives they aim to optimize, which might be redundant for \textit{exact} conditions. For example, multi-step \RP can be implied by \RP + \ZP by \autoref{prop:multi-rp} (or $\phi_{Q^*}$ + \ZP by \autoref{thm:EG_ZP_implies_Er}), and multi-step \ZP can be implied by \ZP by \autoref{prop:multi-zp}.

\subsection{Self-Predictive Representations}

\paragraph{CRAR~\citep{franccois2019combined}: $\phi_{Q^*}$ + \RP + \ZP with online $\ell_2$ + regularization.}  Combined reinforcement via abstract representations (CRAR) is designed to learn self-predictive representations in MDPs. It incorporates \RP and \ZP auxiliary losses into the end-to-end RL objective. They assume the deterministic case (for the encoder, transition and latent transition), thus using $\ell_2$ objective is sufficient (\autoref{prop:l2}). They use online \ZP target and observe the representation collapse when the reward signals are scarce. To prevent this issue,  they introduce regularization terms into the encoder objective. These terms minimize $\E{s_1,s_2}{\exp(-\|\phi(s_1) - \phi(s_2)\|_2)} + \E{s}{\max(\|\phi(s)\|_\infty^2 - 1, 0)}$, where $s_1,s_2,s$ are samples from the state space. These terms, similar to entropy maximization,  encourage diversity within the latent space.

\paragraph{DeepMDP~\citep{gelada2019deepmdp}: $\phi_{Q^*}$ + \RP + \ZP with online $\ell_2$.} DeepMDP aims to learn state representations that match \RP and \ZP. In their experiments, they assume deterministic case, resulting in dirac distributions $\mathbb P_\phi(z'\mid s')$ and $\mathbb P_{\theta}(z'\mid z,a)$. 
Although they use the Wasserstein distance, it reduces to $\ell_2$ distance for two dirac distributions. They use an online target in \ZP loss.  
In their toy DonutWorld task,  they try phased training with \RP + \ZP, but the agent tends to be trapped in a local minimum of zero \ZP. Then they try $\phi_{Q^*}$ + \RP + \ZP in Atari by training \RP + \ZP as an auxiliary task of a distributional RL baseline, outperforming the baseline in their main result.  
They also find that $\phi_{Q^*}$ + \RP + \ZP is comparable to $\phi_{Q^*}$ + \ZP, aligned with our theoretical prediction based on \autoref{thm:EG_ZP_implies_Er}. They also try phased training in Atari and find that  \RP + \ZP performs poorly, while \RP + \ZP + \OR yields good results. 

\paragraph{SPR~\citep{schwarzer2020data}: $\phi_{Q^*}$ +  multi-step \ZP with EMA $\cos$.} 
Self-Predictive Representations (SPR) improves the \ZP objective in DeepMDP. They use a special kind of $\ell_2$ loss (\ie $\cos$ distance) to bound the loss scale, and use an EMA target. They use multi-step prediction loss to learn the condition:
\begin{align}
P(z_{t+1:t+k} \mid s_t,a_{t:t+k-1}) = P(z_{t+1:t+k} \mid \phi(s_t),a_{t:t+k-1}),
\end{align}
where $k=5$ in their experiments. In addition, to reduce the large latent space generated by CNNs, they use a linear projection of the latent states to satisfy \ZP. 

\paragraph{DBC~\citep{zhang2020learning}: $\phi_{Q^*}$ + \RP + stronger \ZP with detached FKL.} Deep Bisimulation for Control (DBC) trains the state encoder $\phi$ with several auxiliary losses, including \RP and \ZP. The \ZP loss uses a forward KL objective with a detached target. Their main contribution is the introducation of the bisimulation metric~\citep{ferns2004metrics} into state representation learning: for any $s_i,s_j\in \mathcal S$ and $a_i,a_j\in\mathcal A$, 
\begin{align}
\|\phi(s_i) - \phi(s_j)\|_1 = |R(s_i,a_i) - R(s_j,a_j)| + \gamma W (\mathbb P_\theta( z'\mid  \phi(s_i),a_i), \mathbb P_\theta (z' \mid \phi(s_j),a_j)), \tag{metric}
\end{align}
where $W$ is Wasserstein distance and $\mathbb P_\theta$ is modeled as a Gaussian.
The metric condition enforces the latent space to be structured with a $\ell_1$ metric. They train $\phi$ satisfying the metric condition by minimizing the mean square error on it as another auxiliary loss. This leads to a stronger \ZP condition.

\paragraph{PBL~\citep{guo2020bootstrap}: $\phi_{Q^*}$ + indirect multi-step \ZP.} Predictions of Bootstrapped Latents (PBL) designs two auxiliary losses, reverse prediction and forward prediction, for their history encoder $\phi$, transition model $\theta$, observation encoder $f$, and projector $g$: 
\begin{align}
\min_{f,g} \E{h}{\|g(f(o)) - \phi(h)\|_2^2}, \tag{Reverse}\\
\min_{\phi,\theta} \E{h,a,o'}{\|\theta(\phi(h),a) - f(o')\|_2^2}. \tag{Forward}
\end{align}
To understand their connection with \ZP, assume the two losses reach zero with $ \phi(h) = g(f(o))$ and $\theta(\phi(h),a) = \E{o'\sim P(\mid h,a)}{f(o')}$ for any $h,a$, although in theory this may be unrealizable. Furthermore, assume deterministic transition, then 
\begin{align}
g(\theta(\phi(h),a)) = g(f(o')) = \phi(h').
\end{align}
Therefore, in deterministic environments, reverse and forward prediction together is equivalent to \ZP if they reach the optimum. They also adopt multi-step version of their loss with a horizon of $20$. 
While forward and reverse prediction both appear critical in this work, the follow-up work BYOL-explore~\citep{guo2022byol} removes reverse prediction.

\paragraph{Successor Representations and Features~\citep{barreto2017successor,lehnert2020successor}: $\phi_{Q^*}$ + \RP + weak \ZP.} 
Here, we introduce successor features (SF) with our notation. Suppose the expected reward function can be computed as  
\begin{align}
\label{eq:ER_SR}
\E{}{r \mid s,a} = g(\phi(s),a)^\top w, \quad \forall s,a,
\end{align}
where  $\phi:\mathcal S \to \mathcal Z$ is a state encoder  and  $g: \mathcal Z \times \mathcal A \to \R^d$ is called state-action feature extractor, and $w\in \R^d$ are weights\footnote{Although it is linear \wrt $w$, it can recover any reward function, \eg when $\phi(s) = s$ and $g(s,a)_i = \E{}{r \mid s,a}$ for some $i$.}. In our notation, \autoref{eq:ER_SR} is \RP condition for $\phi$.  

As a special case, in tabular MDPs with finite state and action spaces with state-dependent reward $R(s)$, let $\phi(s) \in \{0,1\}^{|\mathcal S|}$ be one-hot state representation, and let $g(\phi(s),a)=\phi(s)$ and weight $w_{s} = \E{}{r \mid s}$, this satisfies \autoref{eq:ER_SR}. This special case is known as \textbf{successor representation} (SR) setting~\citep{dayan1993improving}. 
In deep SR~\citep{kulkarni2016deep,lehnert2020successor}, they allow learning $\phi$ with assuming $g(\phi(s),a)=\phi(s)$.

The $Q$-value function of a policy $\pi$ can be rewritten as
\begin{align}
Q^\pi(s,a) &= \E{\pi}{\sum_{t=0}^\infty \gamma^t r_t \mid S_0 =s, A_0=a}\\ 
&= \E{\pi}{\sum_{t=0}^\infty \gamma^t g(\phi(s_t),a_t)^\top w  \mid S_0 =s, A_0=a}\\ 
&= \E{\pi}{\sum_{t=0}^\infty \gamma^t g(\phi(s_t),a_t)  \mid S_0 =s, A_0=a}^\top w \\ 
\label{eq:linear_q}
&\defeq \psi^\pi(s,a)^\top w,
\end{align}
where $\psi^\pi(s,a)$ is called successor features~\citep{barreto2017successor}, a geometric sum of future $g(\phi(s),a)$. 
Although $\psi^\pi$ can belong to any function class, following deep SR~\citep{kulkarni2016deep,lehnert2020successor}, we assume it is parametrized by the state encoder as $\psi^\pi(s,a) = f^\pi(\phi(s),a)$ where $f^\pi:\mathcal Z \times \mathcal A \to \R^d$. 
Then, by plugging \autoref{eq:linear_q} in Bellman equation $Q^\pi(s,a) = \E{s',a'\sim \pi}{R(s,a) + \gamma  Q^\pi(s',a')}$, we have 
\begin{align}
\label{eq:SF}
&f^\pi(\phi(s),a) = g(\phi(s),a) + \gamma \E{s',a'\sim \pi}{f^\pi(\phi(s'),a') }.
\end{align}
Therefore, \autoref{eq:SF} can be viewed as a \textbf{weak} version of \ZP, because given any current latent state and action pair $(\phi(s),a)$, \autoref{eq:SF} can predict the expectation of some function of next latent state $\phi(s')$. \ZP can imply \autoref{eq:SF} because it can predict exactly the distribution of next latent state. 

With a combination of \RP (\autoref{eq:ER_SR}), $\phi_{Q^*}$ (implied by \autoref{eq:SF} when $\pi$ is optimal), and a weak version of \ZP, we show that the state encoder that successor features learn, belongs to a weak version of $\phi_L$.

As a special case, in Linear Successor Feature Model (LSFM)~\citep[Theorem~2]{lehnert2020successor}, they show that SF is \textbf{exactly} the bisimulation ($\phi_L$) under several assumptions: finite action and latent space, the successor features $f^\pi(z,a) = F_a z$ is a linear function, and the policy $\pi:\mathcal Z \to \Delta(\mathcal A)$ conditions on latent space. 
However, here we point it out that with the assumptions above implies \EZP (not necessarily \ZP), thus, still a \textbf{weak} version of bisimulation.

Following~\citet{lehnert2020successor}, assume the finite latent space is composed of one-hot vectors: $\mathcal Z = \{e_1, e_2,\dots,e_n\}$, we can construct a matrix $F^\pi\in \R^{d\times n}$ with each column $F^\pi(i) = \E{a\sim \pi(\mid e_i)}{F_{a}e_i}$. 
\begin{align}
&\frac{1}{\gamma}(f^\pi(\phi(s),a) - g(\phi(s),a)) =\E{s',a'\sim \pi}{f^\pi(\phi(s'),a') } \\
&\quad= \E{s'\sim P(\mid s,a),a'\sim \pi(\mid \phi(s'))}{F_{a'}\phi(s')} \\ 
&\quad=\E{s'\sim P(\mid s,a)}{F^\pi \phi(s')} = F^\pi\E{s'\sim P(\mid s,a)}{\phi(s')}  .
\end{align}
By~\citep[Lemma 4]{lehnert2020successor}, $F^{\pi}$ is invertible, thus there exists a function $J:\mathcal Z \times \mathcal A \to \mathcal Z$ such that $J(\phi(s),a) = \E{s'\sim P(\mid s,a)}{\phi(s')}$, \ie, \EZP holds. 

\paragraph{EfficientZero~\citep{ye2021mastering}: $\phi_{Q^*}$ + \RP + multi-step \ZP with detached $\cos$.} EfficientZero improves MuZero~\citep{schrittwieser2020mastering} by introducing \ZP loss as one of their main contributions. We consider it especially crucial to planning algorithms because \ZP enforces the latent model to be accurate. Similar to SPR~\citep{schwarzer2020data}, they use $5$-step $\cos$ objective with a projection on latent states, and add image data augmentation for visual RL tasks. 

\paragraph{RPC~\citep{eysenbach2021robust}: $\phi_{\pi^*}$ + \ZP with online forward KL.} From the perspective of information compression, robust predictive control (RPC) aims to jointly learn the encoder of policy $\mathbb P_\phi(z\mid s)$ and the latent policy $\pi_z(a\mid z)$ in MDPs.  The policy $\pi(a\mid s)$ is not only maximizing return, but also imposed a constraint on $\E{\pi}{I(s_{1:\infty};z_{1:\infty})} \le C$ where $C > 0$ is a predefined constant. 
By applying variational information bottleneck,
this constraint induces the algorithm RPC to maximize the following objective \wrt $\phi$ and $\theta$ (see their Eq. 6):  
\begin{align}
\mathcal L(\phi, \theta; s_t,a_t) &= \E{z_{t}\sim \mathbb P_\phi(\mid s_{t}),s_{t+1}\sim P(\mid s_t,a_t), z_{t+1}\sim \mathbb P_\phi(\mid s_{t+1})}{\log \frac{\mathbb P_\theta(z_{t+1}\mid z_t,a_t)}{\mathbb P_\phi(z_{t+1}\mid s_{t+1})}} \\ 
&= - \E{z_{t}\sim \mathbb P_\phi(\mid s_{t}),s_{t+1}\sim P(\mid s_t,a_t)}{\kl(\mathbb P_\phi(z_{t+1}\mid s_{t+1}) \mid\mid\mathbb P_\theta(z_{t+1}\mid z_t,a_t))}
\end{align}
which is exactly the practical forward KL objective~\autoref{eq:kl}. In practice, the authors formulate it as constrained optimization and use gradient descent-ascent to update the encoder and Lagrange multiplier. In addition, they also use this objective as an intrinsic reward to regularize the latent policy's reward-maximizing objective. 
It is worth noting that while RPC aims to learn the \ZP condition along with reward maximization, it does not explicitly learn representations to fulfill the \RP or $\phi_{Q^*}$ conditions. As a result, we can consider it as an approach that approximates self-predictive representations.

\paragraph{ALM~\citep{ghugare2022simplifying}: $\phi_{Q^*}$ + multi-step \ZP with EMA reverse KL.} Aligned Latent Models (ALM) is based on variational inference, and aims to learn the latent model $\mathbb P_\theta(z' \mid z,a)$, the state encoder $\mathbb P_\phi(z\mid s)$ and the latent policy $\pi_z(z)$ to jointly maximize the lower bound of the expected return. The objective of their encoder includes maximizing the return and \ZP loss, instantiated as $3$-step reverse KL with an EMA target. Specifically, given a tuple of $(s,a,s')$, the $1$-step objective for their encoder is computed as 
\begin{align}
\min_\phi \E{z \sim \mathbb P_\phi(\mid s)}{-R_z(z,a) + \kl(\mathbb P_\theta(z'\mid z,a) \mid \mid \mathbb P_\phi(z' \mid s')) - \E{z'\sim \mathbb P_\theta(\mid z,a)}{ Q^\pi(z',\pi_z(\overline{z'}))}},
\end{align}
where $R_z(z,a)$ is the latent reward, learned by the \RP condition (with $\phi$ detached), and $\overline{z'}$ indicates stop-gradient.  With the latent reward and also their intrinsic rewards, they perform SVG algorithm~\citep{heess2015learning} for policy optimization with a planning horizon of $3$ steps. We provide a detailed description of ALM and its variants in \autoref{sec:mdp_details}. 

\paragraph{AIS~\citep{subramanian2022approximate}: \RP + \ZP with detached $\ell_2$ or forward KL in their approach, while \RP + \OP with detached $\ell_2$ in their experiments.} Approximate Information States (AIS) adopts a phased training framework where the history encoder $\phi$ learns from \RP instead of maximizing returns. In their approach section~\citep[Sec.~6.1.2]{subramanian2022approximate}, they propose using MMD with $\ell_2$ distance-based kernel $k_d$ to learn \ZP, and detach the target. The distance-based kernel~\citep{sejdinovic2013equivalence} takes a pair of latent states $z_1,z_2 \in \mathcal Z$ as inputs, and is defined as $k_d(z_1,z_2) = \frac{1}{2}(d(z_0,z_1) + d(z_0,z_2) - d(z_1,z_2))$ where $z_0\in \mathcal Z$ is arbitrary. In this case, $d(z_1,z_2) = \|z_1-z_2\|_2^2$ is $\ell_2$ distance.

Let $f_\phi(h)$ be the deterministic encoder, $\mathbb P_{\theta}(z'\mid f_\phi(h),a) \defeq \mathbb P_{\phi,\theta}$ be the predicted next latent distribution, and $\mathbb Q_{\phi}(z'\mid h,a)$ be real next latent distribution. The MMD with $k_d$ can be reduced to $\ell_2$ distance between the expectations of two distributions:
\begin{align}
&\text{MMD}^2_{k_d}(\mathbb P_{\phi,\theta}, \mathbb Q_{\phi}; h,a) \\
&= -\E{z'_1, z'_2\sim \mathbb P_{\phi,\theta}}{d(z'_1,z'_2)} + 2\E{z'_1\sim \mathbb P_{\phi,\theta}, z'_2\sim \mathbb Q_{\phi}}{d(z'_1,z'_2)} - \E{z'_1, z'_2\sim \mathbb Q_{\phi}}{d(z'_1,z'_2)} \\
&= -\E{z'_1, z'_2\sim \mathbb P_{\phi,\theta}}{\|z'_1-z'_2\|_2^2} + 2\E{z'_1\sim \mathbb P_{\phi,\theta}, z'_2\sim \mathbb Q_{\phi}}{\|z'_1-z'_2\|_2^2} - \E{z'_1, z'_2\sim \mathbb Q_{\phi}}{\|z'_1-z'_2\|_2^2}  \\
&= 2\|\E{z'\sim \mathbb P_{\phi,\theta}}{z' \mid h,a} - \E{z'\sim \mathbb Q_{\phi}}{z' \mid h,a}\|_2^2.
\end{align}
Therefore, the MMD objective can be viewed as \ZP with $\ell_2$ distance (\ie, \EZP). They also propose forward KL to instantiate \ZP loss. 
Nevertheless, AIS~\citep{subramanian2022approximate} do not show experiment results on learning \ZP. Instead, they and the follow-up works~\citep{patil2022learning,seyedsalehi2023approximate} implement AIS by learning \OP loss with MMD objectives, resulting in learning \textit{observation-predictive} representations.
Another follow-up work, Discrete AIS~\citep{yang2022discrete}, learns \ZP loss with $\ell_2$ objective in a discrete latent space, so that they can apply value iteration.

\paragraph{TD-MPC~\citep{hansen2022temporal}: $\phi_{Q^*}$ + \RP + multi-step \ZP with EMA $\ell_2$.} Temporal Difference learning for Model Predictive Control (TD-MPC) uses a planning horizon of $5$ for the encoder objective and the latent value objective with TD learning. TD-MPC also uses MPC for action selection during inference. They find that learning \ZP works better than learning \OR or not learning \ZP in the DM Control suite.

\paragraph{TCRL~\citep{zhao2023simplified}: RP + multi-step ZP with EMA $\cos$.} Temporal consistency reinforcement learning (TCRL) simplifies TD-MPC~\citep{hansen2022temporal} by removing the planning component, replacing $\ell_2$ loss with $\cos$ loss, and detaching the encoder parameters during value function learning. They validate their approach on the state-based DM Control suite.
Although the paper refers to TCRL as \textit{minimalist} for learning representations, it is worth noting that TCRL is more complicated than our approach, as it still requires reward prediction and multi-step prediction.

\subsection{Observation-Predictive Representations}

\paragraph{PSR~\citep{littman2001predictive} and related work: \Rec + multi-step \OP and \RP.} The original predictive state representation (PSR)~\citep{littman2001predictive} aims to learn a history encoder $\phi$ and a linear transition model $P_O$ such that
\begin{align}
P(o_{t+1:t+k} \mid h_t, a_{t:t+k-1}) = P_O(o_{t+1:t+k} \mid \phi(h_t), a_{t:t+k-1}), \quad \forall h, a, o,
\end{align}
which implies multi-step \OP (defined in \autoref{prop:multi-op}) in POMDPs.
However, the original PSR struggles to model the future rewards, formally shown by~\citet[Proposition 1]{baisero2021reconciling}. Therefore, 
PSR-IP algorithm~\citep{james2004planning} predicts future rewards by incorporating rewards as additional observations, assuming that rewards are accessible during policy inference. R-PSR~\citep{baisero2021reconciling} generalizes the PSR framework to have the condition of multi-step \RP without adding rewards into a history. Thus, R-PSR bridges the gap between PSR and belief-based approaches.
Belief trajectory equivalence~\citep{castro2009equivalence} also introduces multi-step \RP to PSR and shows that single-step \OP and \RP do not necessarily imply multi-step \OP and \RP in POMDPs, summarized in \autoref{prop:multi-op}. In this sense, PSR is a stronger notion of belief abstraction. 

\paragraph{Causal state representations~\citep{zhang2019learning}: \Rec + \OP + \RP.}  This work connects observation-predictive representations in POMDPs with causal state models in computational mechanics~\citep{shalizi2001computational}. Specifically, they show that belief trajectory equivalence (\Rec + multi-step \OP and \RP)~\citep{castro2009equivalence} implies a causal state of a stochastic process, where \RP means reward \textit{distribution} prediction. The resulting abstract MDP is a causal state model or an $\epsilon$-machine, generating minimal sufficient representations for predicting future observations.  
In the implementation, they train a deterministic RNN encoder and a deterministic transition model to satisfy \OP and \RP conditions, and also train a latent Q-value function using Q-learning by freezing encoder parameters.
Optionally, they also train a discretizer on the latent space in finite POMDPs. 

\paragraph{Belief-based methods~\citep{Hafner2019LearningLD,Hafner2020DreamTC,han2019variational,lee2019stochastic}: \RP + \OR + \ZP with online forward KL.} As a major approach to solving POMDPs, belief-based methods extends belief MDPs~\citep{kaelbling1998planning} to deep RL through variational inference, deriving the encoder objective as ELBO. 
Let the latent variables are $z_{1:T}$, the world model $p(o_{1:T},r_{1:T} \mid a_{1:T})$, and the posterior are $q(z_{1:T} \mid o_{1:T},a_{1:T})$ with the factorization: 
\begin{align}
&p(z_{1:T+1},o_{1:T+1},r_{1:T} \mid a_{1:T}) 
= p(z_1) p(o_1 \mid z_1) \prod_{t=1}^{T}   p(r_{t} \mid z_t,a_t) p(z_{t+1} \mid z_{t}, a_{t}) p(o_{t+1} \mid z_{t+1}),\\
&q(z_{1:T+1} \mid h_{T+1}) = \prod_{t=0}^T q(z_{t+1} \mid h_{t+1}) = \prod_{t=0}^T q(z_{t+1} \mid z_{t}, a_t,o_{t+1}),
\end{align}
where $h_{t+1} = (h_{t},a_t,o_{t+1})$ in our notation. The log-likelihood has a lower bound: 
\begin{align}
&\E{h_{T+1},r_{1:T}}{\log p_\theta (o_{1:T+1},r_{1:T} \mid a_{1:T})} \\
&= \E{h_{T+1},r_{1:T}}{\log  \E{q(z_{1:T+1} \mid h_{T+1})}{ \frac{p(z_{1:T+1},o_{1:T+1},r_{1:T} \mid a_{1:T})}{q(z_{1:T+1} \mid h_{T+1})}}}\\
&\ge \E{h_{T+1},r_{1:T},z_{1:T}\sim q( \cdot \mid h_{T+1})}{ \log \frac{p(z_{1:T},o_{1:T+1},r_{1:T} \mid a_{1:T})}{q(z_{1:T+1} \mid h_{T+1})}}\\
&=  \E{h_{T+1},r_{1:T},z_{1:T+1}\sim q(h_{T+1})}{\sum_{t=0}^T \underbrace{\log p(o_{t+1} \mid z_{t+1})}_{(1)} + \underbrace{\log p (r_t \mid z_t,a_t)}_{(2)}- \underbrace{\log \frac{q(z_{t+1} \mid h_{t+1})}{p (z_{t+1} \mid z_{t}, a_{t})}}_{(3)}}.
\end{align}
When $p,q$ are trained to optimal, the first term becomes \OR condition and the second term becomes reward distribution matching that implies \RP. The third term with expectation can be written as $\E{z_t,h_{t+1}}{\kl(q(z_{t+1} \mid h_{t+1}) \mid \mid p (z_{t+1} \mid z_{t}, a_{t}))}$, which is exactly our practical forward KL objective \autoref{eq:kl} to learn \ZP. From our relation graph (\autoref{fig:relation_main}; \autoref{prop:or_zp}), \ZP + \OR imply \OP, thus belief-based methods aim to approximate observation-predictive representation (\RP + \OP). Normally, they use an online target in forward KL, because they have \OR signals that can help prevent representational collapse. They also train encoders without maximizing returns. Finally, \citet{baisero2020learning} replaces \OR with \OP in their objectives to learn the same representations.

We can also build the connections between \OR and \RP objectives and maximizing mutual information. Let $P(o,z)$ be the marginal joint distribution of observation and latent state at the same time-step, where $P(o',z') = \int P(o',z',h,a) dhda = \int P(h,a)P(o'\mid h,a) P(z'\mid h')dhda$. Consider,
\begin{align}
\label{eq:MI_OR}
\mathbb I(o'; z') &= \E{o',z'\sim P(o',z')}{\log \frac{P(o',z')}{P(o') P(z')}}\\
&=\E{o',z'\sim P(o',z')}{\log \frac{P(o'\mid z')}{P(o')}}\\
&= \E{o',z'\sim P(o',z')}{\log P(o'\mid z')} + \mathbb H(P(o')) \\
&= \E{h,a\sim P(h,a),o'\sim P(\mid h,a),z'\sim P(\mid h')}{\log P(o'\mid z')} + \mathbb H(P(o')).
\end{align}
Since the entropy term is independent of latent states, the \OR objective in belief-based methods is \textbf{exactly} maximizing the $\mathbb I(o; z)$. Similarly, the \RP objective in belief-based methods is exactly maximizing $\mathbb I(r; z)$.

\paragraph{OFENet~\citep{ota2020can}: $\phi_{Q^*}$ + \OP.} Online Feature Extractor
Network (OFENet) trains the state encoder using an auxiliary task of \OP loss with $\ell_2$ distance. This is perhaps the most related algorithm to our Algo.~\ref{pseudo-code} for learning $\Phi_O$. They show strong performance of their approach over model-free baseline in standard MuJoCo benchmark. 
Follow-up work~\citep{lange2023comparing} empirically find that $\phi_{Q^*}$ + \RP slightly improves up model-free RL, but much worse than $\phi_{Q^*}$ + \OP in MuJoCo benchmark. 

\paragraph{SAC-AE~\citep{yarats2021improving}: $\phi_{Q^*}$ + \OR.} Soft Actor-Critic with AutoEncoder (SAC-AE) trains the state encoder with an auxliary task of \OR loss with forward KL and also $\ell_2$-regularization. They detach the state encoder in policy objective. As in MDPs, \OR implies \OP (\autoref{prop:mdp_or}), SAC-AE also approximates observation-predictive representation.

\subsection{Other Related Representations}

\paragraph{UNREAL~\citep{jaderberg2016reinforcement}, Loss is its own Reward~\citep{shelhamer2016loss}.} These works make early attempts at auxiliary task design for RL. 
UNREAL trains recurrent A3C agent with several auxiliary tasks, including reward prediction (\RP), pixel control and value function replay.
Loss is its own Reward trains A3C agent with several auxiliary tasks, including reward prediction (\RP), observation reconstruction (\OR),  inverse dynamics, and a proxy of forward dynamics (\OP) that finds the corrupted observation from a time series. 
Among them, inverse dynamics condition in MDPs is that  
\begin{align}
\exists P_{\text{inv}}: \mathcal Z\times \mathcal Z \to \Delta(\mathcal A), \quad  \st \quad P_{\text{inv}}(a\mid \phi(s), \phi(s')) =  P(a \mid s, s'), \quad \forall s,a,s',
\end{align}
but this condition does not direct relation with forward dynamics (\OP). 

\paragraph{VPN~\citep{oh2017value}, MuZero~\citep{schrittwieser2020mastering}: $\phi_{Q^*}$ + \RP.} From \autoref{thm:EG_ZP_implies_Er}, we know that $\phi_{Q^*}$ + \RP is implied by $\phi_{Q^*}$ + \ZP, thus this representation lies between $\phi_{Q^*}$ and $\phi_L$. 
Both VPN and MuZero learn the shared state encoder and latent model from maximizing the return and predicting rewards. MuZero also predicts actions. They follow the value-equivalence principle~\citep{grimm2020value} to learn value-equivalent models. Their policies are learned by the MCTS algorithm.  

\paragraph{E2C~\citep{watter2015embed} and World Model~\citep{ha2018recurrent}: \ZP + \OR.} They are similar to belief-based methods, but remove the reward prediction loss from the encoder objective. Instead, reward signals are only accessible to latent policies or values. 

\paragraph{Contrastive representation learning in RL (CURL~\citep{laskin2020curl}, DRIML~\citep{mazoure2020deep}, ContraBAR~\citep{choshen2023contrabar}): $\phi_{Q^*}$ (\RP) + weak \OP (\OR).}
CURL ($\phi_{Q^*}$ + weak \OR) introduces contrastive learning using the infoNCE objective~\citep{oord2018representation} as an auxiliary task in MDPs. InfoNCE between positive and negative examples is shown to be a lower bound of mutual information between input and latent state variables~\citep{poole2019variational}. In MDPs, it is a lower bound of $\mathbb I(s;z)$, which corresponds to \OR objectives \autoref{eq:MI_OR}. Therefore, CURL can be interpreted as maximizing a lower bound of \OR. 

DRIML ($\phi_{Q^*}$ + weak \OP) proposes an auxiliary task named InfoMax in MDPs. In its single-step prediction variant, InfoMax maximizes the lower bound of $\mathbb I(z';z,a)$ via the infoNCE objective. Similar to the analysis~\citep{rakelly2021mutual}, by data processing inequality:
\begin{align}
&\mathbb I(z';z,a) \le \mathbb I(z';s,a) \le  \mathbb I (s'; s,a), \\ 
&\mathbb I(z';z,a) \le \mathbb I(s';z,a) \le  \mathbb I (s'; s,a).
\end{align}
When all equalities hold (\eg $\phi$ satisfies \OR), these imply $z' \ind s,a \mid z,a$ (\ZP) and $s'\ind s,a \mid z,a$ (\OP).  

ContraBAR (weak \RP and weak \OP) introduces infoNCE objectives to meta-RL, which requires incorporating reward signals into observations when viewed as POMDPs~\citep{ni2021recurrent}. Similar to DRIML, in its single-step prediction variant, the objective is to maximize the lower bound of mutual information of $\mathbb I(z';z,a)$ where $z$ is a joint representation of state $s$ and reward $r$. As shown in the ContraBAR paper~\citep[Theorem~4.3]{choshen2023contrabar}, under certain optimality condition, the objective can lead to learning \RP and \OP conditions. 

\paragraph{Learning Markov State Abstraction~\citep{allen2021learning}: $\phi_{Q^*}$ + \ZM.}  From \autoref{prop:zm}, we know that \ZM is implied by \ZP, thus representation lies between $\phi_{Q^*}$ and $\phi_L$. They show that \ZM can be implied by inverse dynamics and density ratio matching in MDPs. Thus, they train on these two objectives as auxiliary losses. 

\paragraph{MICo~\citep{castro2021mico}: $\phi_{Q^*}$ + metric.} With a state encoder $\phi$, matching under Independent Coupling (MICo) defines a distance metric $U_\phi$ in the state space. For any pair of states $x,y \in \mathcal S$, 
\begin{align}
U_\phi(x,y) = |r_x^\pi -r_y^\pi| + \gamma \E{x'\sim P_x^\pi, y'\sim P_y^\pi}{U_\phi(x',y')},\tag{metric}
\end{align}
where $r_x^\pi = \E{a\sim \pi(\mid x)}{R(x,a)}$ and $P_x^\pi(x'\mid x) = \E{a\sim \pi(\mid x)}{P(x'\mid x,a)}$. The metric $U_\phi$ is parameterized with
\begin{align}
U_\phi(x,y) = \frac{1}{2}(\|\phi(x)\|_2^2 +\|\phi(y)\|_2^2) + \beta \arctan(\sqrt{1-\cos(\phi(x),\phi(y))^2}, \cos(\phi(x),\phi(y))).
\end{align}
They learn the MICo metric by an auxiliary loss using mean squared error. 

\paragraph{Denoised MDPs~\citep{wang2022denoised}: \OR + \RP +  \ZP in a factorized latent space (implying $\phi_{\pi^*}$, not $\phi_{Q^*}$).} This work aims to learn a state abstraction that ignores components that are either reward-irrelevant or uncontrollable. Such an abstraction can retain the optimal policy while not necessarily preserving optimal value functions. Consequently, denoised MDPs can be conceptualized as approximating $\phi_{\pi^*}$.
Technically, the authors postulate that the latent state of an MDP is composed of elements $(x,y,z)$ where the transition in $y$ is independent of actions. Additionally, the reward function $r(s)$, independent of $z$, is decomposed into $r_x(x)$ and $r_y(y)$. Thus, the optimal policy (though not its value) can only depend on the latent state component $x$.  

In practice, they introduce variational objectives to learn the encoder $p(x,y,z \mid s)$ with observation reconstruction (\OR), reward prediction (\RP), and next latent state prediction (\ZP) using online forward KL divergences. The structure of the latent state space helps the partial encoder $p(x \mid s)$ to gravitate towards $\phi_{\pi^*}$ abstraction, despite the absence of a theoretical guarantee. Finally, they use model-free RL to optimize a policy on the latent $x$  space. 

\paragraph{TD7~\citep{fujimoto2023sale}: \ZP with detached $\ell_2$.} TD7 algorithm is introduced for addressing MDPs and evaluated on the MuJoCo benchmark. TD7 learns a state encoder using \ZP loss with detaching the next latent states, which performs better than EMA version.  
They use $\ell_2$ loss and normalize latent states by average $\ell_1$ norm, which performs better than  $\cos$ loss and other normalization methods.
They find that training with \ZP loss only is slightly better than training with \ZP + \RP, and much better than end-to-end training (\ZP + $\phi_{Q}^*$).  
Lastly, it is noteworthy that in TD7, the critic not only takes a state $s$ and an action $a$ as inputs, but also the latent state $f_\phi(s)$ and the predicted next latent state $g_\theta(f_\phi(s),a)$, which is named as \textit{state-action embedding} in the paper.

\section{Motivating Our Hypotheses}
\label{sec:motivation}
Here we provide our motivation for our hypotheses shown in \autoref{sec:experiments}. 

\begin{itemize}[leftmargin=*,itemsep=0pt, topsep=0pt]
    \item \textbf{Motivating the sample efficiency hypothesis.} 
    The performance of deep RL algorithms is notably influenced by task structure, and no single algorithm consistently outperforms others across all tasks~\citep{wang2019benchmarking,li2022does,ni2021recurrent}. Common wisdom suggests that certain algorithms excel in specific types of tasks~\citep{mohan2023structure}.  For instance, self-predictive representations are often effective in distracting tasks~\citep{zhang2020learning,zhao2023simplified}, while observation-predictive representations typically perform well in sparse-reward scenarios~\citep{zintgraf2020exploration,zhang2021metacure}. However, these methods often incorporate additional complexities like intrinsic rewards and metric learning or are primarily evaluated in pixel-based tasks. 
    
    Given these considerations, we propose the use of our minimalist algorithm as a tool to focus solely on the impact of representation learning in vector-based tasks. This approach aims to provide a clearer understanding of how different representation learning strategies affect sample efficiency in various task structures (including popular standard benchmarks), without the confounding factors present in more complex algorithms or environments.
    
    \item \textbf{Motivating the distraction hypothesis.}  The belief that algorithms predicting observations tend to underperform in distracting tasks is supported by several studies~\citep{zhang2020learning,okada2021dreaming,fu2021learning,deng2022dreamerpro}. The challenge arises from the need for these models to predict every detail of observations, including irrelevant features, which can be extremely difficult due to randomness and high dimensionality. Similar to the motivation in our sample efficiency hypothesis, prior works primarily focus on complex algorithms evaluated in pixel-based tasks, often with real-world video backgrounds as distractors. 

    Considering these, we propose a shift towards studying the impact of distractions using a minimalist algorithm in simpler, configurable environments. This aligns with the setting by \citet{lambrechts2022recurrent}  demonstrating that, in the toy tasks with temporally-correlated Gaussian noises, recurrent model-free RL learns decision-relevant representations compared to a belief-based approach. This experiment design aims to isolate and understand the specific effects of distracting elements in tasks, providing a more straightforward and controlled setting for analysis.

    \item \textbf{Motivating the end-to-end hypothesis.} According to our \autoref{thm:EG_ZP_implies_Er}, learning an encoder end-to-end with the auxiliary task of \ZP can implicitly learn the reward prediction conditioned on its optimality, potentially making it comparable to the phased learning. However, this is a theoretical prediction and may not necessarily translate to practical scenarios, particularly considering that RL agents rarely achieve global optima. On the other hand, prior works~\citep{schwarzer2020data,ghugare2022simplifying} have shown the success of the end-to-end learning, but these algorithms incorporate other moving components (multi-step prediction, intrinsic rewards) and are not directly applicable to POMDPs. 
    
    These limitations underscore the importance of empirically testing whether the benefits of the end-to-end learning extend to POMDPs when employing a minimalist approach in representation learning.

    \item \textbf{Motivating the \ZP objective hypothesis.} Our \autoref{prop:l2} and \autoref{prop:fdiv} suggest that it suffices to use $\ell_2$ objective in deterministic tasks while KL divergences might be more effective in stochastic ones. However, these are theoretical assumptions that do not fully account for the complexities of the learning process. Additionally, most existing research tends to focus on a single objective type in deterministic settings (as summarized in \autoref{tab:main_prior_work}), leaving the performance of alternative objectives, particularly in stochastic tasks, largely unexplored. A notable exception is AIS~\citep{subramanian2022approximate} which discusses various \ZP objectives but lacks practical evaluation on them. 
    
    These gaps in the literature motivate us to undertake a thorough comparison of these \ZP objectives in practical settings.

    \item \textbf{Motivating the \ZP stop-gradient hypothesis.} Our \autoref{thm:collapse} suggests that applying stop-gradient to \ZP targets could help mitigate representational collapse. However, this prediction is based on linear models without incorporating RL loss, which is a significant departure from deep RL scenarios. While most prior studies focus on one type of \ZP target without delving into collapse issues (as summarized in \autoref{tab:main_prior_work}), SPR~\citep{schwarzer2020data} is an exception, comparing online and EMA encoders in Atari tasks. Nonetheless, SPR's analysis focuses on return performance and lacks direct evidence of representational collapse. 
    
    Addressing these research gaps, we aim to conduct an extensive comparison of \ZP targets in both MDPs and POMDPs. Our analysis includes providing direct evidence through the estimation of representational rank.
    
\end{itemize}

\section{Experimental Details}
\label{sec:experiment_details}

\subsection{Small Scale Experiments to Illustrate \autoref{thm:collapse}} 

In this section, we discuss the details of the experiments used to explore the empirical effects of using stop-gradient to detach the \ZP target in the self-predictive loss. First, we discuss the details shared between both domains and then discuss domain-specific details.

We learn on data obtained by rolling out 10 trajectories under a fixed, near-optimal policy starting from a random state. Trajectories are followed until termination or until 200 transition have been observed, whichever happens first. The encoder, $\phi \in \mathbb{R}^{k\times 2}$ where $k$ is the number of observed features, is updated using full gradient descent with a small learning rate, $\alpha = 0.01$, for 500 steps. At every 10 steps, the absolute cosine similarity between the 2 columns of $\phi$ is computed, i.e., $f(x, y) = |x^\top y |/ (||x||_2 ||y||_2)$ and the results are plotted in \autoref{fig:innerproduct}. The optimal transition model $\theta^* = \begin{bmatrix} {\theta_z^*}^\top & {\theta_a^*}^\top \end{bmatrix}^\top$ is solved using singular value decomposition and the Moore-Penrose inverse to minimize the linear least-squares objective: 
\begin{align}
    \left\lVert \begin{bmatrix}
    \phi^\top S & A
    \end{bmatrix} \begin{bmatrix}
            \theta_z \\
            \theta_a
        \end{bmatrix} - \widetilde{\phi}^\top S' \right\rVert_2,
\end{align}
where $S$ and $S'$ are matrices with each row corresponding to the sampled states (histories) and next states (histories), respectively, and, similarly, $A$ is a row-wise matrix of the sampled actions. The $\widetilde{\phi}$ is set as $\phi$ in online target, or $\bar \phi$ in detached target and EMA target where the Polyak step size $\tau=0.005$. 
To avoid numerical issues, singular values close to zero are discarded according to the default behavior of JAX's~\citep{jax2018github} \texttt{jax.numpy.linalg.lstsq} method when using \texttt{float32} encoding.

\paragraph{Mountain car~\citep{moore1990efficient}.}
We follow the dynamics and parameters used in~\citep[Example 10.1]{sutton2018reinforcement}. We encode states using a $10\times 10$ uniform grid of radial basis function (RBF), e.g., $f_i(s) = \exp(-(s - c_i)^\top \Sigma^{-1} (s - c_i))$ for an RBF centered on $c_i$, and with a width corresponding to $0.15$ of the span of the state space. Specifically, $\Sigma$ is diagonal and normalizes each dimension such that the width of the RBF covers $0.15$ in each dimension. As a result, the total number of features $k=100$. Actions are encoded using one-hot encoding and $|\mathcal{A}| = 3$. The policy used to generate data is an energy pumping policy which always picks actions that apply a force in the direction of the velocity and applies a negative force when the speed is zero.

\paragraph{Load-unload~\citep{meuleau2013solving}.}
Load-unload is a POMDP with 7 states arranged in a chain. There are 2 actions which allow the agent to deterministically move left or right along the chain, while attempting to move past the left-most or right-most state results in no movement. There are three possible observations which deterministically correspond to being in the left-most state, the right-most state or in any one of the 5 intermediate states. Observations and actions are encoded using one-hot encodings. The agent's state correspond to the history of observation and actions over a fixed window of size 20 with zero padding for a total of $k=98$ features ($k = 20 \times 3 + 19 \times 2$). Finally, the policy used to generate trajectories is a stateful policy that repeats the last action with probability $0.8$ and always starting with the \texttt{move-left} action.

\subsection{MDP Experiments in \autoref{sec:standard_mdps} and \autoref{sec:distracted_mdps}}
\label{sec:mdp_details}

\paragraph{Standard MuJoCo in \autoref{sec:standard_mdps}.} This is a popular continuous control benchmark from OpenAI Gym~\citep{brockman2016openai}. We evaluate on Hopper-v2 ($11$-dim), Walker2d-v2 ($17$-dim), HalfCheetah-v2 ($17$-dim), Ant-v2 ($111$-dim), and Humanoid-v2 ($376$-dim), where the numbers in the brackets are observation dimensions.

\paragraph{Distracting MuJoCo in \autoref{sec:distracted_mdps}.} We follow \citet{nikishin2022control} to augment the state space with a distracting dimension in Hopper-v2, Walker2d-v2, HalfCheetah-v2, and Ant-v2. The number of distractors varies from $2^4=16$ to $2^8=256$. Therefore, the largest observation dimension is $256+111=367$ in distracting Ant-v2. The distractors follow i.i.d. standard Gaussian $\mathcal N(0,I)$. 

Our algorithm setup in \autoref{sec:standard_mdps} and \autoref{sec:distracted_mdps} largely follows the code of ALM(3)~\citep{ghugare2022simplifying}\footnote{\url{https://github.com/RajGhugare19/alm}}. 
The original ALM paper also introduces an ablation of the method, ``ALM-no-model'',  which uses model-free RL (rather than SVG) to update the actor parameters. This ablation is structurally similar to our method, which similarly avoids using a model. 
However, ALM-no-model still employs a reward model and a latent model for learning representations, although not for updating policy. 

Below, we compare ALM and our minimalist $\phi_L$ implementation. We show ablation results comparing our method and ALM variants in \autoref{sec:ablation}.   

\paragraph{Differences between our minimalist $\phi_L$ (with reverse KL and EMA targets) and ALM.}
\begin{itemize}[leftmargin=*]
    \item \textbf{Reward model}: we remove reward models from both ALM(3) and ALM-no-model. It should be noted that although ALMs learn reward models, they do not update their encoders through reward prediction loss. 

    \item \textbf{Encoder objective}: our state encoder ($\phi$) is updated by \autoref{eq:update}. Given a probabilistic encoder $\mathbb P_\phi(z\mid s)$ and a probabilistic latent model $\mathbb P_\theta(z' \mid z,a)$ for MDPs, and the latent state $z_\phi \sim \mathbb P_\phi(z\mid s)$, we formulate our encoder objective for a data tuple $(s,a,r,s')$ as follows: 
    \begin{equation}
    \begin{split}
    \label{eq:our_rkl}
    \min_\phi\quad  &(Q_{\omega}(z_{\phi},a) - Q^{\text{tar}}(s,a,s',r))^2 -Q_{\omega}(\textcolor{red}{z_{\overline\phi}},\pi_{\nu} (\textcolor{blue}{z_{\phi}})) + \kl(\mathbb P_\theta(z' \mid z_\phi,a)\mid\mid \mathbb P_{\overline \phi}(z'\mid s') ).
    \end{split}
    \end{equation}
    In contrast, \textbf{ALM-no-model} employs a more complicated objective to train the state encoder $\phi$.  It performs a $1$-step rollout with the reward model $R_{\mu}(z,a)$ without a discount factor, and modify the stop-gradients on latent states within the $Q$-value. Given a data tuple $(s,a,s')$, the objective is
    \begin{equation}
    \begin{split}
    \label{eq:alm-no-svg}
    \min_\phi\quad &-R_\mu(z_\phi,a)  - \E{z'_{\phi,\theta} \sim \mathbb P_\theta(\mid z_\phi,a)}{Q_{\omega}(\textcolor{blue}{z'_{\phi,\theta}},\pi_{\nu} (\textcolor{red}{z'_{\overline \phi,\theta}}))}  + \kl(\mathbb P_\theta(z' \mid z_\phi,a)\mid\mid \mathbb P_{\overline \phi}(z'\mid s') ),
    \end{split}
    \end{equation}
    where they eliminate the mean-squared TD loss and maximize the $Q$-value through its latent states rather than actions, as done in our objective. \textbf{ALM(3)} extends  ALM-no-model by implementing a $3$-step rollout in the encoder objective. 

    To isolate the design of stop-gradients in $Q$-value and mean-squared TD error, we introduce \textbf{ALM(0)} that lies between ALM-no-model and ours with the $0$-step objective:
    \begin{align}
    \label{eq:alm-no-reward}
    \min_\phi\quad  -Q_{\omega}(\textcolor{blue}{z_{\phi}},\pi_{\nu} (\textcolor{red}{z_{\overline \phi}})) + \kl(\mathbb P_\theta(z' \mid z_\phi,a)\mid\mid \mathbb P_{\overline \phi}(z'\mid s') ).
    \end{align}

    \item \textbf{Actor objective}: our algorithm share the same actor objective with ALM-no-model and ALM(0), compared to ALM(3) which uses SVG with a $3$-step rollout~\citep{heess2015learning} and additional intrinsic rewards (\ie the negative reverse KL divergence term; see Eq.~8 in ALM paper for details). 
\end{itemize}

\paragraph{Implementation details for our minimalist algorithm learning $\phi_L$, and learning $\phi_O$ and $\phi_{Q^*}$.} 
We follow the exact implementation of the network architectures in ALM(3). The encoder, actor, and critic are parameterized as 2-layer neural networks with 512 hidden units.
The latent transition model (only used in learning $\phi_L$) and observation predictor (only used in learning $\phi_O$) are parameterized as 2-layer networks with 1024 hidden units. 
The probabilistic encoder, latent model and decoder output a Gaussian distribution with a diagonal covariance matrix. We apply layer normalization~\citep{ba2016layer} after the first layer of the critic network. 
We use ELU activation~\citep{clevert2015fast} and Adam optimizers~\citep{kingma2014adam} for all networks.

We enumerate the values of our hyperparameters in \autoref{tab:hparams_mdps}. If a hyperparameter is shared with ALM(3), we maintain the same value as that used in ALM(3)~\citep[Table~3]{ghugare2022simplifying}.

\begin{table}[t]
    \centering
    \caption{\textbf{Hyperparameters used in Markovian agents in standard and distracting MuJoCo.}}
    \vspace{0.5em}
    \begin{tabular}{cc}
    \toprule
       Hyperparameter  & Value  \\
       \midrule
        Discount factor ($\gamma$) & $0.99$ \\ 
        Warmup steps & $5000$ \\ 
        Target network update rate ($\tau$) & $0.005$ \\ 
        Replay buffer size & $10^6$ for Humanoid-v2 and $10^5$ otherwise \\ 
        Batch size & $512$ \\ 
        Learning rate & $0.0001$ \\ 
        Max gradient norm & $100$ \\ 
        Latent state dimension & $50$ \\ 
        Exploration stddev. clip & 0.3 \\ 
        Exploration stddev. schedule & linear$(1.0,0.1,100000)$\\
        Auxiliary loss coefficient ($\lambda$) & $1.0$ for \ZP-FKL, \ZP-RKL and \OP, and $10.0$ for \ZP-$\ell_2$ \\
        \bottomrule
    \end{tabular}
    \label{tab:hparams_mdps}
\end{table}

\subsection{POMDP Experiments in \autoref{sec:pomdps}}
\label{sec:pomdp_details}
\paragraph{MiniGrid in \autoref{sec:pomdps}.} This is a widely-used discrete gridworld benchmark from Farama foundation~\citep{gym_minigrid,MinigridMiniworld23}. In this benchmark, an agent has a first-person view to navigate a 2D gridworld with obstacles (\eg, walls and lava). Some tasks require the agent to pick up keys and open doors to navigate to the goal location. The agent's observations are symbolic (not pixel-based) with a size of $7\times 7 \times 3$ where $7\times 7$ is the spatial field of view, and the $3$ channels encode different semantics. The action space is discrete with $7$ options: turn left, turn right, move forward, pick up, drop, toggle, and done. 
Tasks are goal-oriented; the episode terminates immediately when the agent reaches the goal, or times out after a maximum of $T$ steps. 
Rewards are designed to encourage fast task completion. A successful episode yields a reward of $1-0.9 * H / T \in [0.1,1.0]$ at the terminal step, where $H$ denotes the total steps. Failed episodes result in a reward of $0.0$.   

We select 20 tasks in MiniGrid, following the recent work RQL-AIS~\citep{seyedsalehi2023approximate}. All of these tasks require memory in an agent. The tasks are grouped as follows:
\begin{itemize}[leftmargin=*]
    \item \textbf{SimpleCrossing} (4 tasks): SimpleCrossingS9N1, SimpleCrossingS9N2, SimpleCrossingS9N3, SimpleCrossing11N5 
    \item \textbf{LavaCrossing} (4 tasks): LavaCrossingS9N1, LavaCrossingS9N2, LavaCrossingS9N3, LavaCrossing11N5 
    \item \textbf{Unlock} (2 tasks): Unlock, UnlockPickup
    \item \textbf{DoorKey} (3 tasks): DoorKey-5x5, DoorKey-6x6, DoorKey-8x8
    \item \textbf{KeyCorridor} (3 tasks): KeyCorridorS3R1, KeyCorridorS3R2, KeyCorridorS3R3
    \item \textbf{ObstructedMaze} (2 tasks): ObstructedMaze-1Dl, ObstructedMaze-1Dlh
    \item \textbf{MultiRoom} (2 tasks): MultiRoom-N2-S4, MultiRoom-N4-S5
\end{itemize}
In each group, we arrange the tasks by increasing level of difficulty. Please refer to the MiniGrid website\footnote{\url{https://minigrid.farama.org/environments/minigrid/}} for detailed descriptions. 
Note that while RQL-AIS also evaluates the RedBlueDoors tasks, we have omitted them from our selection as they are MDPs. 

\paragraph{Implementation details on algorithms.} We adopt the RQL-AIS codebase\footnote{\url{https://github.com/esalehi1996/POMDP_RL}} for our implementation of a non-distributed version of R2D2~\citep{kapturowski2018recurrent}. We retain their exact implementation and hyperparameters for R2D2, which includes a recurrent replay buffer with uniform sampling, a $50$-step burn-in period, a $10$-step rollout, and a stepsize of $5$ for multi-step double Q-learning. The only difference is that we replace the periodic hard update of target networks with a soft update, to align with our EMA setting. 

We implement our end-to-end approaches based on R2D2. Minimalist $\phi_L$ introduces a single auxiliary task of \ZP; while $\phi_O$ adds a single auxiliary task of \OP. Both use deterministic $\ell_2$ loss. We normalize the loss coefficient $\lambda$ \autoref{eq:update} by the output dimension (\ie, $128$ for \ZP loss and $147$ for \OP loss) to balance with Q-learning scalar loss. We tune the normalized $\lambda$ between $(0.01, 0.03, 0.3, 1.0, 3.0, 10.0, 100.0)$ in SimpleCrossing and LavaCrossing tasks. We find that $0.01$ works best for \OP and $1.0$ best for \ZP. 

Furthermore, we also implement the phased approaches based on R2D2. Both $\phi_L$ (\RP + \ZP) and $\phi_O$ (\RP + \OP) freeze the encoder parameters during Q-learning. We introduce the coefficient $\alpha$ multiplied to \ZP or \OP loss to integrate with \RP loss. All three losses use deterministic $\ell_2$ objectives. We normalize $\alpha$ to balance reward scalar loss and tune it between $(0.01, 0.1, 0.3, 1.0, 3.0, 10.0)$ in SimpleCrossing and LavaCrossing tasks. We find $1.0$ works best for both \RP + \ZP  and \RP + \OP. 

We enumerate the values of our hyperparameters in \autoref{tab:hparams_pomdps}. If a hyperparameter is shared with R2D2 implemented by RQL-AIS, we maintain the same value as that used in RQL-AIS paper~\citep[Table~3]{seyedsalehi2023approximate}. Our network architecture exactly follows RQL-AIS (see their Appendix F).

Lastly, it is important to highlight the distinction between our implementation of \RP + \OP and the original RQL-AIS approach~\citep{seyedsalehi2023approximate}. Both approaches aim to learn $\phi_O$ in a phased manner with the same architecture. The main differences are:
\begin{itemize}[leftmargin=*]
    \item RQL-AIS employs a pre-trained autoencoder to compress the $147$-dimensional observations into $64$-dimensional latent representations. Then RQL-AIS trains their agent using latent representations while keeping the autoencoder parameters frozen. In contrast, our \RP + \OP implementation removes the autoencoder and instead directly predicts raw observations. 
    \item RQL-AIS uses MMD loss for observation prediction, which we show is equivalent to learning \EZP condition (see our discussion in \autoref{app:prior_works} on AIS). Thus, in \RP + \OP implementation, we replace the MMD loss with an $\ell_2$ loss.
    \item The loss coefficient $\alpha$ is set to $0.5$ in RQL-AIS while $0.1/147$ in our implementation.
    \item We use soft update on target Q network to align with our other implementations.
\end{itemize}
Despite these implementation differences, we find that our \RP + \OP implementation performs \textit{similarly} to RQL-AIS across the 20 tasks.

\begin{table}[t]
    \centering
    \caption{\textbf{Hyperparameters used in recurrent agents in MiniGrid.}}
    \vspace{0.5em}
    \begin{tabular}{cc}
    \toprule
       Hyperparameter  & Value  \\
       \midrule
        Discount factor ($\gamma$) & $0.99$ \\ 
        Number of environment steps & $4*10^6$ \\ 
        Target network update rate ($\tau$) & $0.005$ \\ 
        Replay buffer size & full \\ 
        Batch size & $256$ \\ 
        Learning rate & $0.001$ \\ 
        Latent state dimension & $128$ \\ 
        Epsilon greedy schedule & exponential$(1.0,0.05,400000)$\\
        R2D2 sequence length & $10$ \\ 
        R2D2 burn-in sequence length & $50$ \\
        $n$-step TD & $5$ \\
        Training frequency & every $10$ environment steps \\
        Auxiliary loss coefficient ($\lambda$) & $1.0/128$ for \ZP and $0.01/147$ for \OP \\
        Loss coefficient for phased training ($\alpha$) & $1.0/128$ for \RP + \ZP and $1.0/147$ for \RP + \OP \\ 
        \bottomrule
    \end{tabular}
    \label{tab:hparams_pomdps}
\end{table}

\subsection{Evaluation Metrics}
We evaluate the \textbf{episode return} by executing the deterministic version of the actor to compute the undiscounted sum of rewards. 

We estimate the \textbf{rank} of a batch of latent states by calling \texttt{torch.linalg.matrix\_rank(atol, rtol)} function in PyTorch~\citep{paszke2019pytorch}. This function calculates the number of singular values that are greater than $\max(\mathtt{atol}, \sigma_1 * \mathtt{rtol})$ where $\sigma_1$ is the largest singular value. In MDP experiments, the batch has a size of $(512, 50)$ with \texttt{atol=1e-2, rtol=1e-2}. In POMDP experiments, the batch has a size of $(256 * 10, 128)$ with \texttt{atol=1e-3, rtol=1e-3}, where we reshape the 3D tensor of $(256, 10, 128)$ size into 2D matrix. 

Each algorithm variant of the experiments is conducted across 12 individual runs in MDPs and 9 individual runs in POMDPs. 

We employ the Rliable library~\citep{agarwal2021deep} to compute the IQM and its CI for the aggregated curves (\autoref{fig:minigrid}). Essentially, IQM is the $25\%$ trimmed mean over the data on 20 tasks with 9 seeds, \ie, 180 runs.

\subsection{Computational Resources}

It requires around 1.5 days for us to train our algorithm in a (distracting) MuJoCo task for 1.5M environment steps with 3 runs executed in parallel. The 3 runs share a single A100 GPU and utilize 3 CPU cores. 

On the same machine, training cost (in secs) per update for Ant-v2 is as follows: model-free agents take around 0.032s, self-predictive and observation-predictive agents with $\ell_2$ objective take around 0.036s (13\% more), self-predictive and observation-predictive agents with KL objective take around 0.038s (19\% more), ALM(3) agent takes around 0.058s (81\% more). The brackets show the percentage increase compared to model-free agents.

It requires around 0.5 days for us to train our algorithm in a MiniGrid task for 4M environment steps with 3 runs executed in parallel. The 3 runs share a single V100 GPU and utilize 3 CPU cores. 

\section{Architecture and Code}
\label{sec:code}

\begin{figure}[h]
    \centering

\begin{minipage}[t]{0.7\linewidth}
\begin{tikzpicture}[>=latex',line join=bevel,very thick,scale=0.7]
\node (h) at (18.0bp,72.0bp) [draw,circle] {$h$};
  \node (phi) at (99.0bp,45.0bp) [draw,fill=yellow,trapezium,shape border rotate=270] {$f_\phi$};
  \node (z) at (180.0bp,95.0bp) [draw,circle] {$z$};
  \node (Q) at (261.0bp,95.0bp) [draw,fill=lightgray,rectangle] {$Q_\omega,\pi_\nu$};
  \node (P) at (261.0bp,41.0bp) [draw,fill=lightgray,rectangle] {$g_\theta$};
  \node (a) at (180.0bp,41.0bp) [draw,circle] {$a$};
  \node (r) at (342.0bp,95.0bp) [draw,circle] {$r$};
  \node (next_h) at (18.0bp,18.0bp) [draw,circle] {$h'$};
  \node (next_z) at (342.0bp,18.0bp) [draw,circle] {$z'$};
  \draw [->] (h) ..controls (43.105bp,63.632bp) and (52.841bp,60.386bp)  .. (phi);
  \draw [->] (z) ..controls (205.9bp,95.0bp) and (214.88bp,95.0bp)  .. (Q);
  \draw [->] (z) ..controls (203.77bp,79.154bp) and (214.8bp,71.803bp)  .. (P);
  \draw [->] (a) ..controls (203.77bp,56.846bp) and (214.8bp,64.197bp)  .. (Q);
  \draw [->] (a) ..controls (205.9bp,41.0bp) and (214.88bp,41.0bp)  .. (P);
  \draw [->] (next_h) ..controls (43.105bp,26.368bp) and (52.841bp,29.614bp)  .. (phi);
  \draw [->] (phi) ..controls (135.92bp,67.792bp) and (146.5bp,74.321bp)  .. (z);
  \draw [->,dashed] (phi) ..controls (136.94bp,22.528bp) and (149.65bp,16.873bp)  .. (162.0bp,14.0bp) .. controls (214.56bp,1.769bp) and (278.09bp,7.8538bp)  .. (next_z);
  \draw [] (Q) ..controls (299.99bp,95.0bp) and (313.15bp,95.0bp)  .. (r);
  \draw [] (P) ..controls (300.36bp,29.825bp) and (313.99bp,25.955bp)  .. (next_z);
\end{tikzpicture}
\end{minipage}
    \caption{\textbf{Architecture of our minimalist $\phi_L$ algorithm.} The dashed edge indicates the stop-gradient operator; the undirected edges indicate learning from grounded signals of rewards or next latent states. }
    \label{fig:arch_end_to_end}
\end{figure}
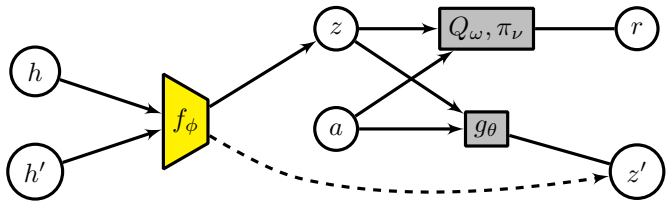

Besides Algo.~\ref{pseudo-code}, we provide a pseudocode written in PyTorch syntax~\citep{paszke2019pytorch} in Algo.~\ref{pytorch-code}. \autoref{fig:arch_end_to_end} shows our architecture.  We also have open-sourced our code at \url{https://github.com/twni2016/self-predictive-rl}.

\begin{algorithm}[H]
    \caption{Our loss function for learning self-predictive representation in PyTorch syntax}
\label{pytorch-code}
\footnotesize
\begin{lstlisting}
def loss(hist, act, next_obs, rew):
    # hist:(B,T,O+A), act:(B,A), next_obs:(B,O), rew:(B,1)
    from torch import cat; from copy import deepcopy

    # Encode histories into latent states
    h_enc = Encoder(hist) #(B,Z)
    next_hist = cat([hist, cat([act, next_obs], dim=-1)], dim=1) #(B,T+1,O+A)
    next_h_enc_tar = Encoder_Target(next_hist) #(B,Z)

    # Compute RL loss
    td_tar = rew + gamma * Critic_Target(next_h_enc_tar, Actor(next_h_enc_tar))
    critic_loss = ((Critic(h_enc, act)- td_tar.detach())**2).mean()
    actor_loss = -deepcopy(Critic)(h_enc.detach(), Actor(h_enc)).mean()

    # Compute ZP loss
    zp_loss = ((Latent_Model(h_enc, act) - next_h_enc_tar)**2).sum(-1).mean()
    
    return critic_loss + actor_loss + zp_coef * zp_loss
\end{lstlisting}
\end{algorithm}

\section{Additional Empirical Results}
\label{sec:ablation}

\paragraph{Ablation studies comparing our minimalist $\phi_L$ to ALM variants.}  \autoref{fig:ablation_ALM} shows the comparison between our method and several ALM variants (introduced in \autoref{sec:mdp_details}). Both ALM(3) and ALM-no-model have similar encoder objectives ($1$-step versus $3$-step) but differ in actor objective (model-free TD3 versus model-based SVG); their performance is similar except for the Humanoid-v2 tasks, suggesting that the model-based actor update is most useful on higher-dimensional tasks. 
Comparing ALM-no-model and ALM(0), which only differ in encoder objective ($1$-step versus $0$-step), we see that ALM(0) performs notably worse. 
This suggests that the use of stop-gradients in $Q$-value and the omission of mean-squared TD-error in \autoref{eq:alm-no-reward}~and~\autoref{eq:alm-no-svg} might be problematic, 
although this issue can be considerably mitigated by the $1$-step variant. 
Finally, our minimalist $\phi_L$ \autoref{eq:our_rkl} performs comparably to ALM-no-model on most tasks and significantly outperforms ALM(0). 
These results suggest that our method can achieve the benefits of a $1$-step rollout without having to unroll a model; however, a method that uses a $3$-step rollout can sometimes achieve better results.

\paragraph{Additional results on \ZP targets in standard MuJoCo.}
\autoref{fig:additional_kl_rank} shows the performance and the estimated representational rank for the \ZP KL divergences (FKL and RKL). Similar to findings in $\ell_2$ objective (\autoref{fig:mujoco_optim}), we notice significantly lower returns when removing the stop-gradient version (Detached and EMA). 
Surprisingly, this decreased performance does not seem to be caused by dimensional collapse; on most tasks, the online version of the KL objective does not suffer from dimensional collapse observed for the $\ell_2$ objective. These findings suggest that our estimated representational rank may not be correlated with expected returns.

\paragraph{Ablation studies on \ZP loss in standard MuJoCo.} \autoref{fig:ablation_zp_loss} shows the \ZP losses ($\ell_2$ loss, FKL loss, RKL loss) for each \ZP objective within our minimalist $\phi_L$ algorithm. We include results for the online, detached, and EMA targets. 
As expected, online \ZP targets directly minimize \ZP losses, thus reaching much lower \ZP loss values. However, a lower \ZP loss value does not imply higher returns, since the agent needs to balance the RL loss and \ZP loss.  In future work, we aim to explore strategies to effectively decrease \ZP loss without compromising the performance of the stop-gradient variant.

\paragraph{Failed experiments in standard MuJoCo.}
We did not explore the architecture design and did little hyperparameter tuning on our algorithm. 
Nevertheless, we observed two failure cases. 
To match the assumption of \autoref{thm:collapse} that the gradient \wrt the latent transition parameters $\theta$ reaches zero, we experimented with higher learning rates ($0.1$, $0.01$, $0.001$) for updating the parameters $\theta$ in MuJoCo tasks. Yet, we did not observe any performance increase compared to the default learning rate ($0.0001$). 
Secondly, inspired by the findings in \autoref{fig:ablation_zp_loss}, we tried a constrained optimization on auxiliary task to adaptively update the loss coefficient using gradient descent-ascent for the stop-gradient version. However, this resulted in a significant performance decline without an explicit decrease in \ZP loss values. 

\paragraph{Full per-task curves in distracting MuJoCo.} \autoref{fig:full_curves_distraction} shows all learning curves in distracting MuJoCo. 

\paragraph{Full per-task curves in MiniGrid.} 
\autoref{fig:minigrid_full} shows all learning curves in MiniGrid tasks. 
Minimalist $\phi_L$ (\ZP) is better than model-free RL (R2D2) in 8 of 20 tasks, and similar in the others except for a single task. On the other hand, $\phi_O$ (\OP) is better than model-free RL (R2D2) in 10 tasks, with the other tasks being identical. Since MiniGrid tasks are deterministic without distraction and the observation is not high-dimensional, $\phi_O$ (\OP) outperforming $\phi_L$ (\ZP) in 7 tasks is expected.
The end-to-end $\phi_L$ (\ZP) surpasses its phased counterpart (\RP + \ZP) in 14 tasks, with the rest tasks being the same. The end-to-end $\phi_O$ (\OP) is better than its phased counterpart (\RP + \OP) in 7 tasks, but falls short in 4 tasks. These findings underline the efficacy of the end-to-end approach to learning $\phi_L$ over the phased approach.

\autoref{fig:ZP_optim} shows all matrix rank curves in MiniGrid tasks. 
Across all 20 tasks, online \ZP targets consistently have the \textit{lowest} matrix rank, aligned with our prediction from \autoref{thm:collapse}. However, while \autoref{thm:collapse} shows that both detached and EMA targets avoid collapse in \textit{linear} setting, we observe that detached targets severely collapse in 3 tasks, a phenomenon absent with EMA targets. This prompts further theoretical investigation in a \textit{nonlinear} context. As expected, \OP consistently achieves the highest rank compared to \ZP and R2D2, since \OP learns the finest abstraction.

\begin{figure}[h]
    \centering
    \includegraphics[width=0.32\linewidth]{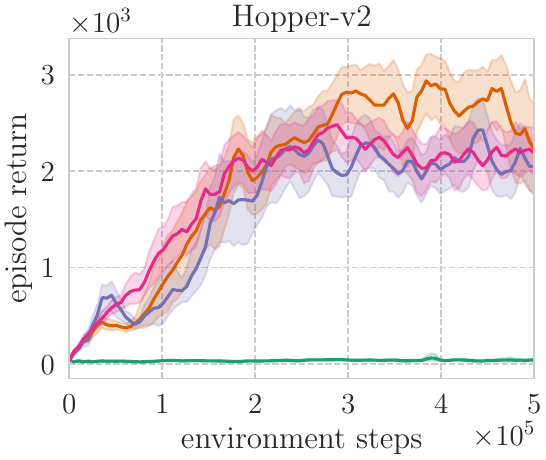}
    \includegraphics[width=0.32\linewidth]{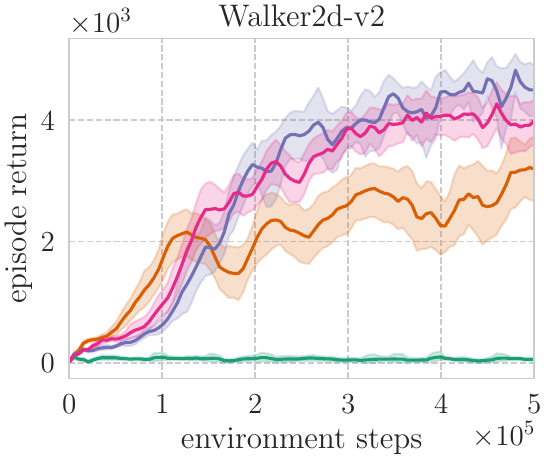}
    \includegraphics[width=0.33\linewidth]{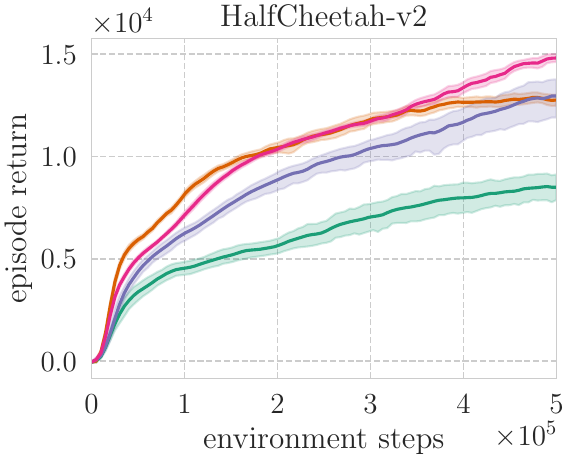}
    \includegraphics[width=0.32\linewidth]{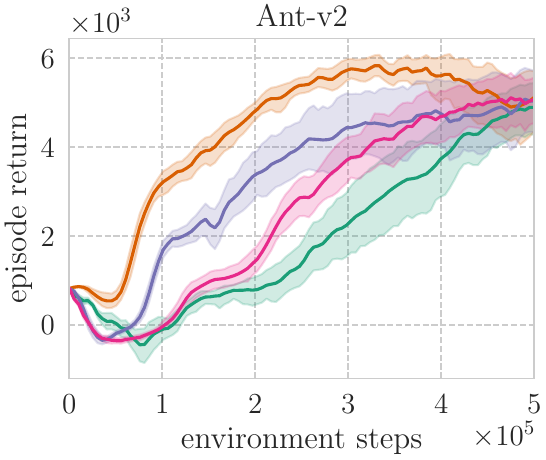}
    \includegraphics[width=0.32\linewidth]{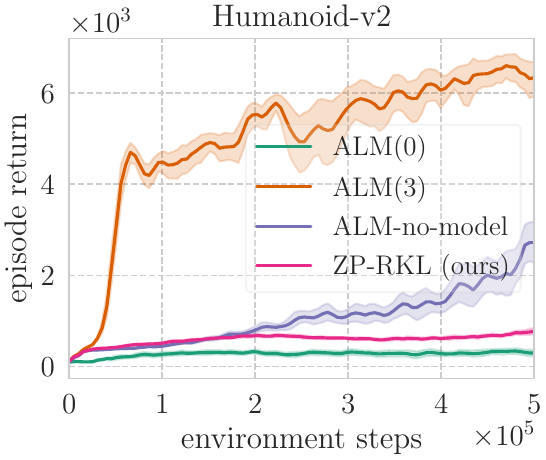}
    \caption{Ablation studies on ALM variants (ALM(3), ALM-no-model, ALM(0)) and our minimalist $\phi_L$ (\ZP-RKL with EMA targets).}
    \label{fig:ablation_ALM}
\end{figure}

\begin{figure}[h]
    \centering
    \textsc{Forward KL} \\
    \includegraphics[width=0.24\linewidth]{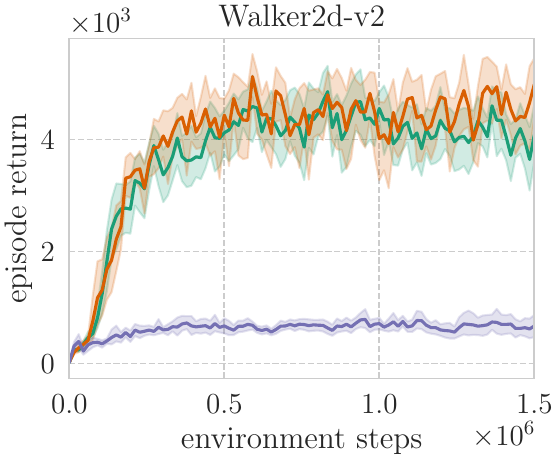}
    \includegraphics[width=0.24\linewidth]{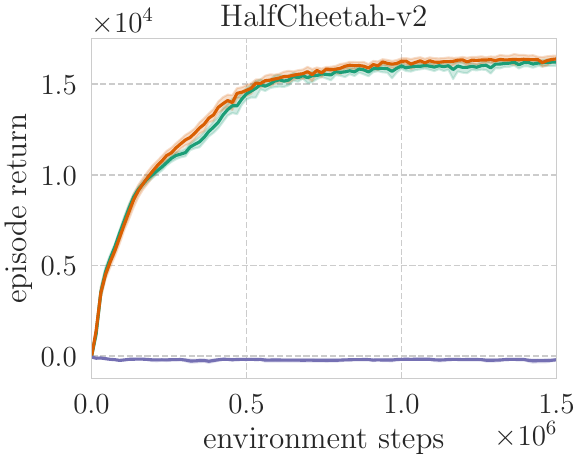}
    \includegraphics[width=0.24\linewidth]{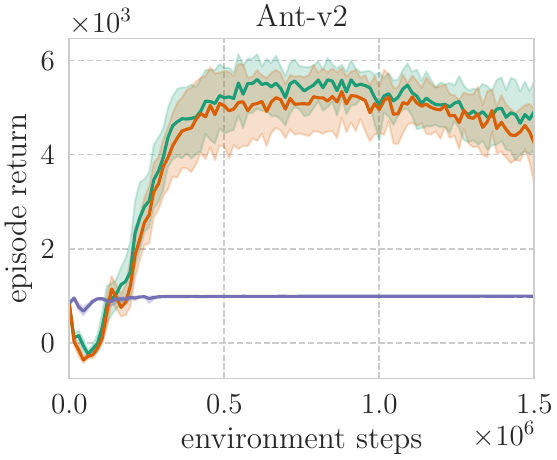}
    \includegraphics[width=0.24\linewidth]{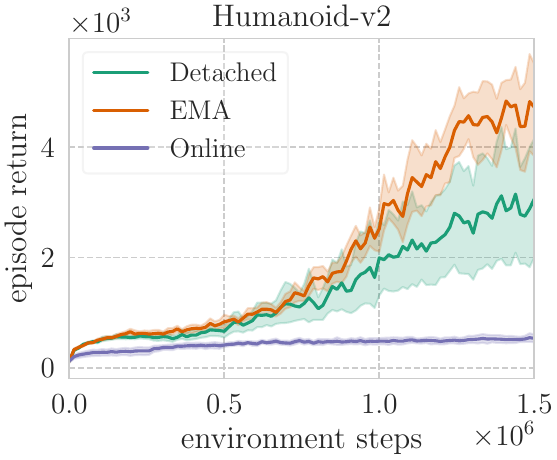}
    
\includegraphics[width=0.24\linewidth]{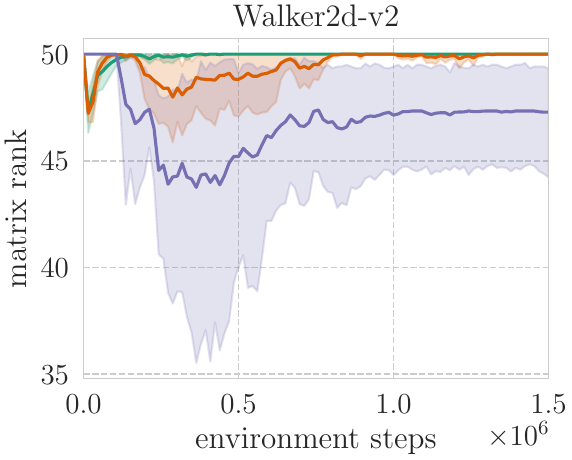}
    \includegraphics[width=0.24\linewidth]{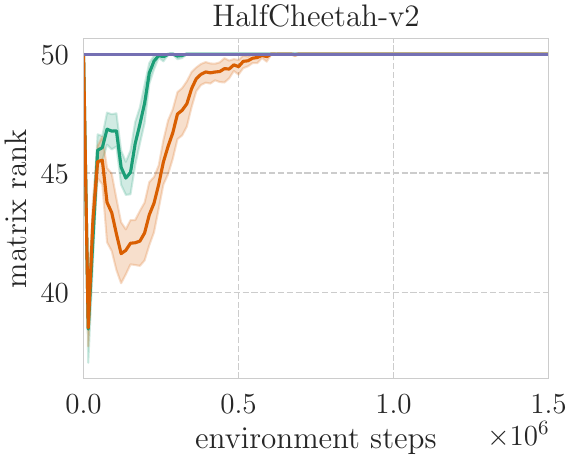}
    \includegraphics[width=0.24\linewidth]{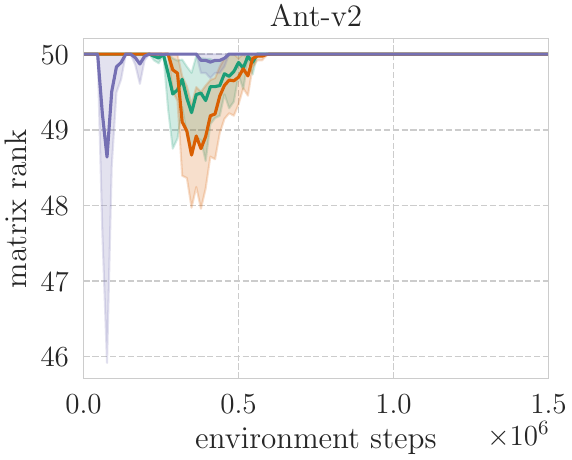}
    \includegraphics[width=0.24\linewidth]{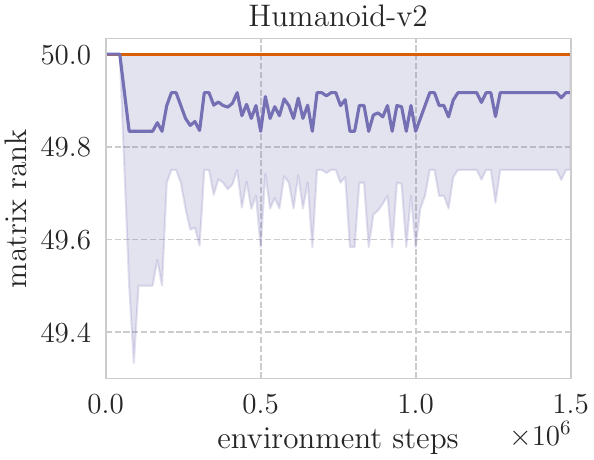}
    
    \textsc{Reverse KL} \\
    \includegraphics[width=0.24\linewidth]{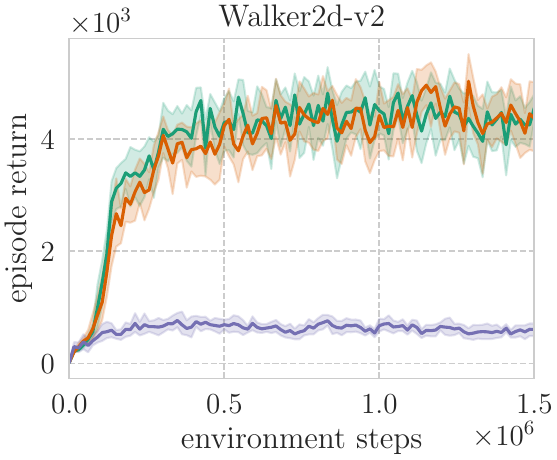}
    \includegraphics[width=0.24\linewidth]{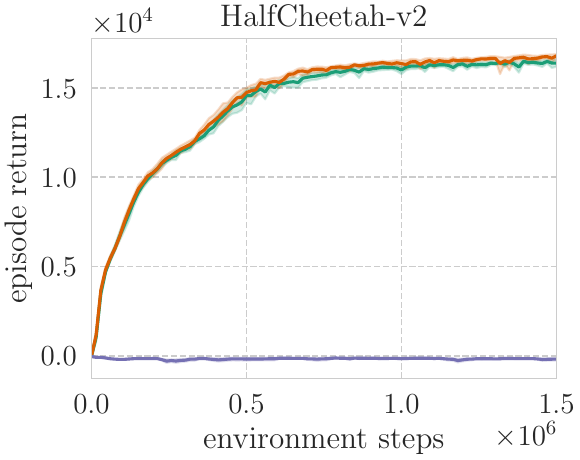}
    \includegraphics[width=0.24\linewidth]{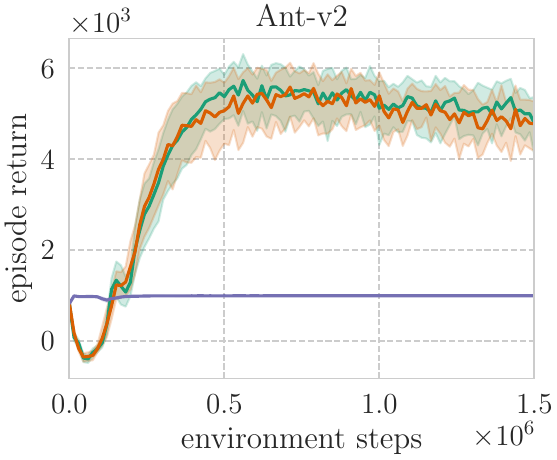}
    \includegraphics[width=0.24\linewidth]{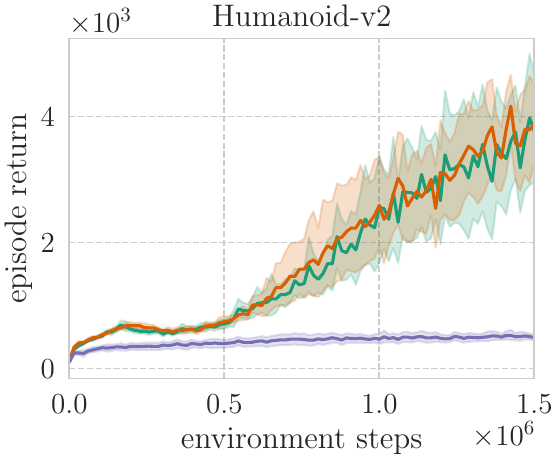}
    
\includegraphics[width=0.24\linewidth]{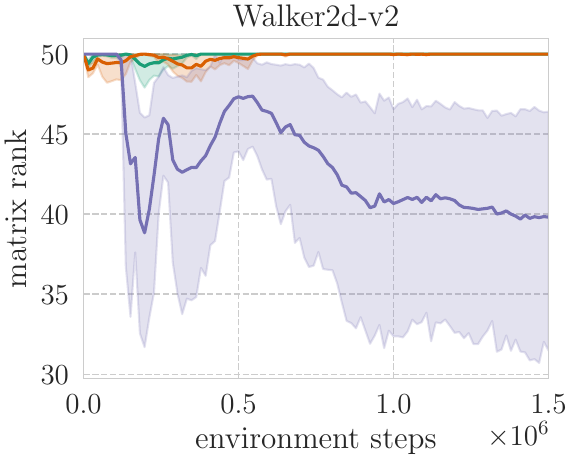}
    \includegraphics[width=0.24\linewidth]{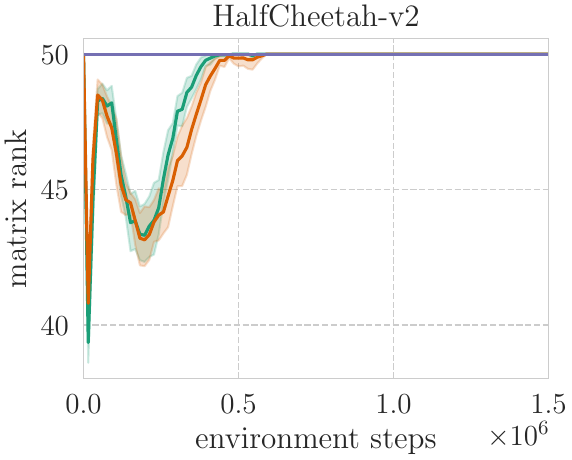}
    \includegraphics[width=0.24\linewidth]{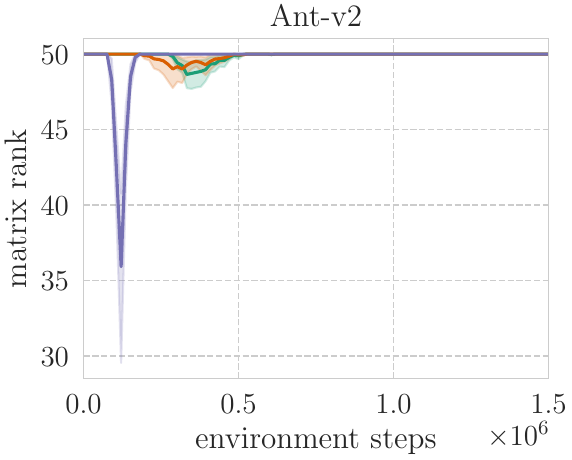}
    \includegraphics[width=0.24\linewidth]{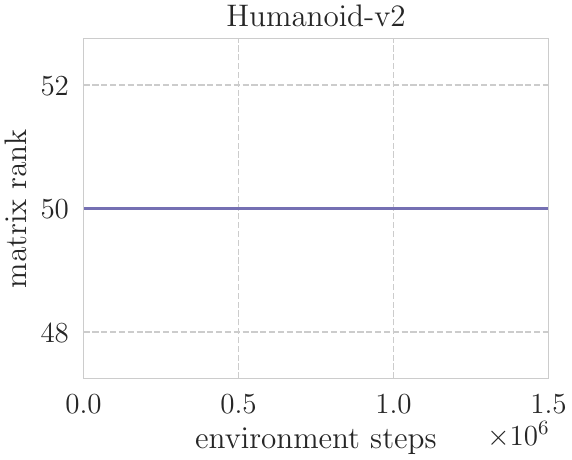}
    
    \caption{\textbf{Representation collapse is less severe with online targets for KL objectives.} On four benchmark tasks, we observe that using the online \ZP target results in lower returns. Rows 1 and 2 show results for the forward KL; rows 3 and 4 show result for the reverse KL.
    }
    \label{fig:additional_kl_rank}
\end{figure}

\begin{figure}[h]
    \centering
    \includegraphics[width=0.24\linewidth]{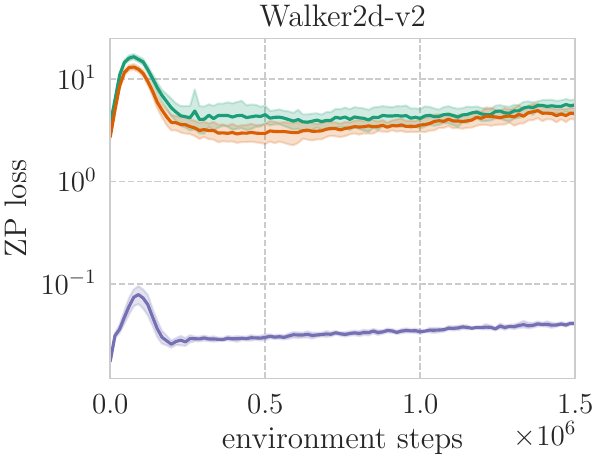}
    \includegraphics[width=0.24\linewidth]{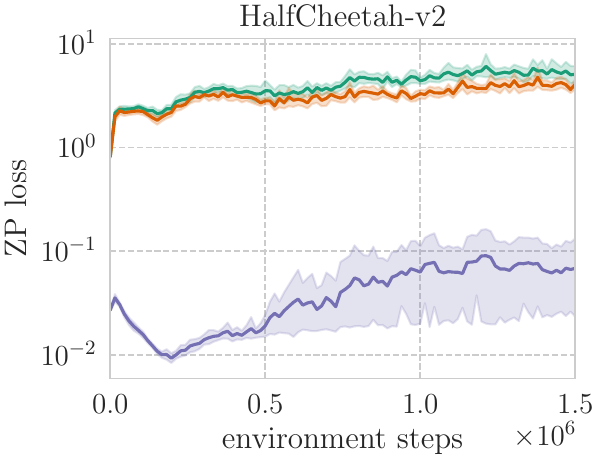}
    \includegraphics[width=0.24\linewidth]{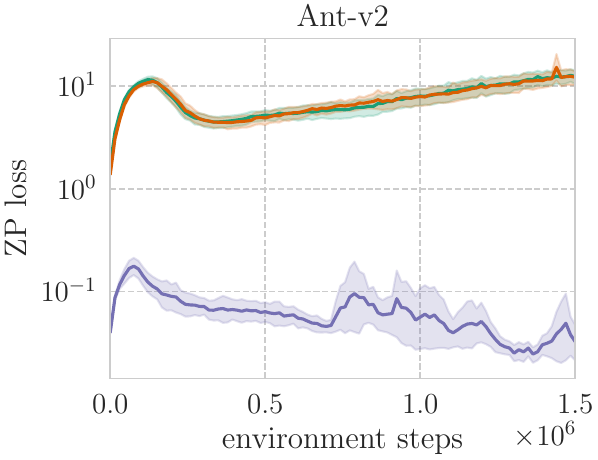}
    \includegraphics[width=0.24\linewidth]{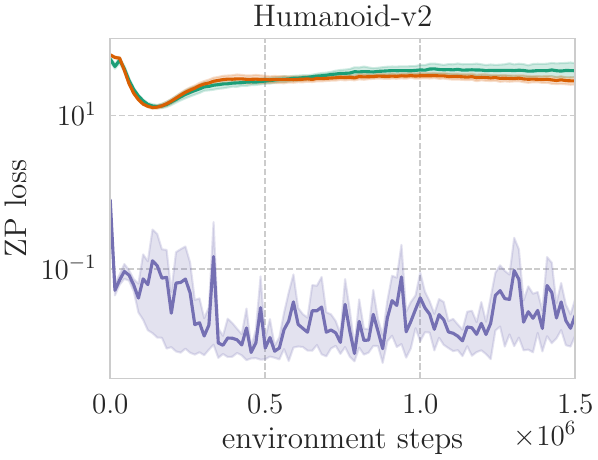}
    \includegraphics[width=0.24\linewidth]{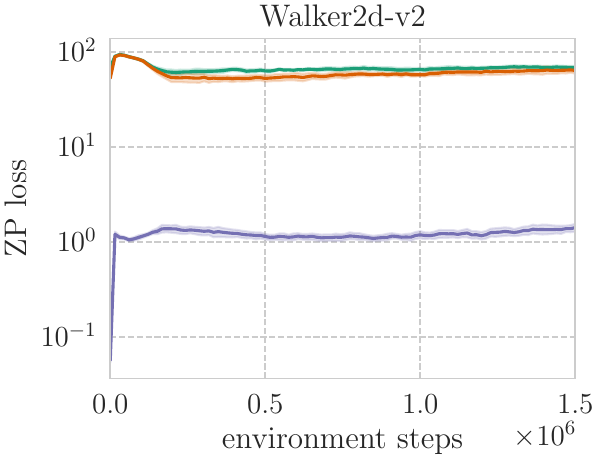}
    \includegraphics[width=0.24\linewidth]{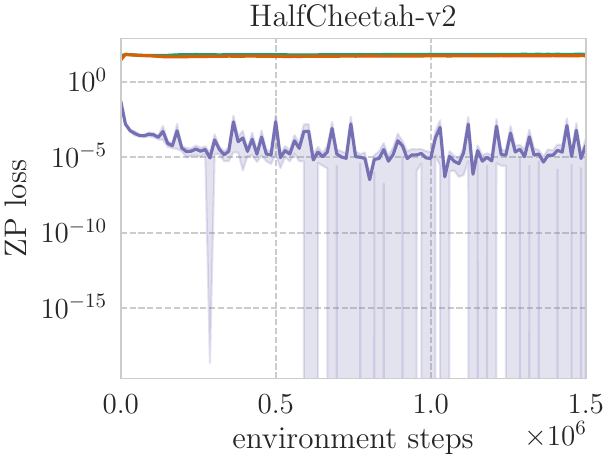}
    \includegraphics[width=0.24\linewidth]{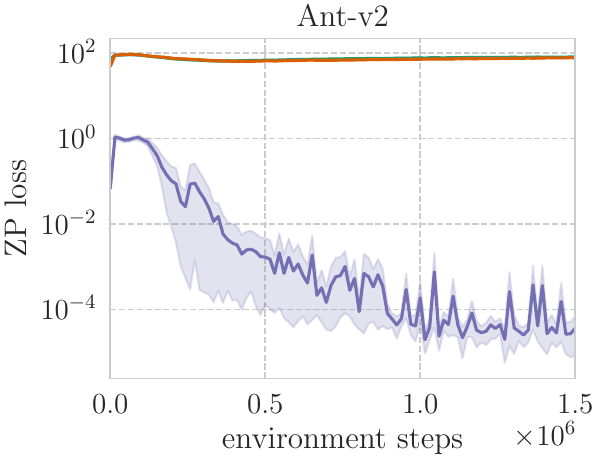}
    \includegraphics[width=0.24\linewidth]{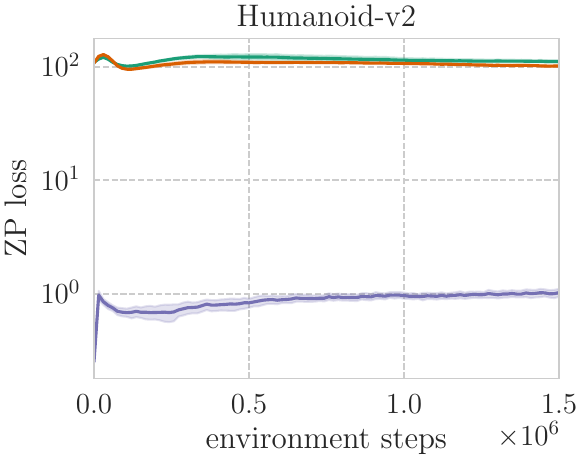}
    \includegraphics[width=0.24\linewidth]{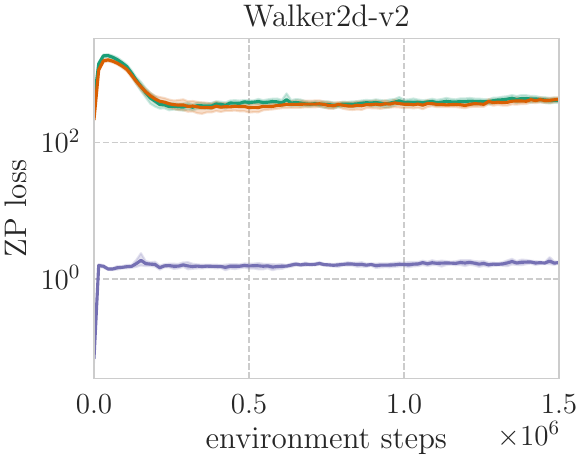}
    \includegraphics[width=0.24\linewidth]{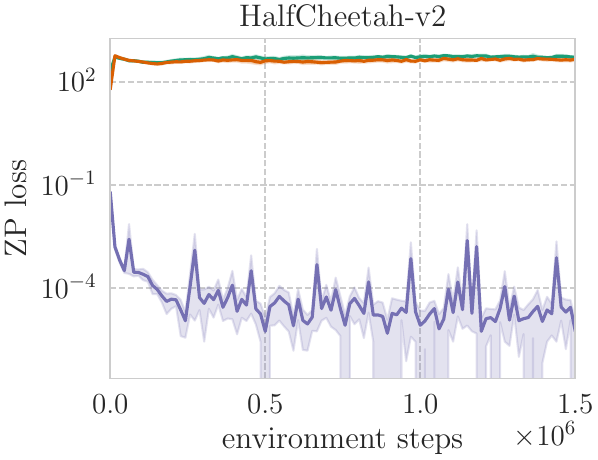}
    \includegraphics[width=0.24\linewidth]{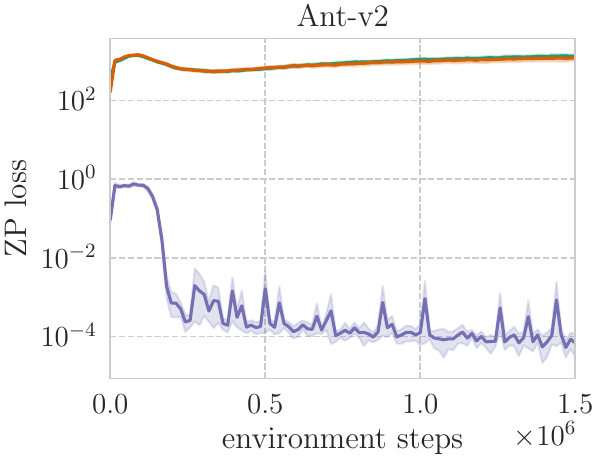}
    \includegraphics[width=0.24\linewidth]{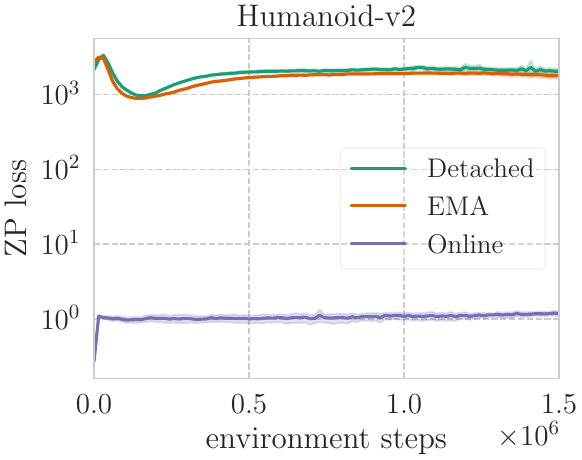}
    \caption{\textbf{Online \ZP targets reach much smaller values on \ZP objectives than stop-gradient \ZP targets in standard MuJoCo.} Top row: \ZP with $\ell_2$ objective; Middle row: \ZP with FKL objective; Bottom row: \ZP with RKL objective.}
    \label{fig:ablation_zp_loss}
\end{figure}

\begin{figure}[h]
    \centering
    \includegraphics[width=0.19\linewidth]{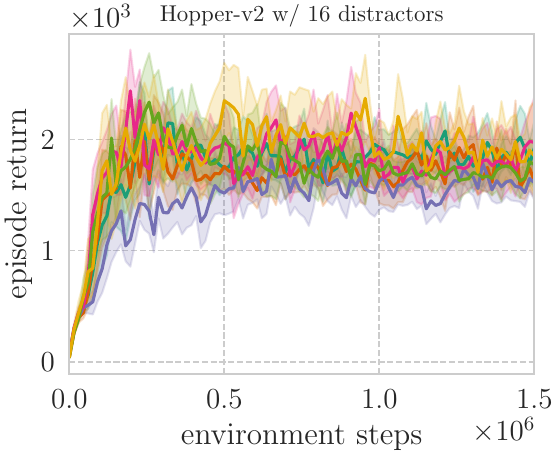}
    \includegraphics[width=0.19\linewidth]{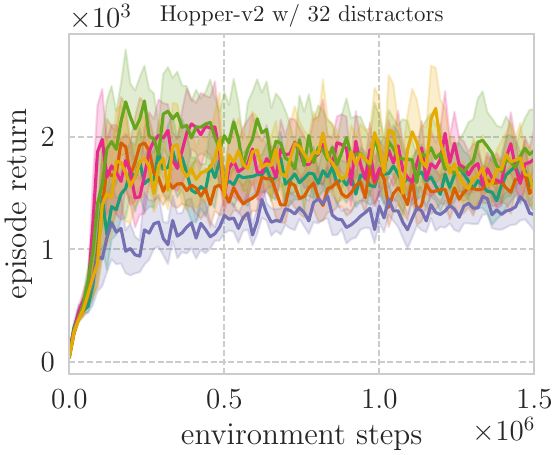}
    \includegraphics[width=0.19\linewidth]{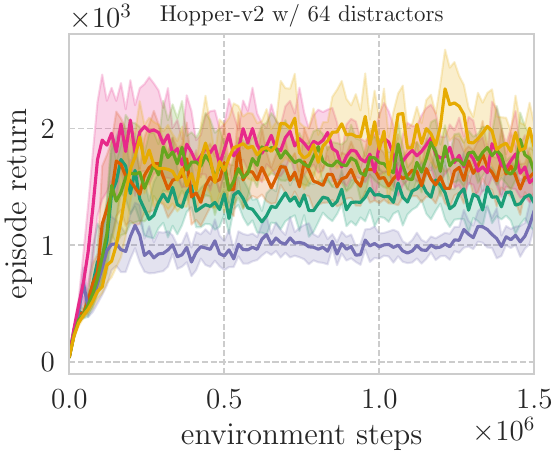}
    \includegraphics[width=0.19\linewidth]{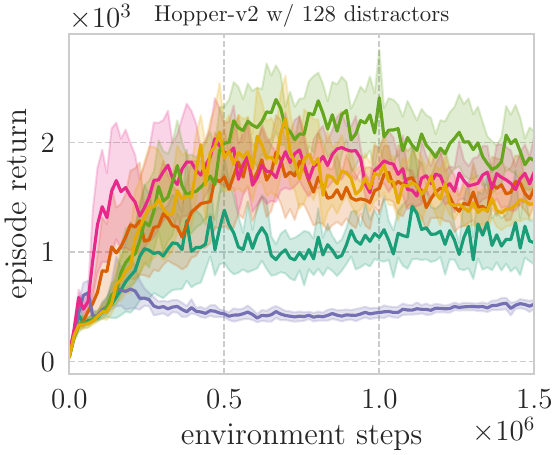}
    \includegraphics[width=0.19\linewidth]{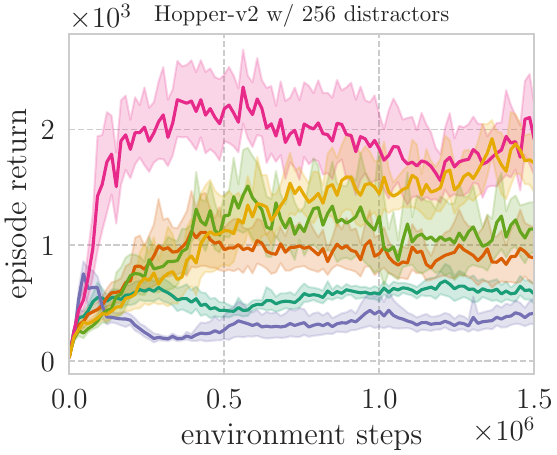}

    \includegraphics[width=0.19\linewidth]{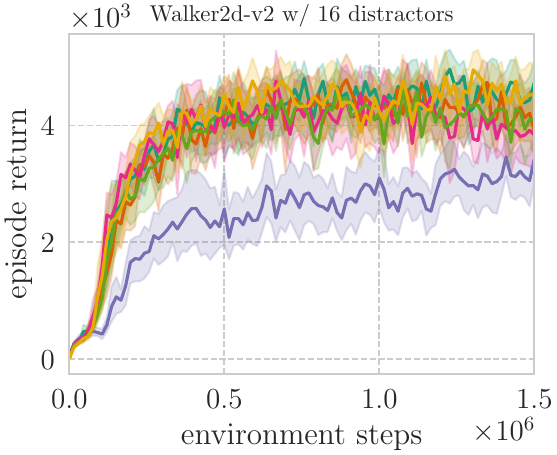}
    \includegraphics[width=0.19\linewidth]{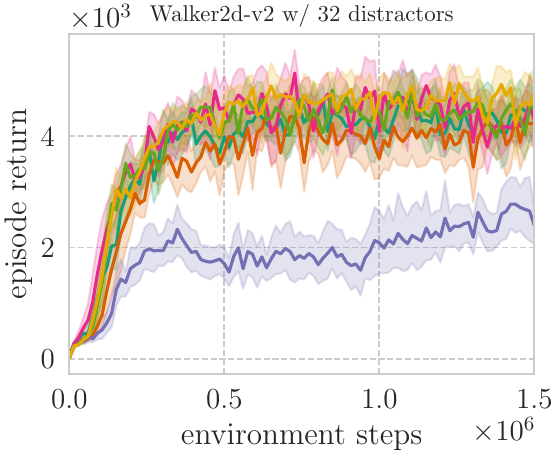}
    \includegraphics[width=0.19\linewidth]{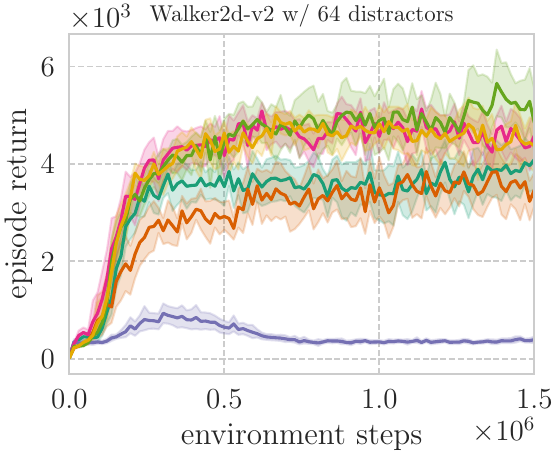}
    \includegraphics[width=0.19\linewidth]{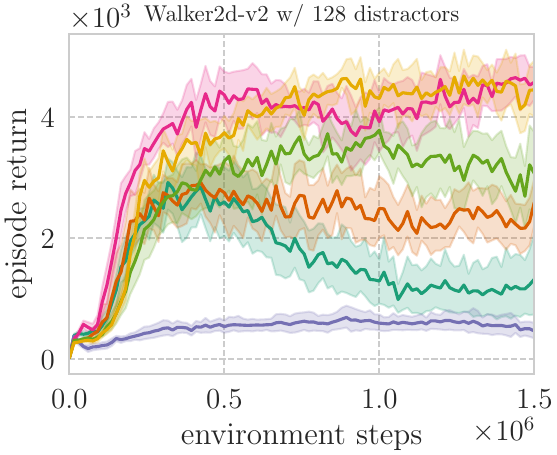}
    \includegraphics[width=0.19\linewidth]{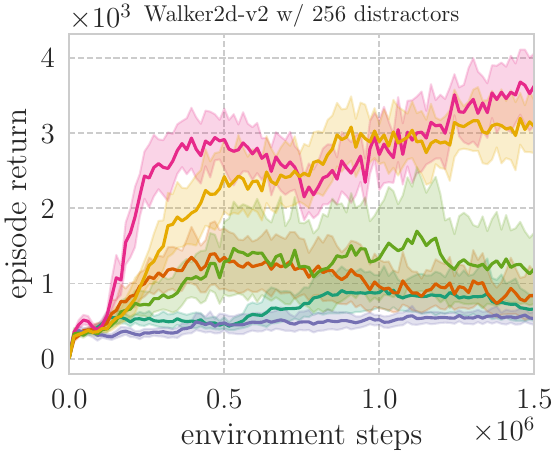}

    \includegraphics[width=0.19\linewidth]{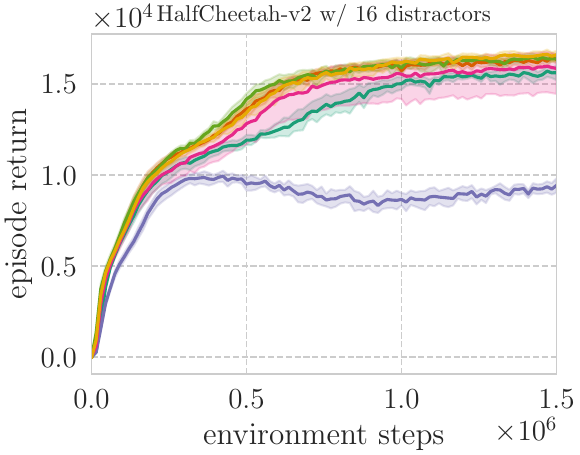}
    \includegraphics[width=0.19\linewidth]{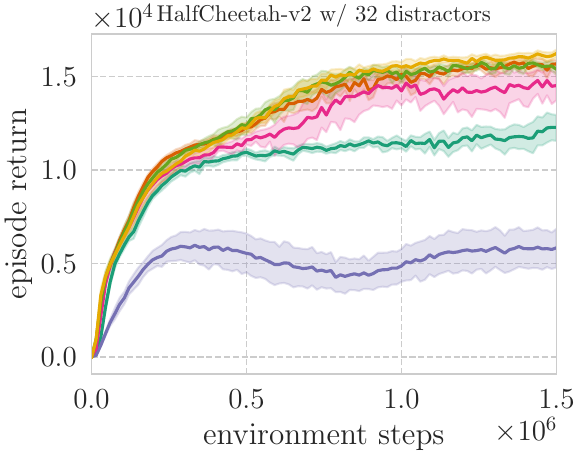}
    \includegraphics[width=0.19\linewidth]{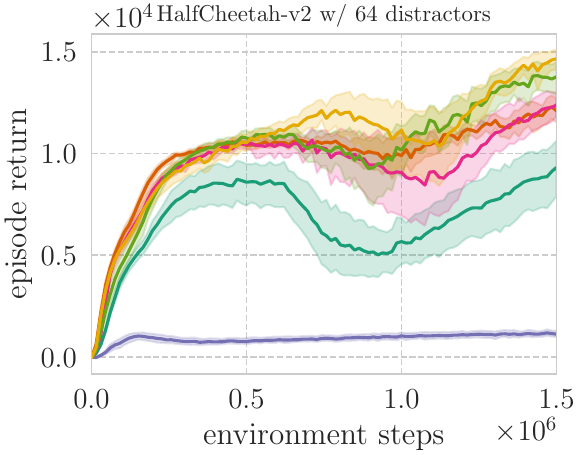}
    \includegraphics[width=0.19\linewidth]{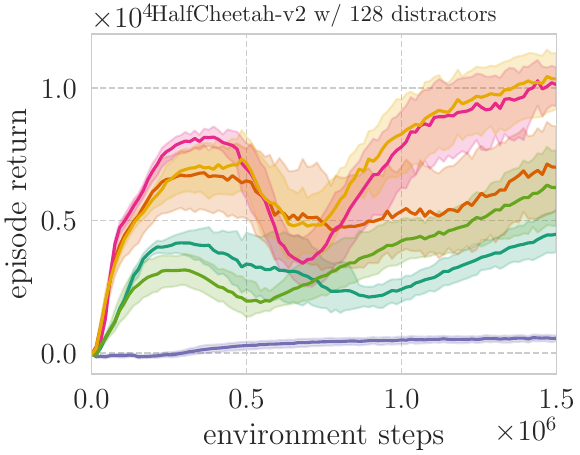}
    \includegraphics[width=0.19\linewidth]{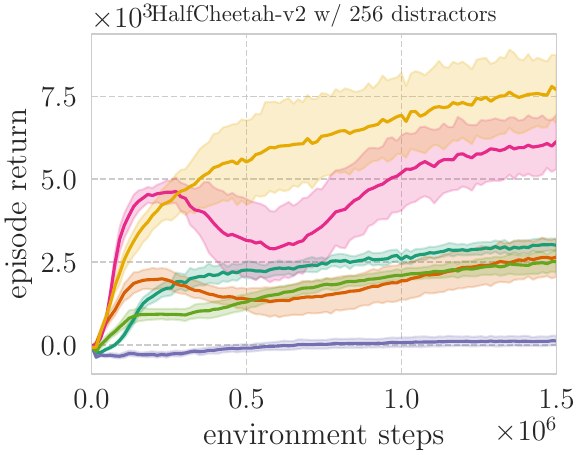}
    
    \includegraphics[width=0.19\linewidth]{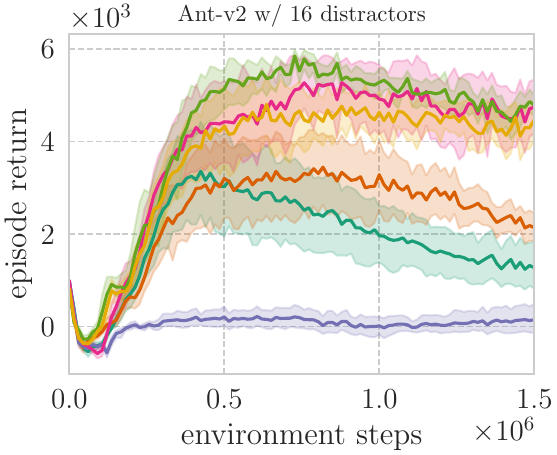}
    \includegraphics[width=0.19\linewidth]{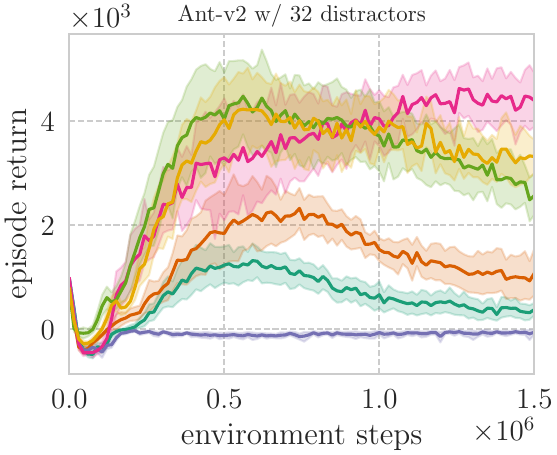}
    \includegraphics[width=0.19\linewidth]{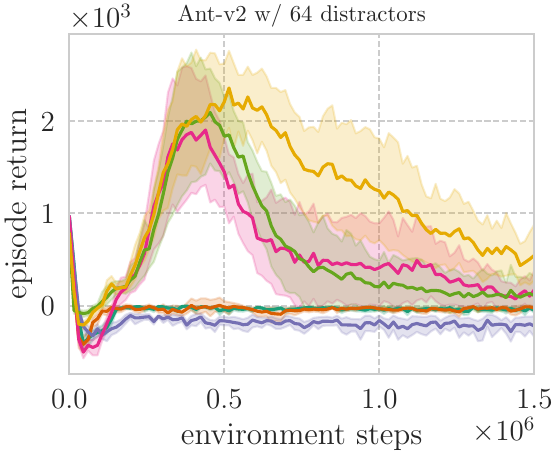}
    \includegraphics[width=0.20\linewidth]{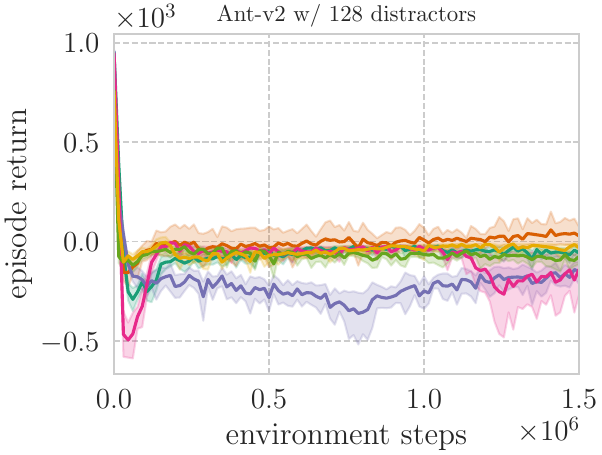}
    \includegraphics[width=0.20\linewidth]{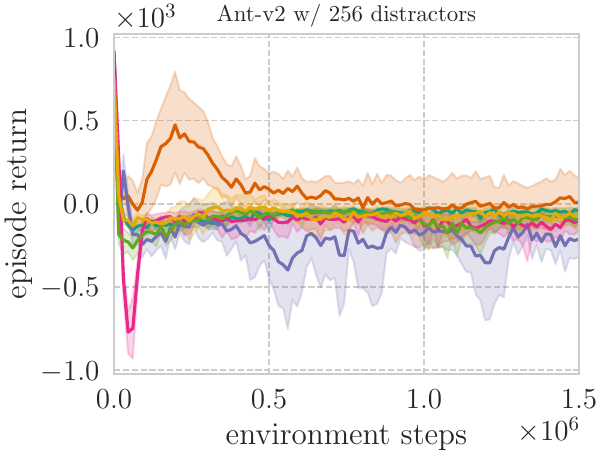}
    
    \includegraphics[width=\linewidth]{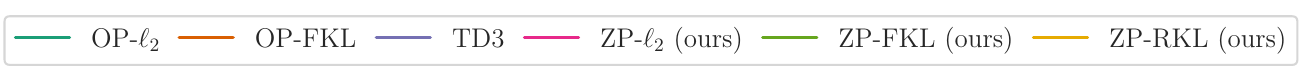}
    \caption{Full learning curves on distracting MuJoCo benchmark.}
    \label{fig:full_curves_distraction}
\end{figure}

\newcommand\ratio{0.19}
\newcommand\ratiorank{0.19}

\begin{figure}
    \centering
    \includegraphics[width=\ratio\linewidth]{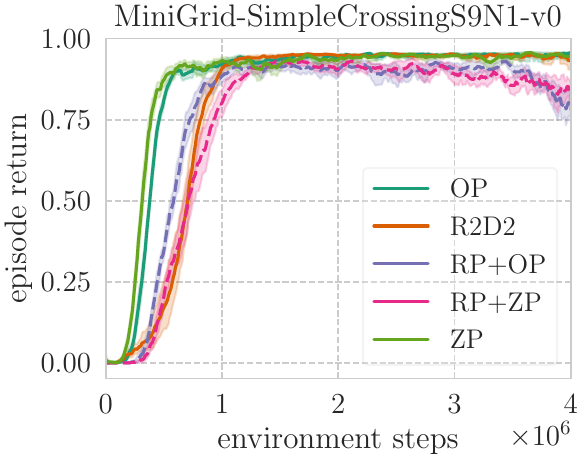}
    \includegraphics[width=\ratio\linewidth]{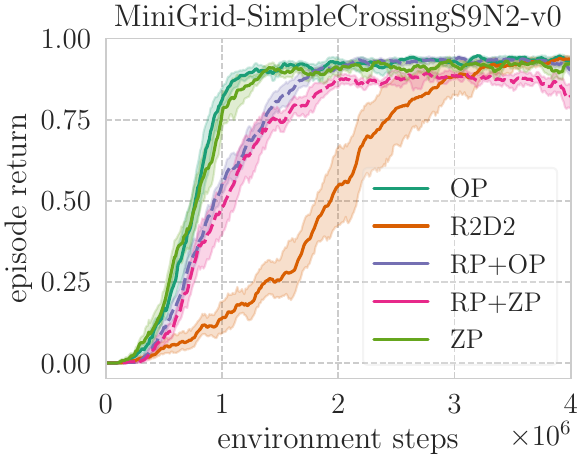}
    \includegraphics[width=\ratio\linewidth]{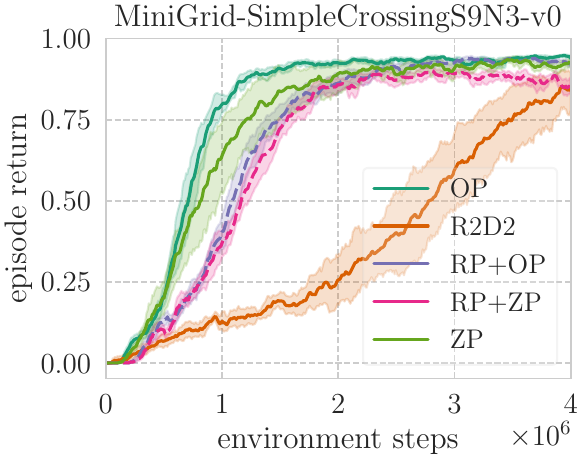}
    \includegraphics[width=\ratio\linewidth]{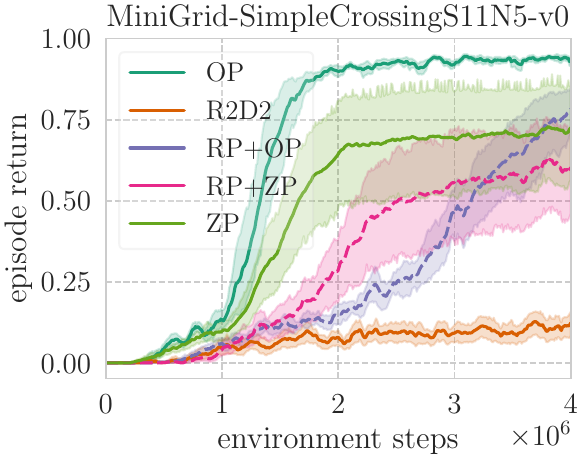}
    \includegraphics[width=\ratio\linewidth]{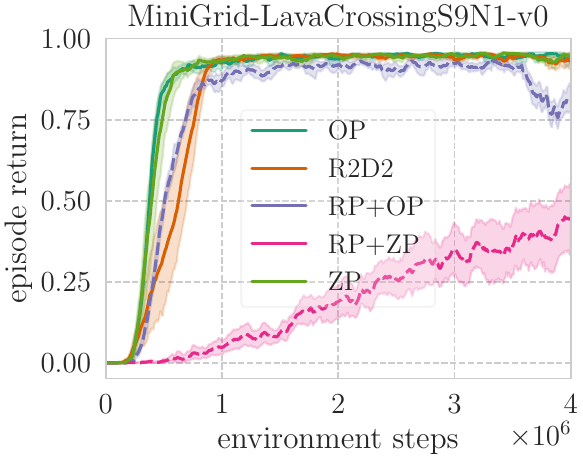}
    \includegraphics[width=\ratio\linewidth]{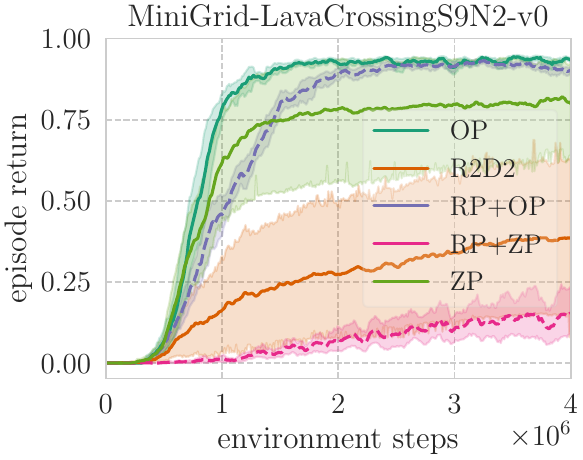}
    \includegraphics[width=\ratio\linewidth]{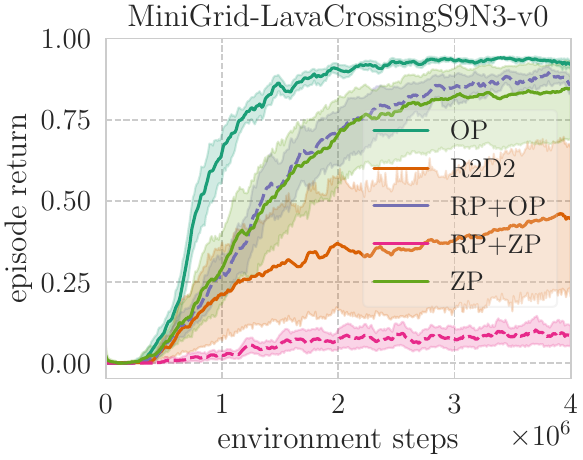}
    \includegraphics[width=\ratio\linewidth]{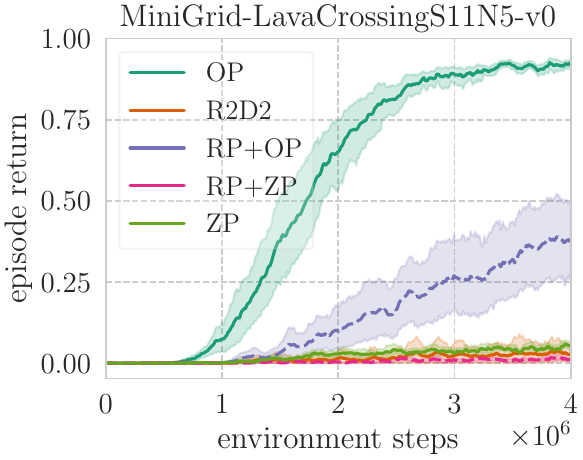}
    \includegraphics[width=\ratio\linewidth]{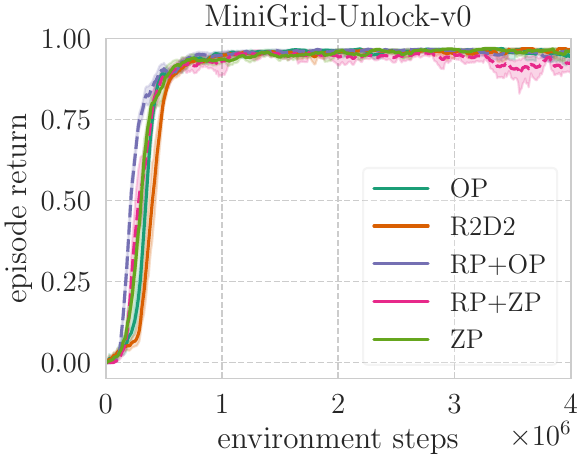}
    \includegraphics[width=\ratio\linewidth]{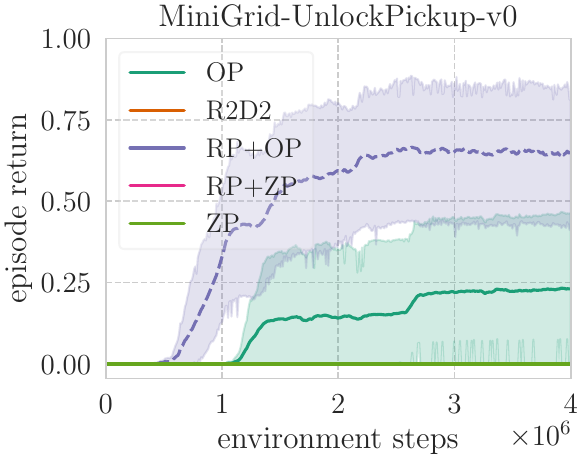}
    \includegraphics[width=\ratio\linewidth]{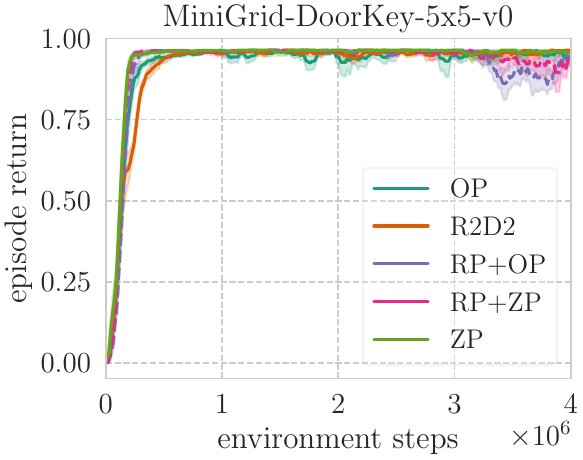}
    \includegraphics[width=\ratio\linewidth]{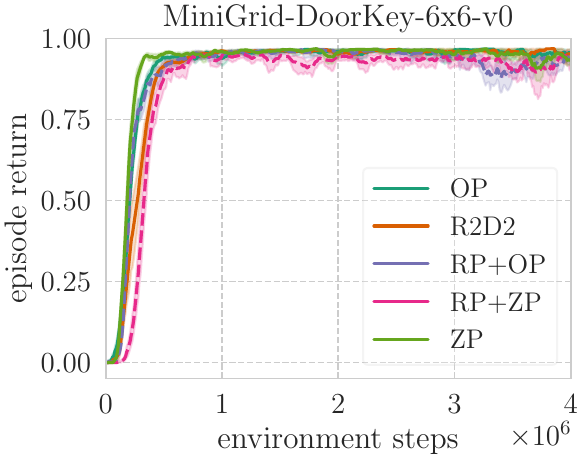}
    \includegraphics[width=\ratio\linewidth]{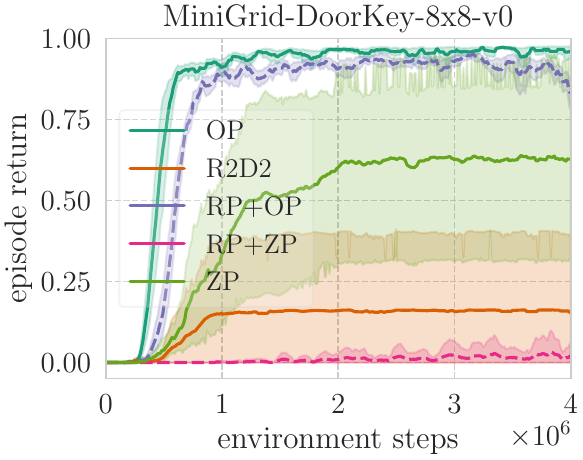}
    \includegraphics[width=\ratio\linewidth]{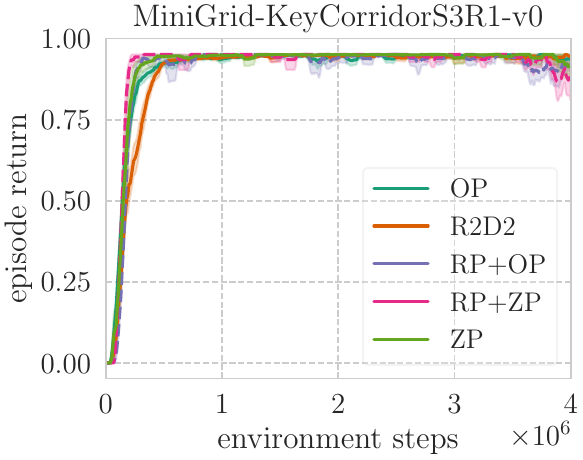}
    \includegraphics[width=\ratio\linewidth]{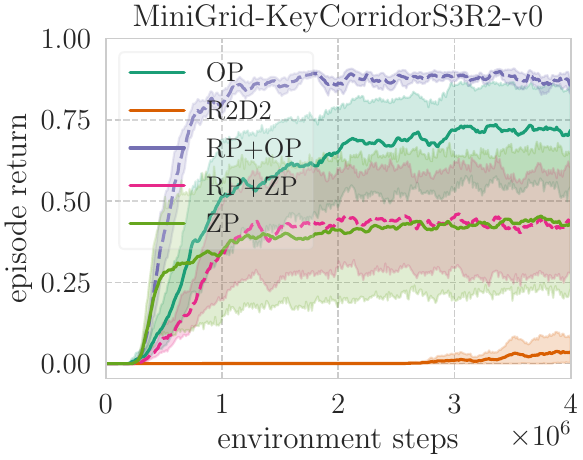}
    \includegraphics[width=\ratio\linewidth]{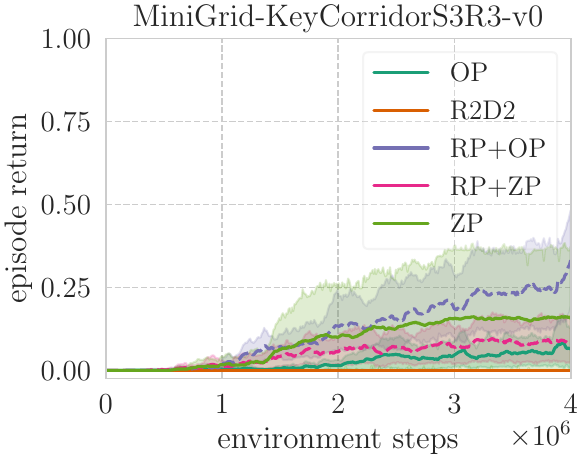}
    \includegraphics[width=\ratio\linewidth]{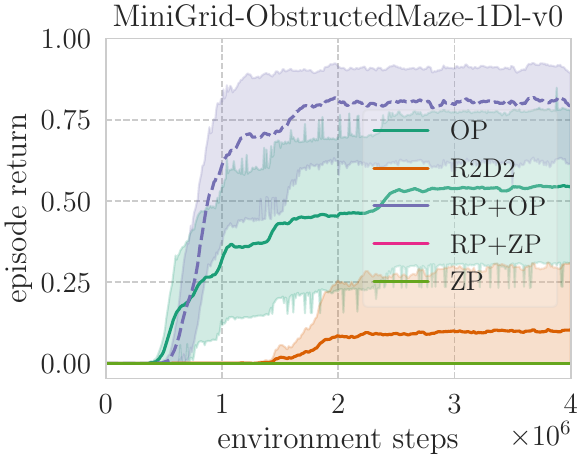}
    \includegraphics[width=\ratio\linewidth]{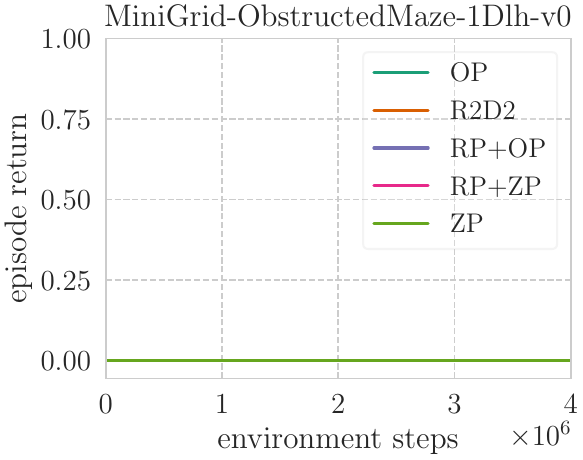}
    \includegraphics[width=\ratio\linewidth]{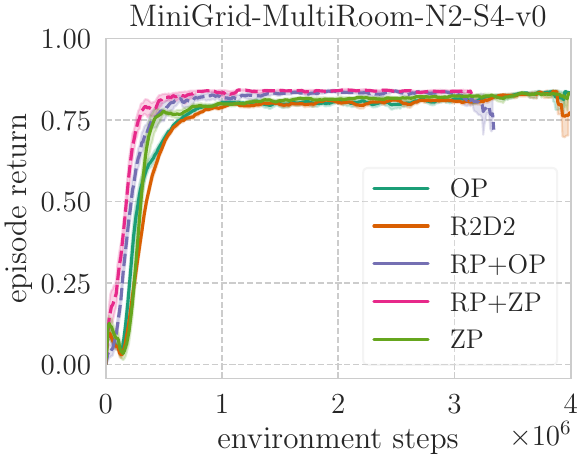}
    \includegraphics[width=\ratio\linewidth]{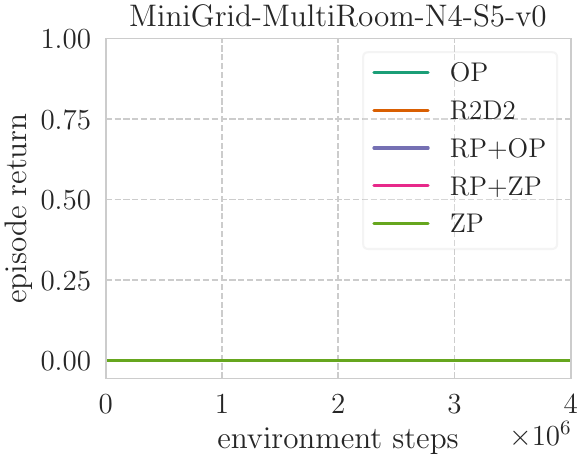}
    \caption{The \textbf{episode return} between $\phi_{Q^*}$ (R2D2), $\phi_L$ (\ZP, \RP + \ZP), $\phi_O$ (\OP, \RP + \OP) in 20 MiniGrid tasks over 4M steps, averaged across $\ge$ 9 seeds. The end-to-end approaches (R2D2, \ZP, \OP) are shown by \textbf{solid} curves, while the phased ones (\RP + \ZP, \RP + \OP) are shown by \textbf{dashed} curves. The \ZP targets use EMA. Tasks sharing the same prefixes (e.g., SimpleCrossing, KeyCorridor) are arranged in order of increasing difficulty.}
    \label{fig:minigrid_full}
\end{figure}

\begin{figure}
    \centering
    \includegraphics[width=\ratiorank\linewidth]{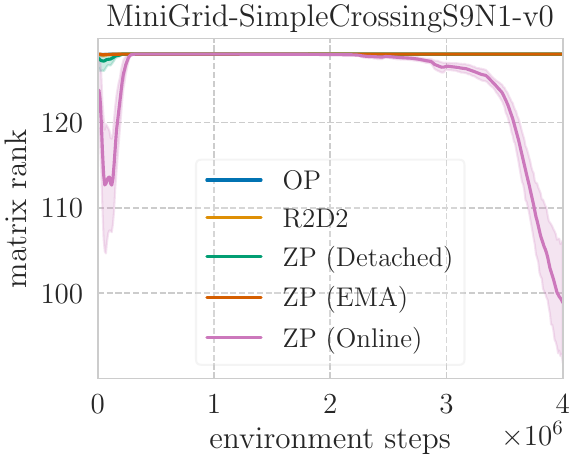}
    \includegraphics[width=\ratiorank\linewidth]{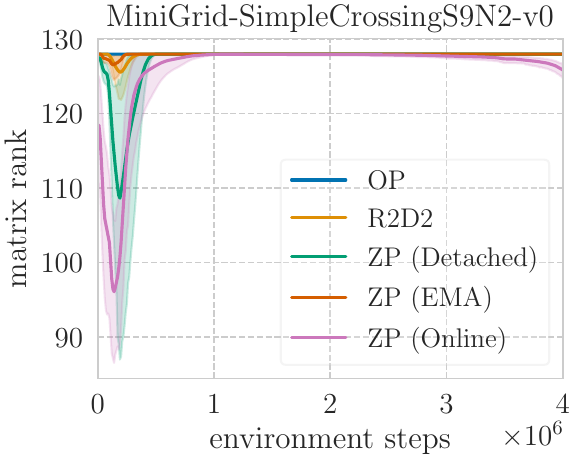}
    \includegraphics[width=\ratiorank\linewidth]{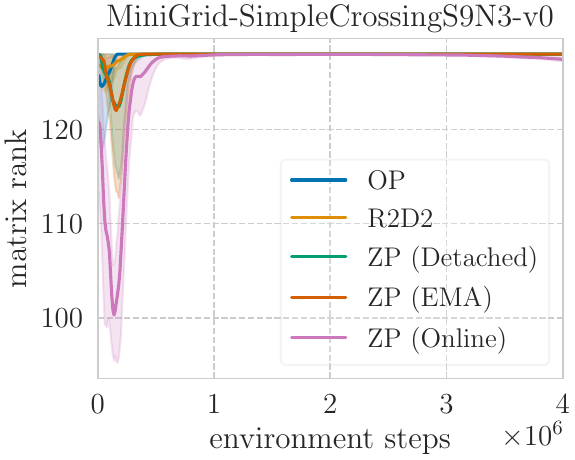}
    \includegraphics[width=\ratiorank\linewidth]{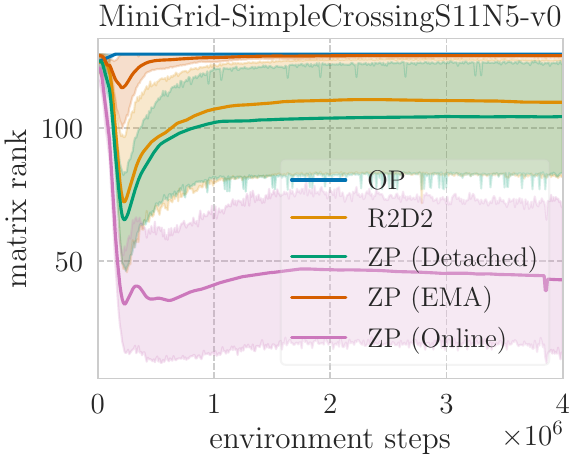}
    \includegraphics[width=\ratiorank\linewidth]{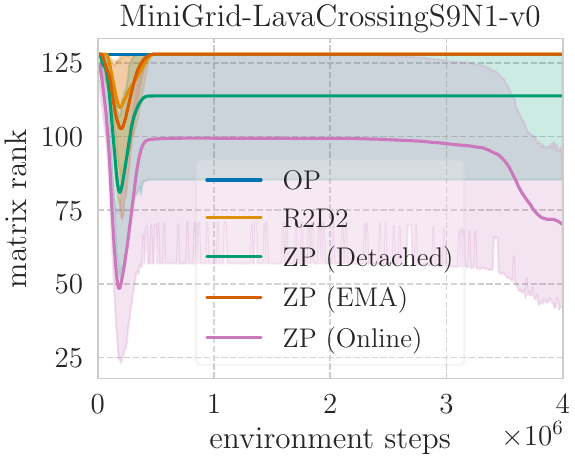}
    \includegraphics[width=\ratiorank\linewidth]{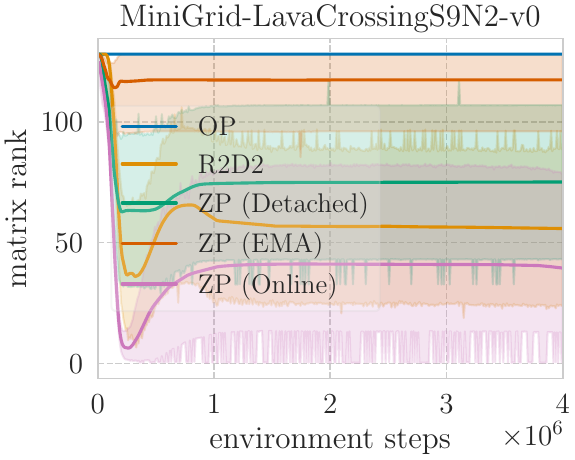}
    \includegraphics[width=\ratiorank\linewidth]{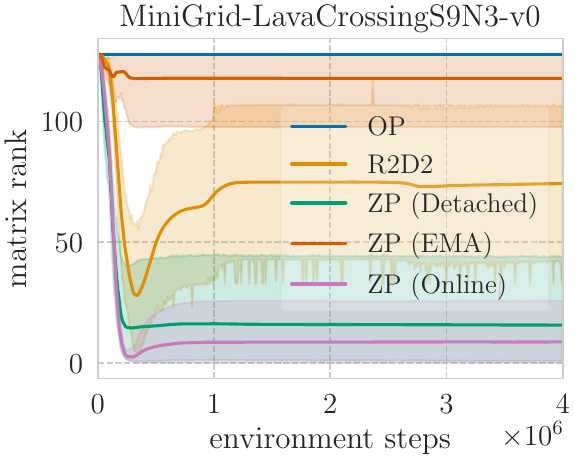}
    \includegraphics[width=\ratiorank\linewidth]{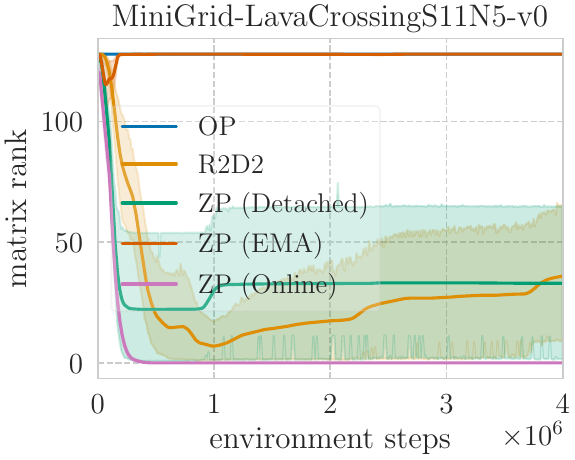}
    \includegraphics[width=\ratiorank\linewidth]{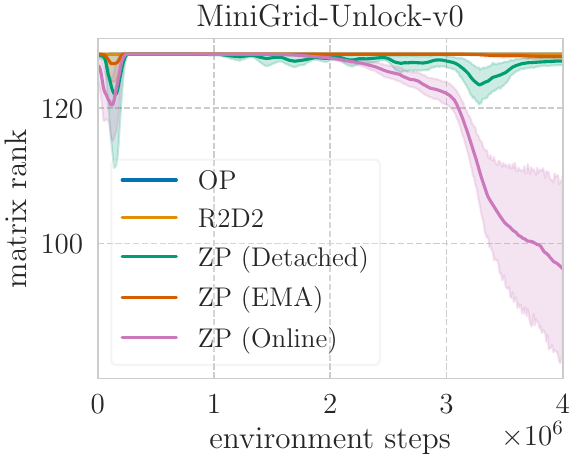}
    \includegraphics[width=\ratiorank\linewidth]{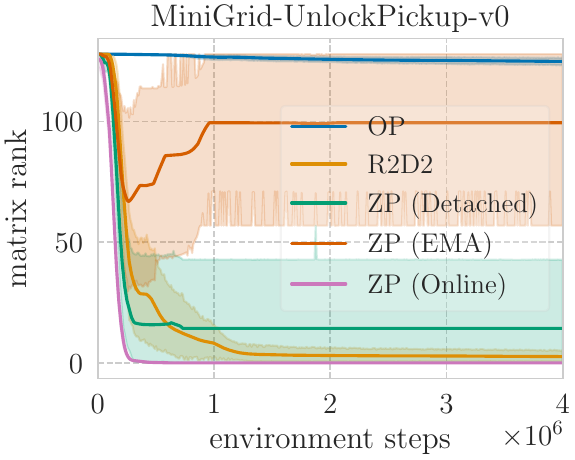}
    \includegraphics[width=\ratiorank\linewidth]{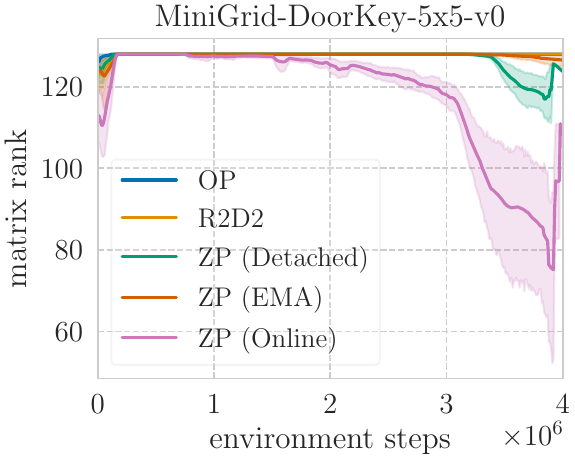}
    \includegraphics[width=\ratiorank\linewidth]{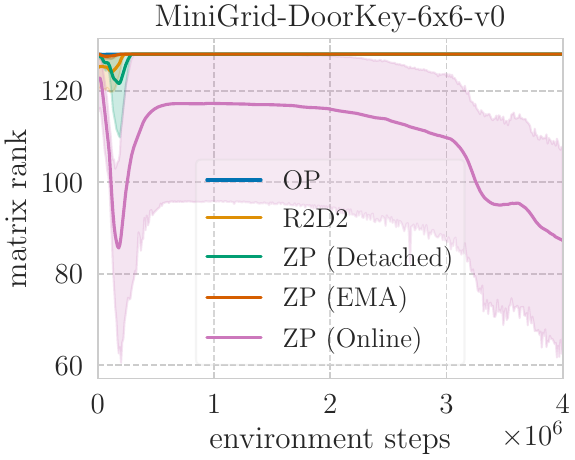}
    \includegraphics[width=\ratiorank\linewidth]{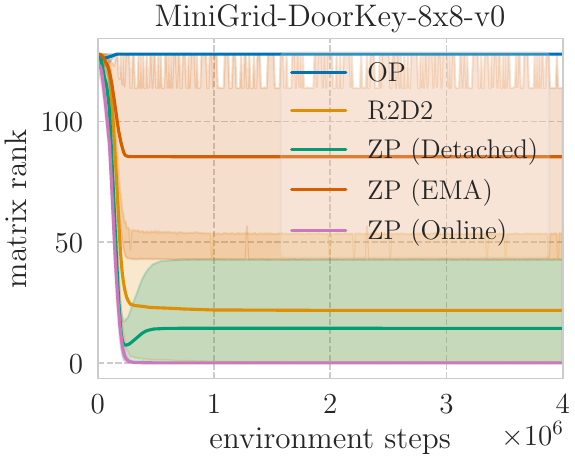}
    \includegraphics[width=\ratiorank\linewidth]{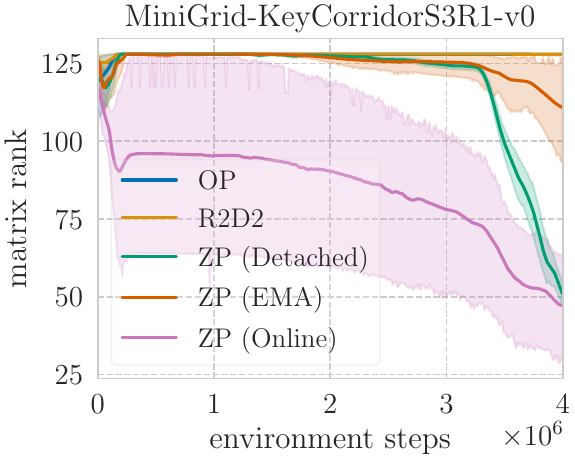}
    \includegraphics[width=\ratiorank\linewidth]{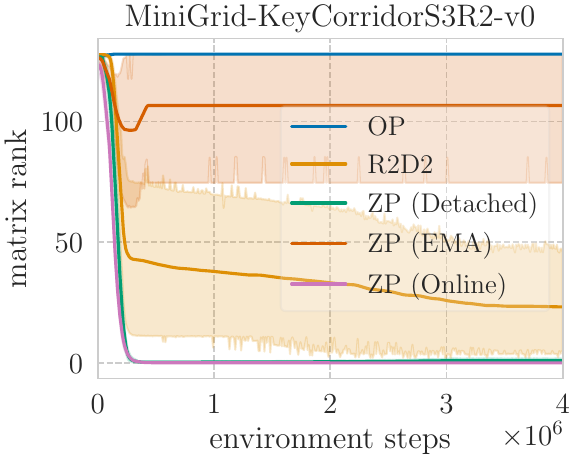}
    \includegraphics[width=\ratiorank\linewidth]{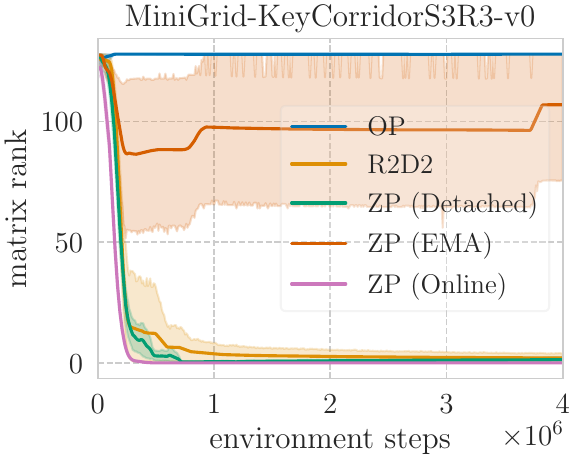}
    \includegraphics[width=\ratiorank\linewidth]{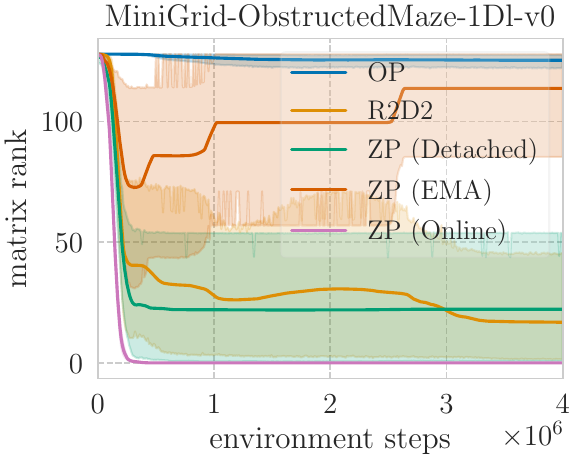}
    \includegraphics[width=\ratiorank\linewidth]{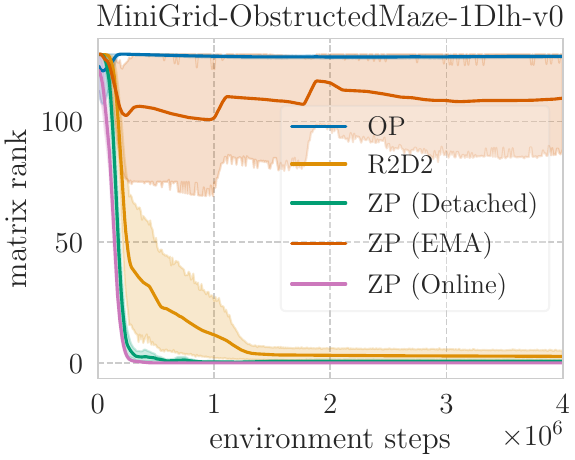}
    \includegraphics[width=\ratiorank\linewidth]{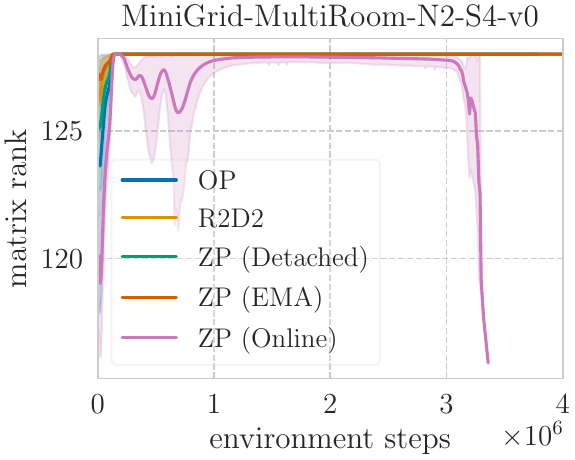}
    \includegraphics[width=\ratiorank\linewidth]{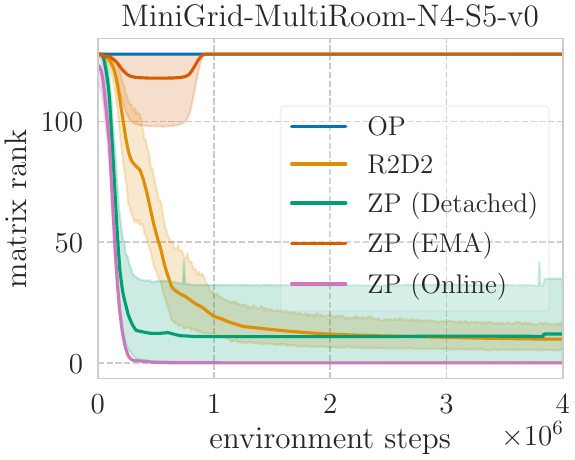}\hfill
    \caption{The estimated \textbf{matrix rank} between \textbf{\ZP targets} (online, detached, EMA), R2D2, and \OP  in 20 MiniGrid tasks over 4M steps, averaged across $\ge$ 9 seeds. The maximal achievable rank is $128$. }
    \label{fig:ZP_optim}
\end{figure}

\end{document}